\documentclass{amsart}
\usepackage[margin=1.4in]{geometry}
\usepackage[numbers,square,compress]{natbib}
\usepackage{amsmath,amsfonts,bm}
\usepackage{amsmath,amssymb,amsthm,mathtools}

\def\eqref#1{equation~\ref{#1}}
\def\1{\bm{1}}

\DeclareMathAlphabet{\mathsfit}{\encodingdefault}{\sfdefault}{m}{sl}
\SetMathAlphabet{\mathsfit}{bold}{\encodingdefault}{\sfdefault}{bx}{n}

\newcommand{\R}{\mathbb{R}}

\newtheorem{theorem}{Theorem}
\newtheorem{lemma}{Lemma}

\newtheorem{corollary}{Corollary}

\newtheorem{definition}{Definition}
\newcommand{\Lip}{\operatorname{Lip}}
\newcommand{\dist}{\operatorname{dist}}

\newcommand{\supp}{\operatorname{supp}}

\usepackage{hyperref}
\usepackage{url}
\usepackage{enumitem}
\usepackage{titletoc}

\title{Multi-Attractor GNNs: Set-Valued Expressivity Beyond Unique Equilibria
}

\author{Jialin Liu}
\address{(JL) School of Data, Mathematical, and Statistical Sciences, University of Central Florida, Orlando, FL 32826.}
\email{jialin.liu@ucf.edu}
\date{\today}

\renewcommand{\eqref}[1]{(\ref{#1})}
\allowdisplaybreaks

\usepackage{booktabs}
\usepackage{longtable}
\usepackage{wrapfig}
\usepackage{graphicx} 
\begin{document}

\begin{abstract}
Recurrent and equilibrium graph neural networks (GNNs) are often designed to have a unique fixed point or trained against a single target per graph. Yet many combinatorial and scientific problems are inherently set-valued: several valid solutions coexist for the same graph, and nothing in the problem singles one out. Training against a single designated target can then make the model fit an arbitrary selection rule rather than the full solution set. For example, when a task is invariant to node relabeling, a symmetric graph has a symmetric solution set, yet may admit no symmetric solution; isolating a single target then imposes an arbitrary symmetry-breaking choice.

In this paper, we show that multiple equilibria of recurrent GNNs are a resource rather than a defect: one weight-tied message-passing GNN can represent such set-valued equivariant maps through its attractor landscape, where different initializations approach different valid solutions. Under the stated regularity assumptions, we establish this expressive power in two steps. First, we prove the existence of globally Lipschitz, permutation-equivariant multi-attractor dynamics that converge almost surely to valid solutions while reaching every solution branch with positive probability. Second, we establish their approximate realization by recurrent message passing with continuous component maps, with arbitrarily small update and limiting errors and arbitrarily high probability. This goes beyond standard universality arguments, since message passing cannot by itself distinguish symmetric nodes: we show that the evolving state keeps nodes distinguishable at every finite step, without auxiliary node identifiers. In practice, such dynamics can be learned without solution labels from problem-specific energies. On three scientific tasks (ground states of Ising models, structural module detection in protein graphs, and steady states of chemical reaction networks), the learned updates produce multiple high-quality predictions and high numerical convergence rates, with better average solution quality than the tested unique-equilibrium, single-target, and feedforward baselines, while remaining competitive with far larger diffusion-based solvers.
\end{abstract}

\maketitle

\section{Introduction}
\label{sec:intro}
%\vspace{-2.5mm}

Recurrent graph neural networks (GNNs) repeatedly apply a shared message-passing update, allowing computation to deepen without introducing new parameters at every step \cite{li2015gated}. This iterative structure provides a natural mechanism for capturing long-range dependencies, while equilibrium-based variants support implicit differentiation without storing the full sequence of intermediate states \cite{gu2020implicit,baker2023implicit,yang2025implicit}.

For an $n$-node graph $G=(A,X)$, $A$ encodes connectivity and edge attributes, and $X_i$ contains node $i$'s features. A prediction $y\in\mathbb{R}^{n\times q}$ assigns each node a $q$-dimensional output $y_i$. A feedforward GNN computes $y=\mathrm{GNN}_{\theta}(G)$ directly; a recurrent predictor updates a state $y_t\in\mathbb{R}^{n\times q}$:
%\vspace{-0.5mm}
\[
y_{t+1}=\mathrm{GNN}_{\theta}(G,y_t),
\qquad t=0,1,2,\ldots.
%\vspace{-0.5mm}
\]
From a deterministic or random $y_0$, each step receives $(X_i,(y_t)_i)$ at node $i$ and exchanges messages along edges to produce $(y_{t+1})_i$. The graph and parameters $\theta$ remain fixed. The usual goal is a limit $y_\infty$ approximating a prescribed target $f(G)$.

An equilibrium, or fixed point, is a state unchanged by the update. Implicit GNNs often seek \emph{well-posedness}: existence and uniqueness of this state for each input \cite{gu2020implicit,liu2021eignn,park2022convergent}. A standard sufficient condition is \emph{global contraction}: each update shrinks the distance between any two states by a common factor $\kappa<1$. Iteration then converges to the same fixed point from every initialization \cite{scarselli2009graph}. Strong monotonicity of an associated operator offers another route to uniqueness \cite{winston2020monotone,baker2023implicit}; uniqueness alone need not ensure convergence of direct iteration. Single-target training instead encourages one designated output without guaranteeing uniqueness. Contraction rules out initialization-dependent solutions regardless of the training objective.

Many graph or combinatorial tasks admit multiple valid solutions for the same input. Selecting one target can discard meaningful alternatives or conflict with input symmetries. For example:
%\vspace{-2.5mm}
\begin{itemize}[leftmargin=5mm]
    \item \textbf{Two coloring.} Consider proper two-coloring of the four-node cycle $1-2-3-4-1$, with identical node features: each node receives a color in $\{0,1\}$, and adjacent nodes must have different colors. There are exactly two valid assignments, $(0,1,0,1)$ and $(1,0,1,0)$, which are exchanged by a one-step rotation that leaves the input graph unchanged. A vanilla GNN receiving only the graph must assign identical outputs to all nodes and therefore cannot produce either coloring. Symmetry breaking can distinguish these nodes and enable a valid coloring, but single-target training still arbitrarily favors one equally valid solution.
    %\vspace{-0.5mm}
    \item \textbf{Nonlinear equilibria.}
Chemical reaction networks are represented as bipartite graphs connecting species to participating reactions, with reaction parameters as attributes; the task is to predict stationary species concentrations. Even one dynamic species can admit multiple solutions for the same graph and parameters: a Schl\"ogl model has steady-state equation $0=8-14c+7c^2-c^3$ and three positive solutions $c\in\{1,2,4\}$. All three satisfy the same physical equation, which selects no preferred target. Discovering multiple steady states can reveal coexisting stable states and help explain switching between them.
%\vspace{-1.5mm}
\end{itemize}

These examples motivate recovering multiple valid solutions per graph
without a preferred target. Recurrent GNNs provide a natural mechanism:
initializations $y_0$ and $y_0^{\prime}$ can converge to different limits
$y_\infty$ and $y_\infty^{\prime}$ under the same learned update
(Figure~\ref{fig:diagram-multi-attractor}).
Only the initial node states change between runs; $G$ and $\theta$ remain fixed. 
We seek convergence to a valid solution within each run and
access to the graph's different solutions across runs.

\begin{figure}
    \centering
    \includegraphics[width=0.65\linewidth]{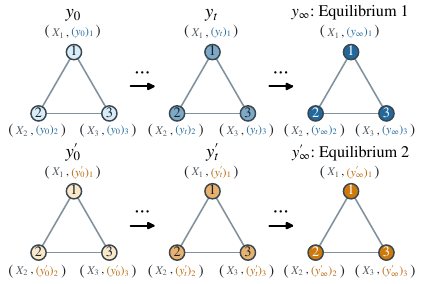}
    \caption{Different initializations can lead to different equilibria
under the same recurrent GNN update, with the graph and model
parameters held fixed.}
    \label{fig:diagram-multi-attractor}
\end{figure}

Building on the set-valued dynamical viewpoint of \citet{jore2026bifurcation}, we study permutation-equivariant representation and message-passing realization with recurrent GNNs. Our contributions address three questions.
%\vspace{-2.5mm}
\begin{itemize}[leftmargin=5mm]
    \item \textbf{Can one shared update represent many solutions?}
    Under the stated assumptions, we prove that a single recurrent message-passing GNN
    can converge arbitrarily
    close to the solution set, with arbitrarily high probability
    over graphs and initializations
    (Theorems~\ref{thm:ideal-dynamics-main}
    and~\ref{thm:gnn-realization-main}).
%\vspace{-0.5mm}
\item \textbf{Is a unique equilibrium merely less diverse?}
Under suitable conditions, every continuous, equivariant, contractive update misses all valid solutions with positive probability, even on almost-surely asymmetric graphs. This applies to MP-GNNs of arbitrary finite depth and width, regardless of training (Theorem~\ref{thm:contractive-limits-main}).
%\vspace{-0.5mm}
    \item \textbf{Can such dynamics be learned without solution
    labels, and do they pay off?}
    On Ising ground states, protein module detection, and chemical
    steady states, energy or residual training from random starts
    yields multiple high-quality predictions and high numerical
    convergence rates. Average solution quality exceeds the tested
    single-target, unique-equilibrium, and feedforward baselines
    and remains competitive with much larger diffusion-based solvers.
\end{itemize}

%\vspace{-2mm}
\section{Main results}
\label{sec:main-theory}
%\vspace{-2mm}

Existing expressivity theory studies single-valued invariant or equivariant maps \citep{xu2018powerful,keriven2019universal,abboud2020surprising}. Can one GNN instead represent solution sets through initialization-dependent limits? We establish existence, illustrate the mechanism, and identify an obstruction under contraction.

%\vspace{-2mm}
\subsection{Can Recurrent GNNs Represent Set-Valued Equivariant Maps?}
\label{sec:theory-represent}
%\vspace{-2mm}

To make the question precise, let $Y_n=\mathbb{R}^{n\times q}$ be the space of node-level outputs.
Let $\mathcal{G}$ be a family of $n$-node graphs that is closed under node permutations.
A permutation $\pi$ acts on a graph by relabeling its nodes, and on $Y_n$ by permuting rows.
A finite solution map assigns to each $G\in\mathcal{G}$ a set
\[
    \mathcal{S}(G)=\{f_1(G),\ldots,f_M(G)\}\subseteq Y_n,
\]
where the branches $f_a$ may coincide on some $G$.
We require the map to be \textit{set-valued equivariant},
\[
    \mathcal{S}(\pi G)=\pi\,\mathcal{S}(G).
\]
\textit{Individual branches need not be equivariant.} In the two-coloring example (Section~\ref{sec:intro}), a one-node rotation preserves the graph but exchanges $(0,1,0,1)$ and $(1,0,1,0)$. The solution set is unchanged; selecting either coloring alone violates equivariance.

A recurrent GNN represents solutions through the limits of its update $y_{t+1}=\mathrm{GNN}_{\theta}(G,y_t)$ with $ y_0\sim\mu_n$, 
where $\mu_n$ is an initialization distribution on $Y_n$ with a density and full support, such as i.i.d.\ standard Gaussian entries.
We collect the limits reached with positive probability in the attractor set
\[
    \mathcal{A}_\theta(G)
    :=
    \left\{
        s\in Y_n :
        \mathbb{P}_{y_0\sim\mu_n}
        \left[
            \lim_{t\to\infty}\mathrm{GNN}_{\theta}^{\,t}(G,y_0)=s
        \right]>0
    \right\},
\]
where $\mathrm{GNN}_{\theta}^{\,t}(G,\cdot)$ denotes $t$ repeated applications of the block.
Ideally, one shared update satisfies, for every $G$: (i) almost every trajectory $\{y_t\}_t$ converges to a limit in $\mathcal{S}(G)$; and (ii) every $s\in\mathcal{S}(G)$ is reached with positive probability.
Together, these ensure $$\mathcal{A}_\theta(G)=\mathcal{S}(G),$$
with no probability mass on nonconvergent trajectories or invalid limits.

Theorem~\ref{thm:ideal-dynamics-main} constructs a globally Lipschitz, equivariant operator satisfying (i)--(ii) on suitably regular graphs. Theorem~\ref{thm:gnn-realization-main} realizes its dynamics approximately by recurrent message passing: limits approach the solution set with arbitrarily high joint probability, and each branch with positive joint probability over graphs and initializations.

\begin{theorem}[Existence of equivariant multi-attractor dynamics]
\label{thm:ideal-dynamics-main}
Let $\mathcal{G}$ be a permutation-invariant family of $n$-node graphs, and let
$\mathcal{S}(G)=\{f_1(G),\ldots,f_M(G)\}\subseteq Y_n=\mathbb{R}^{n\times q}$
be set-valued equivariant, with branch labels as in Definition~\ref{def:equivariant-branches}.
Let $\mathcal{G}_{\star}$ consist of the graphs where the branches are locally Lipschitz and which branches coincide does not change under sufficiently small graph perturbations.

Then there exists a single operator $T:(G,y)\mapsto T_G(y)$ that is jointly globally Lipschitz and permutation equivariant,
$T_{\pi G}(\pi y)=\pi T_G(y)$.
For every $G\in\mathcal{G}_{\star}$, the iteration
$y_{t+1}=T_G(y_t)$, initialized from any distribution $\mu_n$ with a density and full support on $Y_n$, satisfies
\[
\begin{aligned}
    y_t \to  y_\infty\in\mathcal{S}(G)
    &&\text{almost surely},\\
    \mathbb{P}_{y_0\sim\mu_n}(y_\infty=s)>0
    &&\text{for every }s\in\mathcal{S}(G).
\end{aligned}
\]
\end{theorem}

Thus Lipschitz regularity, which bounds sensitivity without requiring contraction, permits multiple solutions and convergent trajectories (Appendix~\ref{app:ideal}). We next realize this behavior by message passing.

\textbf{Message-passing GNNs.}
For a graph $G=(A,X)$ and state $y\in Y_n$, define a GNN block $\mathrm{GNN}_{\theta}(G,y)$ with finitely many layers and finite hidden dimensions:
%\vspace{-2mm}
\begin{equation}
\label{eq:mp-block-main}
\begin{aligned}
    h_i^{(0)}
    &= \operatorname{enc}(X_i,y_i),\\
    %\vspace{-0.5mm}
    h_i^{(\ell+1)}
    &= U_\ell\!\left(
        h_i^{(\ell)},
        \sum_{j=1}^{n}
        M_\ell(h_i^{(\ell)},h_j^{(\ell)},A_{ji})
    \right),\\
    %\vspace{-0.5mm}
    [\mathrm{GNN}_{\theta}(G,y)]_i
    &= O\!\left(
        h_i^{(L)},
        \sum_{j=1}^{n}R_L(h_j^{(L)})
    \right),
    %\vspace{-1mm}
\end{aligned}
\end{equation}
where $\ell=0,\ldots,L-1$ and $i=1,\ldots,n$.
The encoder and component maps $M_\ell,U_\ell,R_L,O$ are
arbitrary continuous maps between finite-dimensional Euclidean
spaces, shared across nodes.
The edge features $A_{ji}$ encode edge attributes and edge
presence; messages can be chosen to vanish on absent edges.
Our theory uses arbitrary continuous component maps, rather than a fixed MLP parameterization; its guarantee is representational, not a guarantee for training.

\begin{theorem}[Approximate realization by a recurrent MP-GNN]
\label{thm:gnn-realization-main}
Let $\mathcal{S}$ and $\mathcal{G}_{\star}$ be as in Theorem \ref{thm:ideal-dynamics-main}, and assume $q\ge2$.
Let $G$ follow a tight Borel probability on $\mathcal{G}_{\star}$, and independently draw $y_0\sim\mu_n$, where $\mu_n$ has a density and full support on $Y_n$.
Let $T_G$ be the ideal operator in Theorem \ref{thm:ideal-dynamics-main}.

For any $\varepsilon,\rho>0$ and $\delta\in(0,1)$, there exists a single GNN of the form \eqref{eq:mp-block-main} such that, with probability at least $1-\delta$ over $(G,y_0)$, the iteration
$y_{t+1}=\mathrm{GNN}_{\theta}(G,y_t)$
converges to a limit $y_\infty$ satisfying
%\vspace{-1mm}
\[
        \sup_{t\ge0}
    \bigl\|
        \mathrm{GNN}_{\theta}(G,y_t)-T_G(y_t)
    \bigr\|
    <\rho 
    ~~~~\text{and}~~~~
    \operatorname{dist}(y_\infty,\mathcal{S}(G))
    <\varepsilon.
%\vspace{-1mm}
\]
Moreover, every solution branch is reached within $\varepsilon$ with positive joint probability:
%\vspace{-1mm}
\[
    \mathbb{P}\!\left(
        y_t\to y_\infty,\ 
        \|y_\infty-f_a(G)\|<\varepsilon
    \right)>0,
    \qquad a=1,\ldots,M.
\]
\end{theorem}

%\vspace{-3mm}
The branch guarantee is joint over graphs and initializations, not per graph. Vanilla GNNs lack universality for continuous equivariant operators \citep{xu2018powerful}; random initialization can provide \textbf{\textit{unique node identifiers}} to overcome this \citep{abboud2020surprising}. But subsequent states depend on the graph and each other, even from i.i.d.\ $y_0$. Do they still distinguish nodes?

Our proof establishes that independence is unnecessary: the ideal trajectories $z_t:=T_G^t(y_0)$ retain distinct node states at every finite time,
%\vspace{-1mm}
\[
    \mathbb{P}\!\left(
        z_{t,i}\ne z_{t,j}
        \text{ for all }i\ne j
        \text{ and all }t\in\mathbb{N}_0
    \right)=1.
%\vspace{-1mm}
\]
These evolving identifiers enable recurrent message passing to approximate the ideal dynamics (Appendix~\ref{app:init-only-family-wide-mpgnn}).

%\vspace{-2mm}
\subsection{Toy Example: Single-Source Graph Diffusion}
\label{sec:toy}
%\vspace{-2mm}

On a problem with exactly three solutions, we test whether one trained
recurrent GNN approaches different solutions from different random
initial states and recovers all three across runs. We compare against
contraction and single-target supervision.

\textbf{Problem.}
The input is a three-node undirected graph $G=(A,0)$ with independent
edge weights $w_{12},w_{23},w_{31}\sim\operatorname{Unif}[0.5,1.5]$ and
no distinguishing static node features. Here $A$ is the symmetric weighted
adjacency matrix with zero diagonal. Define
$L_G=\operatorname{diag}(A\mathbf1)-A$ and $H_G=I_3+L_G$.
We seek a source indicator $b\in\{0,1\}^3$ selecting exactly one node,
$\mathbf1^\top b=1$, and its diffusion field $u\in\mathbb R^3$ satisfying
$H_Gu=b$. Since $H_G$ is positive definite, each source choice determines
one diffusion field, giving exactly three solutions:
\begin{equation}
 \mathcal S(G)=\bigl\{[H_G^{-1}e_k\ \ e_k]:k=1,2,3\bigr\}
 \subset\mathbb R^{3\times2},
 \label{eq:toy-problem}
\end{equation}
where $e_k$ is the $k$th standard basis vector.
The source is an output to discover, not an input supplied to the model.
Figure~\ref{fig:diffusion} illustrates one graph and its three solutions.

\textbf{Model and training.}
The recurrent state $y_t=[u_t\ b_t]\in\mathbb R^{3\times2}$ stores two
values per node. From six independent standard Gaussian entries, we
iterate $y_{t+1}=\mathrm{GNN}_{\theta}(G,y_t)$ using a shared block with two
message-passing layers, following \eqref{eq:mp-block-main}.
For each fixed graph and model, we seek \textit{convergent trajectories reaching each solution in $\mathcal S(G)$ with positive probability across initializations.}
We train using the energy function
%\vspace{-2mm}
\begin{equation}
 E_G(u,b)=\frac13\|H_Gu-b\|_2^2
       +\sum_{i=1}^3 b_i(1-b_i)
       +(\mathbf1^\top b-1)^2,
 \qquad b\in[0,1]^3.
 \label{eq:toy-energy}
 %\vspace{-2mm}
\end{equation}
The three terms are all nonnegative and $E_G(u,b)=0$ holds exactly on $\mathcal S(G)$.
Training uses eight random starts per graph and minimizes the terminal energy $\mathbb E_{G,y_0}[E_G(u_T,b_T)]$ with $T=10$. This objective permits different starts to approach different solutions without prescribing branch labels.

\textbf{Baselines.}
We compare two controls with the same architecture, changing either
the update constraint or training target.
The \emph{\underline{contractive GNN}} uses the same energy objective but projects
its parameters throughout training to enforce
$\operatorname{Lip}_y(\mathrm{GNN}_{\theta}(G,\cdot))\le0.9$, ensuring a unique
fixed point for each graph.
The \emph{\underline{single-target-supervised GNN}} has no contraction constraint.
For each training graph, we uniformly select one of its three exact
solutions and retain it across all initializations and training visits.
Squared-error training asks all starts to predict that target without
guaranteeing a unique equilibrium. Details are in Appendix~\ref{app:toy}.

\textbf{Results.}
Figure~\ref{fig:toy-result} tracks the source coordinates $b_t$ of 128
trajectories on one held-out graph, with identical initial states across
all three models. The multi-attractor model forms three clusters near
the valid source indicators $e_1,e_2,e_3$, whereas the supervised and
contractive controls each concentrate near a single point, visibly
separated from all three. The diffusion coordinates $u_t$ exhibit the
same qualitative pattern (Appendix~\ref{app:toy}, Figure~\ref{fig:toy-u}).
Thus energy training can produce multiple attractors near different
solutions, consistent with Theorems~\ref{thm:ideal-dynamics-main}
and~\ref{thm:gnn-realization-main}, while contraction or single-target
supervision can concentrate trajectories far from every valid solution.

\begin{figure}[t]
 \centering
 \includegraphics[width=0.99\linewidth]{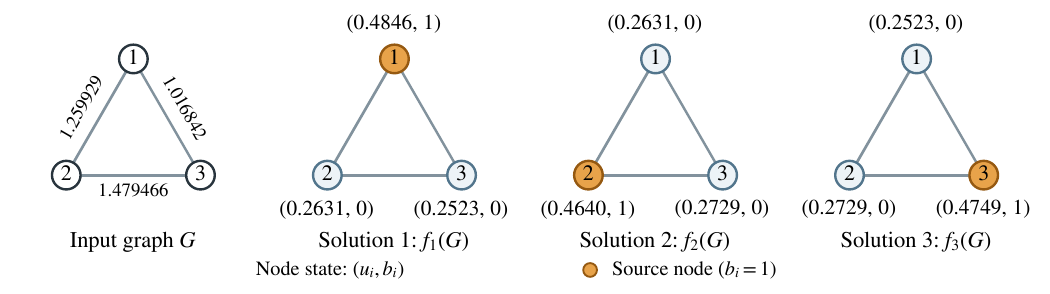}
 %\vspace{-4mm}
 \caption{One graph, three valid solutions to single-source diffusion.
 Left: the input graph, with edge weights labeled.
 Right: each choice of source gives a solution
 $f_k(G)=[H_G^{-1}e_k\ \ e_k]$.
 Node labels show $(u_i,b_i)$, with diffusion values rounded to four
 decimal places; orange marks the selected source ($b_i=1$).
 All three outputs are valid for the same input, which specifies no source.}
 \label{fig:diffusion}
 %\vspace{-4mm}
\end{figure}

\begin{figure}[t]
 \centering
 \includegraphics[width=0.99\linewidth]{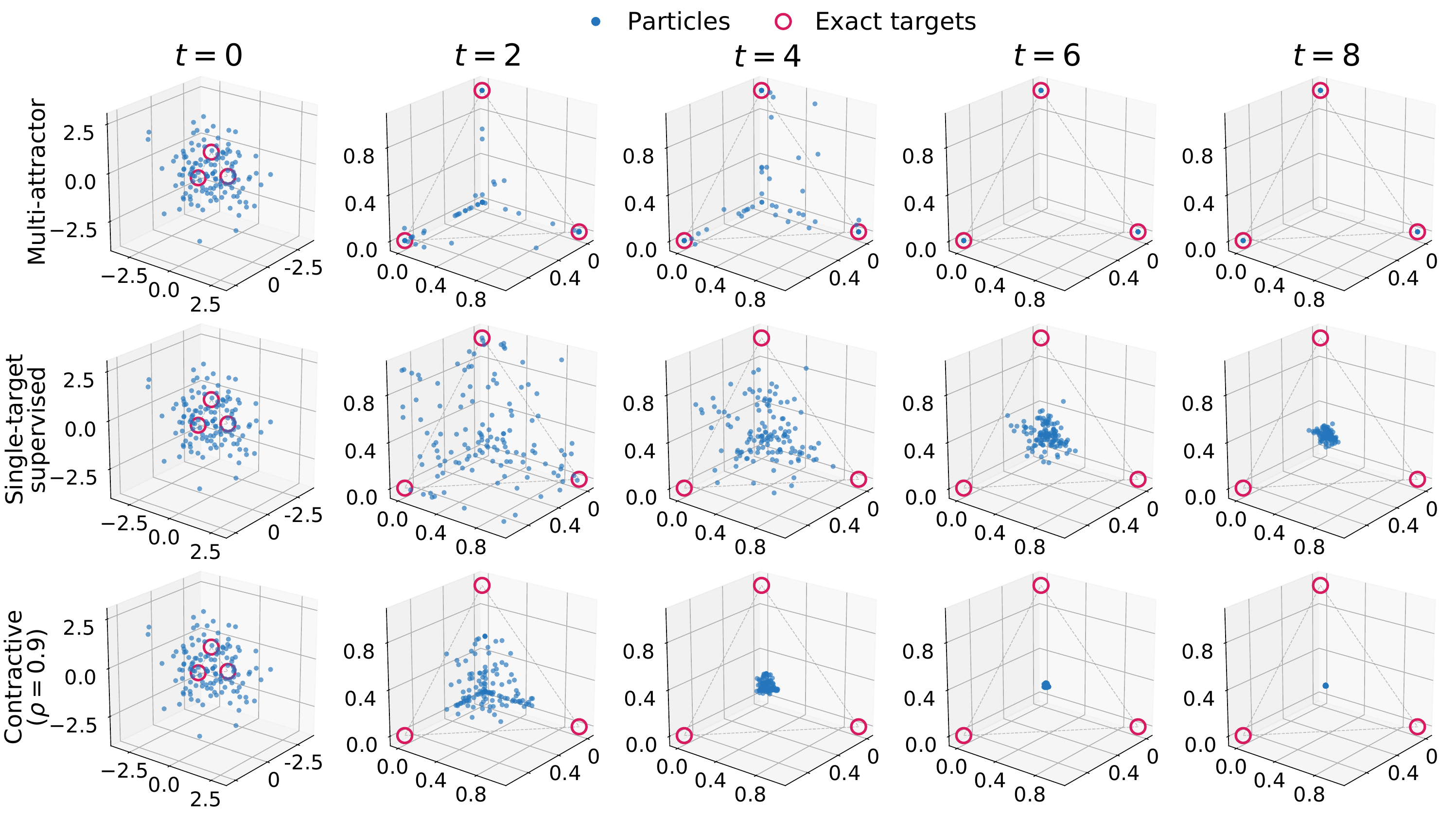}
 %\vspace{-4mm}
 \caption{Different initial states can approach different solutions
 under the same learned update.
 Each blue point represents one of 128 trajectories, plotted in source
 coordinates $((b_t)_1,(b_t)_2,(b_t)_3)$; magenta circles mark the three
 valid source indicators $e_1,e_2,e_3$.
 Columns show recurrent steps $t=0,2,4,6,8$.
 The multi-attractor model (top) forms three clusters near these targets,
 whereas the single-target-supervised (middle) and contractive (bottom)
 controls each concentrate away from them.
 All rows use the same held-out graph (test graph 0), training seed 1,
 and initial states. 
 }
 %\vspace{-4mm}
 \label{fig:toy-result}
\end{figure}

%\vspace{-2mm}
\subsection{Can a Unique Equilibrium Miss Every Valid Solution?}
\label{sec:theory-single-equilibrium}
%\vspace{-2mm}

Do these failures reflect training or a structural limitation?
A unique equilibrium cannot represent multiple solutions through
different initializations, but can it always recover one valid solution?
The following theorem answers no for continuous, permutation-equivariant,
globally contractive updates: even with independently sampled edge
weights, their equilibrium remains separated from every valid solution
with positive probability, regardless of training or initialization.

\begin{theorem}
\label{thm:contractive-limits-main}
Let $\mathcal G_\triangle$ be the family of weighted triangles in
\eqref{eq:toy-problem}, with edge weights in $[w_-,w_+]$, where
$0<w_-<w_+<\infty$. Let $\nu$ be the graph law induced by sampling the
three edge weights independently and uniformly from $[w_-,w_+]$.
Consider any jointly continuous, permutation-equivariant operator
$F:\mathcal G_\triangle\times Y_3\to Y_3$, where
$Y_3=\mathbb R^{3\times2}$, satisfying
\begin{equation}
    \label{eq:thm-contractive}
        \|F_G(y)-F_G(z)\|_*
    \le \kappa\|y-z\|_*
    \qquad
    (G\in\mathcal G_\triangle,\ y,z\in Y_3)
\end{equation}
for some fixed norm $\|\cdot\|_*$ and $\kappa\in[0,1)$.
This includes any MP-GNN of \eqref{eq:mp-block-main} with continuous
component maps satisfying \eqref{eq:thm-contractive},
regardless of its finite depth and width.

For every $0<\varepsilon<\sqrt{2/3}$, there exists
$\delta_{F,\varepsilon}>0$ such that, under any joint law of $(G,Y_0)$
with graph marginal $\nu$, the recurrence $Y_{t+1}=F_G(Y_t)$ converges
to a limit $Y_\infty$ and
\[
    \mathbb P\!\left(
        \operatorname{dist}\bigl(Y_\infty,\mathcal S(G)\bigr)
        >\varepsilon
    \right)
    \ge \delta_{F,\varepsilon},
\]
where distance is measured in the Euclidean norm.
In particular, no such operator converges to a valid solution almost
surely over independently weighted graphs.
\end{theorem}

%\vspace{-2mm}
Random initialization distinguishes nodes even on symmetric graphs;
Figure~\ref{fig:toy-result} shows its influence in the early iterations.
Contraction eventually erases these differences: all trajectories approach one equilibrium.
On the symmetric triangle, equivariance forces this equilibrium
to preserve the graph's symmetries, preventing single-source selection.
For independently weighted graphs, the equilibrium need not be
symmetric, but its continuous dependence on the graph extends
the failure to nearby graphs with positive probability.
The failure probability depends on the operator; no uniform lower bound across models is claimed. See Appendix~\ref{app:contractive-obstruction} for the proof and extension.

%\vspace{-2mm}
\section{Scientific applications}
\label{sec:applications}
%\vspace{-2mm}

We test whether energy or residual training learns multiple high-quality predictions and numerically convergent dynamics. These empirical tests, including scalar states outside Theorem~\ref{thm:gnn-realization-main}'s $q\ge2$ setting, do not certify asymptotic convergence or complete coverage. Baselines use the reported budgets; inference computation is not matched.

\subsection{Ising Models }
\label{ising:main}
%\vspace{-2mm}

\paragraph{Problem.}
The Ising model is a classical model of interacting magnetic spins and is a standard
discrete optimization problem. In the zero-field setting, each node of the weighted graph
$G=(V,E,J)$ carries a spin $s_i\in\{-1,+1\}$, with energy and ground states
\begin{equation}
\label{ising:problem}
 \mathcal S(G)=\mathop{\rm arg\,min}_{s\in\{-1,+1\}^{|V|}}E_G(s),
 \qquad E_G(s)=-\sum_{\{i,j\}\in E}J_{ij}s_is_j.
\end{equation}
We use open triangular
lattices with antiferromagnetic couplings ($J_{i,j}<0$), which favor
opposite neighboring spins. These preferences conflict on triangles,
creating competing low-energy configurations; global spin reversal
also preserves energy. This multiplicity naturally tests set-valued
prediction: we seek diverse low-energy configurations per graph
rather than one arbitrarily selected target.

\paragraph{Models and baselines.}
Our multi-attractor model (Ours in the tables) uses
$y_{t+1}=\mathrm{GNN}_{\theta}(G,y_t)$ from
$y_0\sim\mathcal N(0,I)$, refining continuous spin estimates and
applying the sign function to $y_T$ to predict discrete spins $\hat{s}$.
Training minimizes $\mathbb E[E_G(y_{T})]$ with $T=10$ shared updates,
encouraging low energy without prescribing a solution for each
initialization. Testing uses 200 updates.

We compare four baseline groups.
\emph{(A) Single-target supervision} minimizes MSE to one Gurobi solution
per graph using the same recurrent architecture, testing the effect of
a designated target on accuracy and diversity.
\emph{(B) Unique-equilibrium models} are energy-trained IGNNs
\citep{gu2020implicit} with contraction constraints and increasing hidden
widths, comparing quality and multiplicity under enforced uniqueness
at several capacities.
\emph{(C) Feedforward models} stack 1, 3, or 10 independently
parameterized copies of our block with the same energy objective,
comparing untied depth with shared recurrence.
\emph{(D) Distributional models} learn sampling distributions or
generative processes, whereas ours learns a shared update with
initialization-dependent solutions. We compare quality and diversity
with Mean-field GNN, Annealed Bernoulli GNN, DIFUSCO, and VAG (VAG-CO)
\citep{karalias2020erdos,sun2022annealed,sun2023difusco,
sanokowski2023variational} as alternative approaches to set-valued
prediction. Appendix~\ref{ising:appendix} provides all details.

\begin{table}\centering
\small
\caption{
Ising results on 500 test graphs with 20 candidates per graph
(mean $\pm$ standard deviation over three training seeds).
$\bar e$: mean energy per node (lower is better);
Hit10: percentage of candidates within 10\% of the certified minimum energy;
$C_{200}$: percentage of trajectories converged by step 200;
$D_{20}$: mean number of distinct optimal configurations found per graph, counting global spin flips separately;
$G_{20}$: percentage of graphs with at least two such configurations,
up to numerical tolerance. Details in Appendix~\ref{ising:appendix}.
}
\label{ising:main-table}
\setlength{\tabcolsep}{2.9pt}
\begin{tabular}{lrrcrr}
\toprule
Method & $\bar e\downarrow$ & Hit10 $\uparrow$ & $C_{200}\uparrow$ & $D_{20}\uparrow$ & $G_{20}\uparrow$ \\
\midrule
Ours & $-1.0794\pm 0.0028$ & $83.98\pm 0.93$ & $99.96\pm 0.02$ & $2.16\pm 0.23$ & $68.73\pm 5.19$ \\
\midrule
Single-target Supervision & $0.2077\pm 0.5878$ & $9.45\pm 4.37$ & $10.11\pm 3.38$ & $0.10\pm 0.05$ & $0.80\pm 0.99$ \\
IGNN adapter (H16) & $-0.4581\pm 0.2538$ & $5.20\pm 3.31$ & $100.00\pm 0.00$ & $0.03\pm 0.02$ & $0.00\pm 0.00$ \\
IGNN adapter (H32) & $-0.7516\pm 0.0185$ & $8.00\pm 0.71$ & $100.00\pm 0.00$ & $0.06\pm 0.00$ & $0.00\pm 0.00$ \\
IGNN adapter (H64) & $-0.7954\pm 0.0143$ & $13.20\pm 2.14$ & $100.00\pm 0.00$ & $0.08\pm 0.01$ & $0.00\pm 0.00$ \\
\midrule
Feedforward ($\times 1$) & $-0.9720\pm 0.0035$ & $36.13\pm 0.75$ & -- & $1.31\pm 0.03$ & $36.67\pm 1.23$ \\
Feedforward ($\times 3$) & $-1.0515\pm 0.0016$ & $69.07\pm 0.66$ & -- & $0.95\pm 0.95$ & $22.40\pm 31.68$ \\
Feedforward ($\times 10$) & $-1.0538\pm 0.0012$ & $70.93\pm 1.52$ & -- & $0.28\pm 0.01$ & $0.00\pm 0.00$ \\
\midrule
Mean-field GNN & $-0.9736\pm 0.0026$ & $36.71\pm 0.78$ & -- & $1.34\pm 0.06$ & $37.80\pm 1.31$ \\
Annealed Bernoulli GNN & $-0.9766\pm 0.0011$ & $37.71\pm 0.86$ & -- & $1.37\pm 0.03$ & $38.27\pm 0.66$ \\
DIFUSCO (matched) & $-0.9033\pm 0.0023$ & $37.64\pm 0.23$ & -- & $1.56\pm 0.03$ & $45.87\pm 0.09$ \\
DIFUSCO (large) & $-1.0362\pm 0.0024$ & $69.82\pm 1.64$ & -- & $2.60\pm 0.10$ & $68.00\pm 2.95$ \\
VAG (matched) & $-0.8893\pm 0.0075$ & $17.18\pm 2.02$ & -- & $0.65\pm 0.11$ & $17.40\pm 4.98$ \\
VAG (large) & $-0.8814\pm 0.0547$ & $17.34\pm 5.69$ & -- & $0.56\pm 0.17$ & $14.00\pm 5.74$ \\
\bottomrule
\end{tabular}
\end{table}

\paragraph{Quality and multiple solutions.}
Our model achieves the lowest mean energy and highest Hit10
(Table~\ref{ising:main-table}): $83.98\%$ of candidates are within 10\%
of optimum. It finds $2.16$ distinct optima per graph on average and
at least two on $68.73\%$ of graphs. Numerical convergence reaches
$99.96\%$ by step 200 and 100\% by step 500 for the same trajectories
(Appendix~\ref{ising:appendix}). Single-target supervision has substantially poorer quality and diversity;
the IGNN adapters converge reliably but return one configuration per graph with low
Hit10. Deeper untied models improve quality but lag in Hit10 and
distinct-optimum discovery. DIFUSCO (large) finds more optima
($D_{20}=2.60$), but has worse mean energy despite approximately 606 times
as many parameters. Consistent with our theory, \textit{one shared update trained on energy
can converge to multiple high-quality solutions}, offering a favorable
combination of convergence, quality, and diversity over single-target
training or enforced uniqueness in these experiments.

%\vspace{-1mm}
\subsection{Structural module detection in protein graphs}
\label{sec:proteins}
%\vspace{-1mm}

\paragraph{Problem.}
We seek structural modules in protein graphs: groups more internally
connected than expected from their node degrees. We replace graph
classification on the 1,113 PROTEINS graphs
\citep{borgwardt2005protein,morris2020tudataset} with instance-wise
modularity maximization \citep{newman2006modularity}.
For an undirected graph $G=(V,E)$ with node degree $d_i$ and node assignments
$z_i\in\{1,\ldots,K\}$, the objective and solution set are
\[
 Q_G(z) = \frac{1}{|E|}\sum_{\{i,j\}\in E}\mathbf1[z_i=z_j] - \sum_{k=1}^{K}\left(\frac{\sum_i d_i\mathbf1[z_i=k]}{2|E|}\right)^2, ~~
 \mathcal S(G)=\mathop{\rm arg\,max}_{z\in\{1,\ldots,K\}^{|V|}} Q_G(z).
\]
We allow at most four modules ($K=4$, with empty modules permitted).
Partitions can have similar energies even modulo module-label permutations,
motivating distinct low-energy predictions without one selected target.

\textbf{Models and baselines.}
Our shared recurrent update
$P_{t+1}=\mathrm{GNN}_{\theta}(G,P_t)$ refines a randomly initialized
\textit{soft partition}, where $P_{t,i}\in\mathbb R^K$ contains node $i$'s
assignment weights over $K$ modules. After $T$ iterations, each node
selects its highest-weight module, $z_i=\arg\max_k (P_{T,i})_k$.
Training minimizes $\mathbb E[e_G(P_T)]$ with $T=5$ over training graphs
and initializations, where $e_G$ is a differentiable soft-partition
energy derived from modularity.
At test time, we apply the learned update for 200 steps from
20 independent initializations per graph.

The four baseline groups (Section~\ref{ising:main}) test training objectives,
equilibrium uniqueness, weight sharing, and alternative solution generation.
Appendix~\ref{app:proteins} details the soft energy, architectures,
training, and evaluation protocols.

\textbf{Results.}
Our model has the highest mean modularity ($0.5560$) and Hit10
($77.49\%$) in Table~\ref{tab:proteins-main}, finding $5.83$ distinct
qualified partitions per graph on average and at least two on $80.78\%$
of graphs. Single-target supervision and unique-equilibrium IGNNs have
poorer quality and diversity. Untied depth improves performance, but
five independent blocks still lag in mean modularity, Hit10, and
partition discovery. DIFUSCO (large) finds more qualified partitions
($D^{.10}_{20}=11.02$) but has lower mean modularity with approximately
$155$ times as many parameters. Together with the convergence results
in Appendix~\ref{app:proteins} (Table~\ref{tab:proteins-long-horizon}), these findings support our
theoretical premise: energy training without prescribed targets lets
one shared update recover multiple high-quality solutions from
different initializations.

\begin{table}[t]
\centering
\small 
%\vspace{-3mm}
\caption{Structural module detection on 111 PROTEINS test graphs with 20 candidates per graph.
$\overline Q_G$ averages modularity over candidates and then graphs.
For optimal modularity $Q_G^\star$, Hit10 is the percentage of candidates
with relative gap $(Q_G^\star-Q_G(z))/|Q_G^\star|\leq0.10$.
$D^{.10}_{20}$: mean number of distinct partitions per graph with absolute gap
$Q_G^\star-Q_G(z)\leq0.10$; $G^{.10}_{20}$: percentage of graphs with
at least two such partitions; $C_{200}$: percentage of trajectories numerically
converged at $t=200$. Results: mean $\pm$ standard deviation over training
seeds 0/1/2 on the same test set. Full protocol in Appendix~\ref{app:proteins}.}
\label{tab:proteins-main}
\setlength{\tabcolsep}{3.5pt}
\begin{tabular}{lrrcrr}
\toprule
Method & $\overline Q_G\uparrow$ & Hit10 $\uparrow$ & $C_{200}\uparrow$ & $D^{.10}_{20}\uparrow$ & $G^{.10}_{20}\uparrow$ \\
\midrule
Ours & $0.5560\pm0.0024$ & $77.49\pm2.06$ & $88.05\pm2.23$ & $5.83\pm0.38$ & $80.78\pm0.42$ \\
\midrule
Single-target Supervision & $0.4911\pm0.0114$ & $41.28\pm7.16$ & $50.81\pm10.84$ & $1.11\pm0.08$ & $33.03\pm2.25$ \\
IGNN (H16) & $0.4147\pm0.0079$ & $8.11\pm2.65$ & $100.00\pm0.00$ & $0.20\pm0.04$ & $0.00\pm0.00$ \\
IGNN (H32) & $0.4196\pm0.0044$ & $11.71\pm2.21$ & $100.00\pm0.00$ & $0.22\pm0.03$ & $0.00\pm0.00$ \\
IGNN (H64) & $0.4235\pm0.0121$ & $8.41\pm2.36$ & $100.00\pm0.00$ & $0.22\pm0.06$ & $0.00\pm0.00$ \\
\midrule
Feedforward ($\times 1$) & $0.4436\pm0.0048$ & $16.52\pm2.22$ & -- & $1.35\pm0.23$ & $30.93\pm6.26$ \\
Feedforward ($\times 2$) & $0.4846\pm0.0156$ & $29.50\pm3.61$ & -- & $2.09\pm0.26$ & $53.15\pm6.62$ \\
Feedforward ($\times 3$) & $0.5088\pm0.0108$ & $39.92\pm5.17$ & -- & $2.95\pm0.62$ & $71.47\pm9.71$ \\
Feedforward ($\times 4$) & $0.5285\pm0.0022$ & $51.35\pm1.68$ & -- & $4.24\pm0.51$ & $82.88\pm6.62$ \\
Feedforward ($\times 5$) & $0.5305\pm0.0044$ & $54.97\pm3.59$ & -- & $4.32\pm0.47$ & $78.68\pm4.49$ \\
\midrule
Mean-field GNN & $0.4020\pm0.0122$ & $8.53\pm2.01$ & -- & $0.69\pm0.05$ & $16.22\pm0.00$ \\
Annealed Bernoulli GNN & $0.4131\pm0.0103$ & $9.91\pm2.35$ & -- & $0.80\pm0.01$ & $18.32\pm1.53$ \\
DIFUSCO (matched) & $0.5115\pm0.0127$ & $41.85\pm7.97$ & -- & $8.47\pm1.20$ & $97.30\pm1.27$ \\
DIFUSCO (large) & $0.5515\pm0.0014$ & $74.49\pm1.66$ & -- & $11.02\pm0.03$ & $97.00\pm0.42$ \\
VAG (matched) & $0.3952\pm0.0024$ & $8.05\pm0.14$ & -- & $0.57\pm0.01$ & $12.01\pm0.85$ \\
VAG (large) & $0.2577\pm0.0237$ & $0.44\pm0.31$ & -- & $0.35\pm0.15$ & $7.21\pm1.47$ \\
\bottomrule
\end{tabular}
\end{table}

%\vspace{-1mm}
\subsection{Chemical reaction-network multistationarity}
\label{chem:main}
%\vspace{-1mm}

\textbf{Problem.}
A chemical reaction network has $s$ species $i$ and $R$ reactions $r$.
The matrices $\alpha,\beta\in\mathbb Z_{\ge0}^{R\times s}$ give the
numbers of molecules of species $i$ consumed ($\alpha_{ri}$) and produced
($\beta_{ri}$) by reaction $r$; $\kappa\in\mathbb R_{>0}^{R}$ contains
rate constants. The graph $G=(\alpha,\beta,\kappa)$ has one node per
species and reaction, connected bidirectionally when
$\alpha_{ri}+\beta_{ri}>0$. Edges carry $(\alpha_{ri},\beta_{ri})$;
reaction nodes encode $\log\kappa_r$.
Under mass-action kinetics, the species concentrations
$c\in\mathbb R_{>0}^{s}$ evolve according to
\[
  \dot c=f_G(c)=Nv(c),\qquad
  N_{ir}=\beta_{ri}-\alpha_{ri},\qquad
  v_r(c)=\kappa_r\prod_{i=1}^{s}c_i^{\alpha_{ri}}.
\]
We seek stationary solutions
$\mathcal S(G)=\{c\in\mathbb R_{>0}^{s}:f_G(c)=0\}$.
Multiple positive solutions for one $G$ constitute multistationarity,
a set-valued prediction task. The solver seeks roots of $f_G$ without
requiring physical stability under the chemical ODE.
Our kinetic instances use published network structures
\cite{yao2025understanding}, with two numerically verified stationary
witnesses per instance.

\textbf{Models and baselines.}
Species log concentrations follow the shared update
$z_{t+1}=\mathrm{GNN}_{\theta}(G,z_t)$, with output $c_T=\exp z_T$.
Training minimizes equation residuals without prescribed targets, using
randomized unrolls of at most 32 steps; the main comparison uses $T=200$.
The four baseline groups (Section~\ref{ising:main}) are single-target MSE,
physics-trained contractive IGNNs, untied feedforward depth controls,
and distributional models. For continuous states, the latter comprise
DDPM, DDIM, and DIS with matched and larger upstream-recipe graph
adaptations; details are in Appendix~\ref{chem:appendix}.

\begin{table}[t]
\centering\small
\setlength{\tabcolsep}{3.5pt}
\caption{Chemistry results on 236 test graphs with 20 candidates per graph.
$\overline E_{\rm res}$: mean Huber penalty on the mass-action residual (lower is better);
Hit$_{10^{-3}}$: percentage of final states with normalized physical residual $R\le10^{-3}$;
$C$: percentage of recurrent trajectories with next-step log-concentration
RMS change below $10^{-3}$ at $T=200$;
$D_{20}^{\rm raw}$: mean number of distinct accepted prediction clusters per graph;
$G_{20}^{\rm raw}$: percentage of graphs with at least two. These count separated accepted predictions, not certified distinct roots.
Entries: mean $\pm$ population standard deviation over training seeds 0/1/2.
}
\label{chem:main-results}
\begin{tabular}{lrrrrr}
\toprule
Method & $\overline E_{\rm res}\downarrow$ & Hit$_{10^{-3}}\uparrow$ &
 $C\uparrow$ & $D_{20}^{\rm raw}\uparrow$ & $G_{20}^{\rm raw}\uparrow$\\
\midrule
Ours & 0.0188 $\pm$ 0.0046 & 83.77 $\pm$ 3.44 & 91.99 $\pm$ 0.35 & 3.506 $\pm$ 0.314 & 41.67 $\pm$ 17.01 \\
\midrule
Single-target Supervis. & 0.0373 $\pm$ 0.0059 & 18.78 $\pm$ 3.88 & 88.48 $\pm$ 4.29 & 1.371 $\pm$ 0.137 & 11.30 $\pm$ 0.72 \\
IGNN (H16) & 0.1214 $\pm$ 0.0823 & 0.00 $\pm$ 0.00 & 100.00 $\pm$ 0.00 & 0.000 $\pm$ 0.000 & 0.00 $\pm$ 0.00 \\
IGNN (H32) & 0.1180 $\pm$ 0.0709 & 0.00 $\pm$ 0.00 & 100.00 $\pm$ 0.00 & 0.000 $\pm$ 0.000 & 0.00 $\pm$ 0.00 \\
IGNN (H64) & 0.0960 $\pm$ 0.0056 & 0.00 $\pm$ 0.00 & 100.00 $\pm$ 0.00 & 0.000 $\pm$ 0.000 & 0.00 $\pm$ 0.00 \\
\midrule
Feedforward ($\times 1$) & 0.1082 $\pm$ 0.0081 & 0.00 $\pm$ 0.00 & --- & 0.000 $\pm$ 0.000 & 0.00 $\pm$ 0.00 \\
Feedforward ($\times 2$) & 0.0586 $\pm$ 0.0050 & 0.01 $\pm$ 0.02 & --- & 0.003 $\pm$ 0.004 & 0.14 $\pm$ 0.20 \\
Feedforward ($\times 3$) & 0.0441 $\pm$ 0.0065 & 0.11 $\pm$ 0.10 & --- & 0.021 $\pm$ 0.019 & 0.28 $\pm$ 0.20 \\
Feedforward ($\times 4$) & 0.0295 $\pm$ 0.0016 & 0.42 $\pm$ 0.35 & --- & 0.076 $\pm$ 0.061 & 0.42 $\pm$ 0.35 \\
Feedforward ($\times 5$) & 0.0255 $\pm$ 0.0022 & 0.58 $\pm$ 0.41 & --- & 0.110 $\pm$ 0.078 & 0.71 $\pm$ 0.53 \\
\midrule
DDPM / matched & 0.0576 $\pm$ 0.0047 & 0.03 $\pm$ 0.01 & --- & 0.006 $\pm$ 0.002 & 0.00 $\pm$ 0.00 \\
DDIM / matched & 0.0732 $\pm$ 0.0055 & 0.01 $\pm$ 0.01 & --- & 0.003 $\pm$ 0.002 & 0.00 $\pm$ 0.00 \\
DIS / matched & 0.9315 $\pm$ 0.0007 & 0.00 $\pm$ 0.00 & --- & 0.000 $\pm$ 0.000 & 0.00 $\pm$ 0.00 \\
DDPM / large & 0.1197 $\pm$ 0.0040 & 0.01 $\pm$ 0.01 & --- & 0.003 $\pm$ 0.002 & 0.00 $\pm$ 0.00 \\
DDIM / large & 0.1317 $\pm$ 0.0030 & 0.00 $\pm$ 0.00 & --- & 0.000 $\pm$ 0.000 & 0.00 $\pm$ 0.00 \\
DIS / large & 0.0463 $\pm$ 0.0003 & 0.02 $\pm$ 0.02 & --- & 0.004 $\pm$ 0.003 & 0.00 $\pm$ 0.00 \\
\bottomrule
\end{tabular}
\end{table}

\textbf{Results.}
Our model achieves $83.77\%$ candidate hits and $3.506$ accepted
prediction clusters per graph on average, with at least two on $41.67\%$
of graphs (Table~\ref{chem:main-results}). Numerical settling increases
from $91.99\%$ at $T=200$ to $96.15\%$ at $T=5{,}000$; remaining
oscillations and unresolved cases appear in
Table~\ref{tab:chemistry-extended-dynamics}. Single-target supervision reaches only $18.78\%$ hits; IGNNs settle in
their hidden-state operators but yield no accepted predictions,
illustrating that convergence alone does not ensure validity.
Untied depth improves accuracy, but five independent stages reach only
$0.58\%$ hits; tested distributional models also have low hit rates.
These results demonstrate multiple residual-qualified predictions and high numerical settling rates. They do not certify distinct exact roots or complete root coverage. Baselines differ in objectives and inference budgets, so the comparison does not isolate recurrence at equal computational cost.

%\vspace{-2mm}
\section{Conclusions}
%\vspace{-2mm}

Multi-attractor dynamics provide a mechanism for representing set-valued graph tasks with recurrent, weight-tied GNNs. Under the stated assumptions, one shared update can converge arbitrarily close to different valid solutions, while continuous, equivariant contraction can obstruct recovery of even one. Across three scientific tasks, energy or residual training produces multiple high-quality predictions and high numerical convergence rates without solution labels. These findings motivate preserving access to multiple equilibria when tasks admit multiple answers. This general framework opens opportunities across graph applications, with success depending on problem-specific objectives and optimization. Understanding how training shapes convergence and solution coverage remains an important direction.

\bibliographystyle{plainnat}
\bibliography{references}

\newpage

\appendix

\section*{Appendix Contents}
\startcontents[appendix]
\printcontents[appendix]{}{1}[2]{}

\section{Ideal Operator Existence}
\label{app:ideal}

In the appendices, $S(G)$ and $\mathcal Y_n$ denote $\mathcal S(G)$ and $Y_n$ from the main text. Fix the number of nodes $n$. We represent a graph instance as
\[
    G=(A,X)\in\mathcal E_n
    :=
    \R^{n\times n\times d_e}
    \times
    \R^{n\times d_v},
\]
where
\[
    A=(A^{(1)},\ldots,A^{(d_e)})
\]
contains the edge features and $X$ contains the node features. The output space
is
\[
    \mathcal Y_n=\R^{n\times q}\cong\R^m.
\]
Both spaces are equipped with their Euclidean norms.

For a permutation $\pi$ of the $n$ nodes, let $P_\pi$ denote its permutation
matrix, with $P_\pi e_i=e_{\pi(i)}$. Relabeling the nodes of a graph and an
output is defined by
\[
    \pi G
    :=
    \left(
        \bigl(P_\pi A^{(r)}P_\pi^\top\bigr)_{r=1}^{d_e},
        P_\pi X
    \right),
    \qquad
    \pi y:=P_\pi y.
\]
Since permutation matrices are orthogonal, these relabelings preserve Euclidean
distances.

Throughout, $\mathcal G\subseteq\mathcal E_n$ is closed under node relabeling:
\[
    G\in\mathcal G
    \quad\Longrightarrow\quad
    \pi G\in\mathcal G
\]
for every node permutation $\pi$. The domain $\mathcal G$ need not be bounded,
closed, connected, or compact.

\subsection{Basic concepts and main result}

\begin{definition}[Permutation-equivariant finite-branch solution map]
\label{def:equivariant-branches}
A finite-branch solution map on $\mathcal G$ is represented, for some integer
$M\ge 1$, by labeled branches
\[
    f_1,\ldots,f_M:\mathcal G\to\mathcal Y_n
\]
through
\[
    S(G):=\{f_1(G),\ldots,f_M(G)\}.
\]
Distinct branch indices are allowed to have coincident values. Since
$S(G)$ is a set, if $f_i(G)=f_j(G)$ for $i\ne j$, their common value
appears only once in $S(G)$.

We say that this representation is permutation equivariant if, for every node
permutation $\pi$, there exists a permutation $\tau_\pi$ of the branch indices
$\{1,\ldots,M\}$ such that
\[
    f_{\tau_\pi(i)}(\pi G)
    =
    \pi f_i(G)
    \qquad
    \text{for every }G\in\mathcal G
    \text{ and }i\in\{1,\ldots,M\}.
\]
Thus, relabeling the nodes may relabel the solution branches, but it does not
change the underlying solution set. In particular,
\[
    S(\pi G)
    =
    \pi S(G)
    :=
    \{\pi s:s\in S(G)\}.
\]
\end{definition}

\begin{definition}[Regular and stable graph instances]
\label{def:regular-stable}
For each branch, let
\[
    D_i:=\{G\in\mathcal G:f_i\text{ is not locally Lipschitz at }G\},
    \qquad
    D:=\bigcup_{i=1}^M D_i.
\]
A graph $G_0$ has a \emph{stable collapse pattern} if some relatively open
neighborhood $V$ of $G_0$ admits a partition
\[
    \{1,\ldots,M\}=C_1\mathbin{\dot\cup}\cdots\mathbin{\dot\cup}C_r
\]
such that, for every $G\in V$,
\[
    f_i(G)=f_j(G)
    \quad\Longleftrightarrow\quad
    i,j\in C_\alpha
    \text{ for some }\alpha.
\]
Let $U$ be the set of graphs without a stable collapse pattern and define
\[
    \mathcal G_\star:=\mathcal G\setminus(D\cup U).
\]
We call $\mathcal G_\star$ the \emph{regular stable domain}. Both $D$ and $U$
are relatively closed in $\mathcal G$. No emptiness, finiteness, or
measure-zero assumption is imposed on $D$.
\end{definition}

\noindent\textbf{Remark (Equivariant branch labels).}
At the level of the underlying set-valued map, the fixed label
permutations in Definition~\ref{def:equivariant-branches} impose no
additional restriction. Indeed, given any finite representation
$S(G)=\{f_1(G),\ldots,f_M(G)\}$ satisfying $S(\pi G)=\pi S(G)$, define
\[
    \widetilde f_{(i,\sigma)}(G):=\sigma f_i(\sigma^{-1}G),
    \qquad
    \widetilde\tau_\pi(i,\sigma):=(i,\pi\sigma),
    \qquad \sigma\in\mathfrak S_n.
\]
These $Mn!$ labels represent the same solution set and satisfy
Definition~\ref{def:equivariant-branches}. If $D$ and $\mathcal G_\star$
are computed from the original labels, the symmetrized representation has
\[
    \widetilde D=\bigcup_{\sigma\in\mathfrak S_n}\sigma D,
    \qquad
    \widetilde{\mathcal G}_\star
    =\bigcap_{\sigma\in\mathfrak S_n}\sigma\mathcal G_\star.
\]
For the second identity, a stable partition restricts to each
$\sigma$-block; conversely, local continuity and separation of the
distinct solution values make matches between these blocks locally
constant. Thus symmetrization may shrink the regular stable domain,
and probability assumptions on that domain must be rechecked. If the
original representation already satisfies
Definition~\ref{def:equivariant-branches}, both domains are invariant
by Lemma~\ref{lem:bad-set-invariance} and remain unchanged.

\begin{definition}[Switching set and strict basins]
\label{def:switching}
For $G\in\mathcal G_\star$, the \emph{reduced switching set} is
\[
    \Sigma_G^\circ
    :=
    \left\{
        y\in\mathcal Y_n:
        \#\operatorname*{argmin}_{s\in S(G)}\|y-s\|\ge2
    \right\},
\]
where coincident branch values are counted only once. For $s\in S(G)$, its
strict Voronoi basin is
\[
    V_s(G)
    :=
    \{y\in\mathcal Y_n:
      \|y-s\|<\|y-r\|
      \text{ for every }r\in S(G)\setminus\{s\}\}.
\]
\end{definition}

\begin{theorem}[Jointly globally Lipschitz equivariant multi-attractor dynamics]
\label{thm:joint-multi-attractor}
Let $S$ be an equivariant finite-branch solution map in the sense of
Definition~\ref{def:equivariant-branches}, with regular stable domain and
switching geometry given by Definitions~\ref{def:regular-stable}
and~\ref{def:switching}. Then there exists
\[
    T:\mathcal G\times\mathcal Y_n\to\mathcal Y_n,
    \qquad (G,y)\mapsto T_G(y),
\]
with the following properties.

\begin{enumerate}[leftmargin=5mm]
    \item \textbf{Regularity and symmetry.}
    The map $T$ is jointly globally Lipschitz and permutation equivariant. In
    particular, for some $L_T<\infty$,
    \begin{equation}
        \|T_G(y)-T_H(z)\|
        \le
        L_T\sqrt{\|G-H\|^2+\|y-z\|^2},
        \qquad
        T_{\pi G}(\pi y)=\pi T_G(y).
        \label{eq:ideal-joint-regularity}
    \end{equation}

    \item \textbf{Exceptional configurations.}
    One may choose $T_G(y)=y$ whenever $G\in D\cup U$, or whenever
    $G\in\mathcal G_\star$ and $y\in\Sigma_G^\circ$.

    \item \textbf{Basin-wise convergence.}
    Fix $G\in\mathcal G_\star$ and $y_0\notin\Sigma_G^\circ$, and let
    \[
        s:=\operatorname*{argmin}_{r\in S(G)}\|y_0-r\|.
    \]
    Then $y_{t+1}=T_G(y_t)$ converges to $s$, and there is a
    trajectory-dependent constant $c(G,y_0)\in(0,1)$ such that
    \begin{equation}
        \|y_t-s\|
        \le
        \bigl(1-c(G,y_0)\bigr)^t\|y_0-s\|.
        \label{eq:ideal-geometric-convergence}
    \end{equation}
    No uniform positive lower bound on $c(G,y_0)$ is claimed.

    \item \textbf{Almost-everywhere recovery.}
    For every $G\in\mathcal G_\star$, $\mathcal L^m(\Sigma_G^\circ)=0$. 
    Let $\mu$ be a Borel probability measure on $\mathcal Y_n$ with $\mu\ll\mathcal L^m$, and draw $y_0\sim\mu$. Then the iteration $y_t$ converges to $S(G)$ for $\mu$-almost every initialization. If in addition $\mu(V_s(G))>0$ for every $s\in S(G)$, then
\begin{equation}
\supp\mathcal L(y_\infty\mid G)=S(G),
    \label{eq:ideal-full-support}
\end{equation}
    where $y_\infty$ denotes the almost-sure iteration limit.
\end{enumerate}
\end{theorem}

\subsection{Bad set and displacement functions}

Before the proof, we state some definitions and lemmas.

\begin{definition}[Joint bad set]
\label{def:joint-bad-set}
Let $Z:=\mathcal G\times\mathcal Y_n$ carry the Euclidean product metric, and
define
\[
    \boldsymbol\Sigma^\circ
    :=
    \{(G,y):G\in\mathcal G_\star,\ y\in\Sigma_G^\circ\}.
\]
The \emph{joint bad set} and its complement are
\begin{equation}
    B
    :=
    \bigl((D\cup U)\times\mathcal Y_n\bigr)
    \cup
    \overline{\boldsymbol\Sigma^\circ}^{\,Z},
    \qquad
    \Omega:=Z\setminus B.
    \label{eq:ideal-joint-bad-set}
\end{equation}
Here, $B$ contains singular and unstable graph instances and the
closed joint switching graph.
\end{definition}

\begin{definition}[Selector and normalized displacement]
\label{def:selector-displacement}
Define $d:Z\to[0,1]$ by
\[
d(G,y):=
\begin{cases}
\min\{1,\dist((G,y),B)\}, & B\ne\varnothing,\\
1, & B=\varnothing.
\end{cases}
\]
For $(G,y)\in\Omega$, the nearest distinct solution is unique; define the branch selector $P$ and normalized displacement $H$ as follows:
\begin{equation}
P(G,y):=\operatorname*{argmin}_{s\in S(G)}\|y-s\|,
\qquad
H(G,y):=\frac{P(G,y)-y}{d(G,y)}.
    \label{eq:ideal-selector-displacement}
\end{equation}
\end{definition}

First, let's consider the invariance/equivariance properties of the above definitions.

\begin{lemma}
\label{lem:bad-set-invariance}
The bad set $B$ and its complement $\Omega$ are both permutation invariant.
\end{lemma}

\begin{proof}
Fix a node permutation $\pi$. Its action on $\mathcal G$ is an
isometric bijection and therefore maps relatively open neighborhoods
to relatively open neighborhoods. Definition~\ref{def:equivariant-branches}
gives
\[
    f_{\tau_\pi(i)}=\pi\circ f_i\circ\pi^{-1}.
\]
Consequently, $f_i$ is locally Lipschitz at $G$ if and only if
$f_{\tau_\pi(i)}$ is locally Lipschitz at $\pi G$.
Since $\tau_\pi$ is a bijection of the labels, $\pi D=D$.
If $V$ and $\{C_\alpha\}$ give a stable collapse pattern at $G_0$,
then $\pi V$ and $\{\tau_\pi(C_\alpha)\}$ give one at $\pi G_0$.
Applying the same argument to $\pi^{-1}$ yields $\pi U=U$.
In particular, $\pi\mathcal G_\star=\mathcal G_\star$, and
\[
    \pi((D\cup U)\times\mathcal Y_n)=(D\cup U)\times\mathcal Y_n.
\]

Moreover, solution-set equivariance and distance preservation imply $\pi(\boldsymbol\Sigma^\circ) =\boldsymbol\Sigma^\circ$. 
Indeed, let $(G,y)\in\boldsymbol\Sigma^\circ$. Then there exist distinct
$s_1,s_2\in S(G)$ such that
\[
    \|y-s_1\|=\|y-s_2\|=\dist(y,S(G)).
\]
Since $S(\pi G)=\pi S(G)$ and the permutation action is an isometry,
$\pi s_1$ and $\pi s_2$ are distinct nearest elements of $S(\pi G)$ to
$\pi y$. Since $\pi G\in\mathcal G_\star$, it follows that
\[
    \pi(G,y)=(\pi G,\pi y)\in\boldsymbol\Sigma^\circ.
\]
Thus
\[
    \pi(\boldsymbol\Sigma^\circ)
    \subseteq\boldsymbol\Sigma^\circ.
\]
Applying the same argument to $\pi^{-1}$ gives the reverse inclusion and
therefore
\[
    \pi(\boldsymbol\Sigma^\circ)
    =\boldsymbol\Sigma^\circ.
\]
Since the permutation mapping is a homeomorphism,
\[
    \pi\left(
        \overline{\boldsymbol\Sigma^\circ}^{\,Z}
    \right)
    =
    \overline{
        \pi(\boldsymbol\Sigma^\circ)
    }^{\,Z}
    =
    \overline{\boldsymbol\Sigma^\circ}^{\,Z}.
\]
Using $B=((D\cup U)\times\mathcal Y_n) \cup\overline{\boldsymbol\Sigma^\circ}^{\,Z}$
now gives $\pi B=B$. In addition, 
\[
    \pi\Omega
    =
    \pi(Z\setminus B)
    =
    \pi Z\setminus\pi B
    =
    Z\setminus B
    =
    \Omega.
\]
which shows the invariance of $\Omega$ and finishes the proof.
\end{proof}

\begin{lemma}
\label{lem:selector-equivariance}
Function $d$ is permutation invariant, and $P$ and $H$ are permutation equivariant.
\end{lemma}

\begin{proof}
Since the permutation action is an isometry and $\pi B=B$,
\[
    d(\pi G,\pi y)=d(G,y).
\]
Let $s_\star=P(G,y)$. By solution-set equivariance,
\[
    S(\pi G)=\pi S(G).
\]
Thus every $\widetilde s\in S(\pi G)$ has the form $\widetilde s=\pi s$
for some $s\in S(G)$. Since permutations preserve distances,
\[
    \|\pi y-\widetilde s\|
    =\|y-s\|
    \ge \|y-s_\star\|
    =\|\pi y-\pi s_\star\|.
\]
Hence $\pi s_\star$ is a nearest solution to $\pi y$. Its uniqueness on
$\Omega$ gives
\[
    P(\pi G,\pi y)=\pi P(G,y).
\]
Finally, by Definition \ref{def:selector-displacement},
\[
    H(\pi G,\pi y)=\frac{P(\pi G,\pi y)-\pi y}{d(\pi G,\pi y)} =\pi\frac{P(G,y)-y}{d(G,y)}
    =\pi H(G,y),
\]
which finishes the proof.
\end{proof}

To further establish the properties of displacement $H$, we state and use the following lemma.

\begin{lemma}[{\cite[Theorem~A.4]{liu2026expressive}}]
\label{lem:flattening}
Let $X$ be an arbitrary subset of a Euclidean space, equipped with
the induced metric, and let $h:X\to\mathbb R$ be locally Lipschitz
intrinsically on $X$. Then there exists a function $\varepsilon:X\to(0,1)$ such that both $\varepsilon$ and $\varepsilon h$ are globally Lipschitz on $X$.
\end{lemma}

Here intrinsic local Lipschitzness means Lipschitzness on a relative
neighborhood $X\cap B(x,r)$ of each $x\in X$.
The cited theorem requires no openness, closedness, boundedness, or
local compactness of $X$, so it applies to $X=\Omega$.

We first extend this result to the vector-valued and equivariant form required below. Consider a general group setting where $\Gamma$ is a finite group with orthogonal representations:

\[
\rho_X:\Gamma\to \mathrm{O}(p),     
\qquad     
\rho_Y:\Gamma\to \mathrm{O}(m).
\]

For brevity, given $\gamma\in\Gamma$, $x\in\R^p$, and $v\in\R^m$, we write:
\[
\gamma x:=\rho_X(\gamma)x,     \qquad     \gamma v:=\rho_Y(\gamma)v.
\]

The previously defined permutation $\pi$ is an example of such an action; it applies to both $G$ and $y$, acting as an orthogonal transformation on each domain.

\begin{lemma}[Equivariant bounded flattening]
\label{lem:equivariant-flattening}
Let $X\subseteq\R^p$ be $\Gamma$-invariant.
Suppose that $H:X\to\R^m$ is locally Lipschitz and $\Gamma$-equivariant.
Then there exists a $\Gamma$-invariant function $\beta:X\to(0,1)$ such that both $\beta$ and $\beta H$ are bounded and globally Lipschitz on $X$.
Moreover, $\beta H$ is $\Gamma$-equivariant.
\end{lemma}

\begin{proof}
Write
\[
    H=(H_1,\ldots,H_m).
\]
For every coordinate $H_j$, apply Lemma \ref{lem:flattening} to obtain a function
\[
    \varepsilon_j:X\to(0,1)
\]
such that $\varepsilon_j$ and
\[
    u_j:=\varepsilon_jH_j
\]
are globally Lipschitz. Define
\[
    a_j(x)
    :=
    \frac{\varepsilon_j(x)}{1+|u_j(x)|}.
\]
The scalar map $r\mapsto(1+|r|)^{-1}$ is bounded and globally Lipschitz. Since $\varepsilon_j$ is also bounded and globally Lipschitz, $a_j$ is bounded and globally Lipschitz, and
\[
    0<a_j<1.
\]
Moreover,
\[
    a_jH_j
    =
    \frac{u_j}{1+|u_j|}.
\]
Because $r\mapsto r/(1+|r|)$ is bounded and globally Lipschitz, $a_jH_j$ is bounded and globally Lipschitz as well.
Set
\[
    a_0(x):=\prod_{j=1}^m a_j(x).
\]
A finite product of bounded globally Lipschitz functions is bounded and globally Lipschitz. For every coordinate $j$,
\[
    a_0H_j
    =
    (a_jH_j)\prod_{k\ne j}a_k,
\]
which is again bounded and globally Lipschitz. Hence $a_0H$ is bounded and globally Lipschitz.

Finally, define
\[
    \beta(x):=\prod_{\gamma\in\Gamma}a_0(\gamma x).
\]
Since $X$ is $\Gamma$-invariant, each $a_0\circ\gamma$ is well defined
on $X$. Since the representation $\rho_X$ is orthogonal,
\[
    \Lip(a_0\circ\gamma)\le\Lip(a_0),
    \qquad
    \|a_0\circ\gamma\|_\infty\le\|a_0\|_\infty.
\]
The finiteness of $\Gamma$ therefore implies that $\beta$ is bounded
and globally Lipschitz.

For every $\gamma_0\in\Gamma$,
\[
        \beta(\gamma_0x) =
        \prod_{\gamma\in\Gamma}a_0(\gamma\gamma_0x) =
        \prod_{\gamma\in\Gamma}a_0(\gamma x)
        =\beta(x),
\]
where the second equality follows because right multiplication by $\gamma_0$ permutes the elements of $\Gamma$. Thus $\beta$ is $\Gamma$-invariant. Furthermore,
\[
    \beta(x)H(x)
    =
    a_0(x)H(x)
    \prod_{\substack{\gamma\in\Gamma\\\gamma\ne e}}
    a_0(\gamma x),
\]
which is a finite product of bounded globally Lipschitz factors. Thus $\beta H$ is bounded and globally Lipschitz. Moreover, the equivariance of $\beta H$ follows from the invariance of $\beta$ and equivariance of $H$. This proves the lemma.
\end{proof}

\subsection{Proof of the main result}

\begin{proof}[Proof of Theorem~\ref{thm:joint-multi-attractor}]
We construct $T$ and verify the four asserted properties.

\medskip
\noindent
\textbf{Step 1: The set $B$ is closed and permutation-invariant.}

For each branch, the locus on which it is locally Lipschitz is relatively open. Therefore $D$ is relatively closed in $\mathcal G$. The set of points having a stable collapse pattern is also relatively open by definition, so $U$ is relatively closed. It follows that
\[
    (D\cup U)\times\mathcal Y_n
\]
is relatively closed in $Z$. The set $\overline{\boldsymbol\Sigma^\circ}^{\,Z}$ is closed by definition. Hence $B$ is closed in $Z$.

In addition, Lemma \ref{lem:bad-set-invariance} provides the permutation-invariance of the bad set $B$ and its complement.

\medskip
\noindent
\textbf{Step 2: The nearest-solution selector is locally Lipschitz on $Z\setminus B$.}

Fix $(G_0,y_0)\in\Omega$. Then $G_0\notin D\cup U$, the collapse
pattern is stable near $G_0$, and $y_0$ has a unique nearest distinct
branch value. Let $C_\star$ be the stable cluster realizing this nearest
value, and choose any representative $i_\star\in C_\star$. Branches in
$C_\star$ agree identically throughout some neighborhood of $G_0$.

If $C_\star$ is the only stable cluster, then all branch values agree
throughout that neighborhood, and hence
\[
    P(G,y)=f_{i_\star}(G)
\]
there; the local Lipschitz conclusion follows immediately.

Otherwise, choose one representative $i_\alpha$ from every other stable
cluster. Since the nearest branch value at $(G_0,y_0)$ is unique, the
margin
\[
    \delta
    :=
    \min_{\alpha\ne\star}
    \left(
        \|y_0-f_{i_\alpha}(G_0)\|
        -
        \|y_0-f_{i_\star}(G_0)\|
    \right)
\]
is strictly positive. All regular branches are continuous near $G_0$. After shrinking the neighborhood of $(G_0,y_0)$ if necessary, $C_\star$ remains the unique nearest cluster throughout that neighborhood. Hence
\[
    P(G,y)=f_{i_\star}(G)
\]
there. Because $G_0\notin D$, the representative branch $f_{i_\star}$ is locally Lipschitz near $G_0$. Thus $P$ is jointly locally Lipschitz near $(G_0,y_0)$. Since $(G_0,y_0)$ was arbitrary,
\[
    P:\Omega\to\mathcal Y_n
\]
is jointly locally Lipschitz. This is the same local mechanism used in \cite[Lemma~A.12]{jore2026bifurcation}, but no boundedness or compactness of the graph domain is required here.

\medskip
\noindent
\textbf{Step 3: Construction of a damping factor that removes the singularities.}

Since $B$ is relatively closed in $Z$, for every $z\in Z\setminus B$ we have $\dist(z,B)>0$; otherwise $z\in\overline{B}^{\,Z}=B$. Hence it holds that
\[
    d(z)>0\quad\text{for }z\in\Omega,
    \qquad
    d(z)=0\quad\text{for }z\in B.
\]
Moreover, $d$ is bounded, $1$-Lipschitz, and permutation-invariant. 
Since $d$ is $1$-Lipschitz and strictly positive on $\Omega$, it is locally bounded away from zero there. Hence $d^{-1}$ is locally Lipschitz on $\Omega$. Together with the local Lipschitzness of $P$, this implies that $H=(P-y)/d$ is locally Lipschitz.
In addition, thanks to Lemma \ref{lem:selector-equivariance}, $H$ is permutation equivariant.

Apply Lemma \ref{lem:equivariant-flattening} on $\Omega$ and permutation action group. There exists a permutation-invariant function
\[
    \beta:\Omega\to(0,1)
\]
such that
\[
    A:=\beta H
\]
is bounded and globally Lipschitz on $\Omega$. Define
\[
    q(G,y)
    :=
    \begin{cases}
        d(G,y)A(G,y),&(G,y)\in\Omega,\\
        0,&(G,y)\in B.
    \end{cases}
\]
For $(G,y)\in\Omega$,
\[
    q(G,y)
    =
    \beta(G,y)\bigl(P(G,y)-y\bigr).
\]
Finally, define
\[
    T_G(y):=y+q(G,y).
\]
Equivalently,
\begin{equation}
    T_G(y)
    =
    \begin{cases}
        (1-\beta(G,y))y+\beta(G,y)P_G(y),
        &(G,y)\notin B,\\[1mm]
        y,&(G,y)\in B.
    \end{cases}
    \label{eq:ideal-operator-construction}
\end{equation}
The freezing property on $B$ is immediate from this definition.

\medskip
\noindent
\textbf{Step 4: Joint global Lipschitz continuity of $T$.}

Let
\[
    \|A\|_\infty\le M_A,
    \qquad
    \Lip(A)\le L_A.
\]
For any $z,z'\in\Omega$,
\[
    \begin{aligned}
        \|q(z)-q(z')\|
        &=
        \|d(z)A(z)-d(z')A(z')\|\\
        &\le
        |d(z)-d(z')|\,\|A(z)\|
        +
        d(z')\|A(z)-A(z')\|\\
        &\le
        (M_A+L_A)\|z-z'\|,
    \end{aligned}
\]
where we used $0\le d\le1$ (This follows from Definition \ref{def:selector-displacement}) and the $1$-Lipschitz property of $d$.

If $z\in\Omega$ and $z'\in B$, then
\[
        \|q(z)-q(z')\|
        =\|q(z)\|
        =d(z)\|A(z)\|
        \le M_A\dist(z,B)
        \le M_A\|z-z'\|.
\]
If both points lie in $B$, the difference is zero. Thus $q$ is globally Lipschitz on all of $Z$. Since the coordinate projection $(G,y)\mapsto y$ is $1$-Lipschitz and
\[
    T_G(y)=y+q(G,y),
\]
$T$ is jointly globally Lipschitz. In particular, one may take
\[
    L_T\le 1+M_A+L_A.
\]
\medskip
\noindent
\textbf{Step 5: Permutation equivariance.}

The functions $d$ and $\beta$ are invariant, whereas $H$ is equivariant. Therefore
\[
    A(\pi G,\pi y)=\pi A(G,y)
\]
on $\Omega$. Since $B$ is invariant, it follows that
\[
    q(\pi G,\pi y)=\pi q(G,y)
\]
on all of $Z$. Consequently,
\[
T_{\pi G}(\pi y)=\pi y+q(\pi G,\pi y)=\pi y+\pi q(G,y)=\pi T_G(y).
\]

\medskip
\noindent
\textbf{Step 6: On a regular stable graph, taking the closure does not enlarge the switching section.}

Fix $G\in\mathcal G_\star$. We claim that
\begin{equation}
    \{y:(G,y)\in B\}=\Sigma_G^\circ.
    \label{eq:ideal-bad-set-section}
\end{equation}
The inclusion ``$\supseteq$'' follows immediately from the definition of the closure. For the reverse inclusion, it suffices to show that any point in the $G$-section of $\overline{\boldsymbol\Sigma^\circ}^{\,Z}$ belongs to $\Sigma_G^\circ$. Suppose that
\[
    (G_k,y_k)\in\boldsymbol\Sigma^\circ,
    \qquad
    (G_k,y_k)\to(G,y).
\]
For every $k$, choose two labels whose distinct branch values attain the common minimum distance from $y_k$. Because there are only finitely many pairs of labels, after passing to a subsequence we may assume that these labels are fixed, say $i\ne j$.

Since $G$ is a stable-collapse point, if $i$ and $j$ belonged to the same stable cluster, then the two branches would agree identically throughout a neighborhood of $G$. For all sufficiently large $k$, this would contradict the fact that they represent distinct branch values at $G_k$. Therefore $i$ and $j$ belong to different stable clusters, and in particular
\[
    f_i(G)\ne f_j(G).
\]
Because $G\notin D$, every branch is continuous at $G$. Passing to the limit in the equal-minimum relations gives
\[
    \|y-f_i(G)\|
    =
    \|y-f_j(G)\|
    =
    \min_\ell\|y-f_\ell(G)\|.
\]
Thus $y\in\Sigma_G^\circ$, proving the reverse inclusion. 

\medskip
\noindent
\textbf{Step 7: Invariance of the strict Voronoi segment.}

Fix
\[
    G\in\mathcal G_\star,
    \qquad
    y_0\notin\Sigma_G^\circ,
\]
and let
\[
    s=P_G(y_0).
\]
For any other distinct solution $r\in S(G)\setminus\{s\}$,
\[
    \|y_0-s\|<\|y_0-r\|.
\]
For $a\in[0,1]$, define
\[
    z_a=(1-a)y_0+as.
\]
Then
\[
    \|z_a-s\|=(1-a)\|y_0-s\|.
\]
On the other hand, the reverse triangle inequality yields
\[
    \begin{aligned}
        \|z_a-r\|
        &\ge
        \|y_0-r\|-\|z_a-y_0\|\\
        &=
        \|y_0-r\|-a\|y_0-s\|\\
        &>
        (1-a)\|y_0-s\|\\
        &=
        \|z_a-s\|.
    \end{aligned}
\]
Therefore
\[
    P_G(z_a)=s
    \qquad\text{and}\qquad
    z_a\notin\Sigma_G^\circ
\]
for every $a\in[0,1]$. Hence
\[
    P_G(z)=s
    \qquad
    \text{for every }z\in[y_0,s],
\]
and Step~6 implies
\[
    \{G\}\times[y_0,s]\subseteq\Omega.
\]

\medskip
\noindent
\textbf{Step 8: Convergence.}

The set $\{G\}\times[y_0,s]$ is a compact subset of $\Omega$. Since $\beta$ is continuous and strictly positive on $\Omega$,
\[
    c(G,y_0)
    :=
    \min_{z\in[y_0,s]}\beta(G,z)
    >0.
\]
Furthermore, $0<\beta<1$. If $y_t\in[y_0,s]$, then Step~7 gives $P_G(y_t)=s$, and therefore
\[
y_{t+1}=(1-\beta(G,y_t))y_t+\beta(G,y_t)s\in[y_t,s]\subseteq[y_0,s].
\]
By induction, every iterate remains on $[y_0,s]$. Moreover,
\begin{equation}
    y_{t+1}-s
    =
    \bigl(1-\beta(G,y_t)\bigr)(y_t-s).
    \label{eq:ideal-radial-recurrence}
\end{equation}
It follows that
\[
    \|y_{t+1}-s\|
    \le
    \bigl(1-c(G,y_0)\bigr)\|y_t-s\|.
\]
Iterating this estimate gives
\[
    \|y_t-s\|
    \le
    \bigl(1-c(G,y_0)\bigr)^t\|y_0-s\|.
\]
Hence $y_t\to s$.

\medskip
\noindent
\textbf{Step 9: Almost-everywhere convergence and full support.}

Fix $G\in\mathcal G_\star$. For any two distinct solutions $s,r\in S(G)$, the equidistance set
\[
    \{y:\|y-s\|=\|y-r\|\}
\]
is an affine hyperplane, because the defining equality is equivalent to
\[
    2\langle y,r-s\rangle
    =
    \|r\|^2-\|s\|^2.
\]
Therefore
\[
    \Sigma_G^\circ
    \subseteq
    \bigcup_{\substack{s,r\in S(G)\\s\ne r}}
    \{y:\|y-s\|=\|y-r\|\}
\]
is contained in a finite union of Lebesgue-null affine hyperplanes. Thus
\[
    \mathcal L^m(\Sigma_G^\circ)=0.
\]
If $\mu\ll\mathcal L^m$, then
\[
    \mu(\Sigma_G^\circ)=0.
\]
Step~8 therefore shows that the iteration converges to a point of $S(G)$ for $\mu$-almost every initialization.

Finally, if every strict Voronoi basin has positive $\mu$-mass, then for every $s\in S(G)$,
\[
    \Pr(y_\infty=s\mid G)
    =
    \mu(V_s(G))
    >0.
\]
Since $S(G)$ is finite and $y_\infty\in S(G)$ almost surely,
\[
    \supp\mathcal L(y_\infty\mid G)=S(G).
\]
This proves all the conclusions.
\end{proof}

\section{Family-wide asymptotic realization by a recurrent MP-GNN}
\label{app:init-only-family-wide-mpgnn}

We now realize the ideal dynamics of Appendix~\ref{app:ideal} using one
message-passing GNN block that is reused at every recurrent step:
\begin{equation}
  y_{t+1}=\operatorname{GNN}(G,y_t).
  \label{eq:family-init-cell-recurrence}
\end{equation}
We suppress the parameter subscript $\theta$ in this appendix. The random initialization supplies both the initial node addresses and the
choice of solution basin. At later steps, the block receives only the graph
and the current state. Throughout this appendix, write $d_y:=q\geq2$ and
$m=nd_y$, using the graph and output spaces of Appendix~\ref{app:ideal}.

\subsection{Architecture and main result}
\label{sec:family-architecture-result}

\textbf{The recurrent block.}
For one call, write $y^+=\operatorname{GNN}(G,y)$. Choose a finite number
of hidden layers $L$, finite hidden widths, and a continuous encoder
\[
  \operatorname{enc}:\mathbb R^{d_v}\times\mathbb R^{d_y}
  \longrightarrow\mathbb R^{d_0}.
\]
The block has the form
\begin{equation}
  \begin{aligned}
  h_i^{(0)}&=\operatorname{enc}(X_i,y_i),
\\
  m_i^{(\ell)}
  &=\sum_{j=1}^n
    M_\ell\!\left(h_i^{(\ell)},h_j^{(\ell)},A_{ji}\right),
\\
  h_i^{(\ell+1)}
  &=U_\ell\!\left(h_i^{(\ell)},m_i^{(\ell)}\right),
  \qquad 0\leq\ell<L,
\\
  g^{(L)}&=\sum_{j=1}^n R_L\!\left(h_j^{(L)}\right),
\\
  y_i^+&=O\!\left(h_i^{(L)},g^{(L)}\right).
  \end{aligned}
  \label{eq:family-init-mp-initialization}
\end{equation}
Here the aggregation runs over all $n$ senders, including $j=i$,
since the edge tensor $A$ assigns a feature vector to every ordered pair.
The component maps, $\{M_\ell, U_\ell\}_{\ell=0}^{L-1}, R_L, O$, are arbitrary continuous maps between finite-dimensional Euclidean spaces and are shared across nodes. The only global operation is the final invariant sum: no global aggregate is formed or broadcast in the hidden layers. The entire $L$-layer block, including its encoder and readout, is denoted as $\operatorname{GNN}(G,\cdot)$ and is reused with the same parameters at every recurrent time. Hidden representations are recomputed within each call; only $y_t$ is carried to the next call.

\noindent\textbf{Remark (Sparse graphs).}
For an edge set $E(G)\subseteq[n]\times[n]$, encode
$A_{ji}=(e_{ji},a_{ji})$, where
$e_{ji}=\mathbf 1_{\{(j,i)\in E(G)\}}$ and $a_{ji}=0$
whenever $e_{ji}=0$. The class of such tensors is closed under
relabeling, so Appendix~\ref{app:ideal} applies to any
permutation-invariant domain of these graphs.
The block constructed in
Lemma~\ref{lem:family-wide-realization-initialization} has
$M_0=c_i\otimes c_j\otimes A_{ji}$, which vanishes on every
absent pair, even outside the compact realization family.
Hence its all-pairs message sum agrees exactly with the neighbor
sum over $(j,i)\in E(G)$. On the compact realization family,
$\mathcal R_{\rm edge}$ records $(1,a_{ji})$ in present-edge
slots and zero in absent-edge slots. Consequently, the block in
Theorem~\ref{thm:family-wide-convergent-initialization} can use
neighbor-only hidden message passing followed by the
readout in \eqref{eq:family-init-mp-initialization}.

It's straightforward that reindexing the message and readout sums gives the equivariance:
\begin{equation}
  \operatorname{GNN}(\pi G,\pi y)
  =\pi\operatorname{GNN}(G,y).
  \label{eq:family-seed-free-equivariance}
\end{equation}

\textbf{Facts about the ideal update.}
Fix the operator $T$ constructed in Theorem~\ref{thm:joint-multi-attractor}. Let $\Omega$ and $P$ be as in Definitions~\ref{def:joint-bad-set} and~\ref{def:selector-displacement}. The proof in Appendix~\ref{app:ideal} gives
\begin{equation}
  T_G(y)=(1-\beta(G,y))y+\beta(G,y)P(G,y),
  \qquad 0<\beta(G,y)<1,
  \qquad (G,y)\in\Omega,
  \label{eq:family-radial-ideal}
\end{equation}
where $\beta$ is continuous. The selector is continuous on $\Omega$ and
is constant along the segment from $y$ to $P(G,y)$. We will also use the
bound
\begin{equation}
  \|T_G(y)-y\|\leq\operatorname{dist}(y,S(G))
  \qquad(G\in\mathcal G_\star,\ y\in\mathcal Y_n).
  \label{eq:family-ideal-displacement-bound}
\end{equation}
Indeed, off the switching set this follows from
\eqref{eq:family-radial-ideal}; on the switching set $T_G(y)=y$.

\begin{theorem}[Family-wide asymptotic realization from random initialization]
\label{thm:family-wide-convergent-initialization}
Assume $d_y\geq2$. Let $(G,Y_0)$ have a tight Borel probability law
$\mathbb P$ on $\mathcal G\times\mathcal Y_n$, and let $\mu_G$ be a
regular conditional distribution of $Y_0$ given $G$.
Such a Borel kernel exists by disintegrating the joint law on the
ambient Euclidean product and restricting the graph variable to
$\mathcal G$. Assume:
\begin{enumerate}[leftmargin=5mm]
\item $G\in\mathcal G_\star$ almost surely;
\item conditionally on $G$, the law $\mu_G$ of $Y_0$ satisfies
      $\mu_G\ll\mathcal L^m$ almost surely;
\item for almost every $G$ and every $s\in S(G)$,
\[
  \mu_G(V_s(G))>0.
\]
\end{enumerate}
All solution-dependent probability events are understood on the
full-measure set $G\in\mathcal G_\star$.
Fix a limit tolerance $\varepsilon>0$, an operator tolerance $\rho>0$,
and a probability loss $\delta\in(0,1)$. Then there exist a compact
permutation-invariant event
$E_{\varepsilon,\rho,\delta}\subseteq\mathcal G\times\mathcal Y_n$
with
\begin{equation}
  \mathbb P(E_{\varepsilon,\rho,\delta})\geq1-\delta
  \label{eq:family-init-probability}
\end{equation}
and one deterministic, autonomous, weight-tied MP-GNN block such that, for
every $(G,y_0)\in E_{\varepsilon,\rho,\delta}$, the trajectory
\[
  \widehat y_0=y_0,
  \qquad
  \widehat y_{t+1}=\operatorname{GNN}(G,\widehat y_t),
  \qquad t\geq0,
\]
converges to a limit $\widehat y_\infty$ satisfying
\begin{equation}
  \begin{aligned}
  \operatorname{dist}(\widehat y_\infty,S(G))&<\varepsilon,
\\
  \sup_{t\geq0}
  \left\|\operatorname{GNN}(G,\widehat y_t)-T_G(\widehat y_t)\right\|
  &<\rho.
  \end{aligned}
  \label{eq:family-init-complete-error}
\end{equation}
The block receives only $(G,\widehat y_t)$ at every recurrent time.

The graph projection
\[
  \mathcal K_{\varepsilon,\rho,\delta}
  :=\operatorname{proj}_{\mathcal G}E_{\varepsilon,\rho,\delta}
\]
is compact, and the same fixed collection of continuous component maps
works for every graph in this projection on its corresponding retained
initialization fiber. The block is permutation equivariant as in
\eqref{eq:family-seed-free-equivariance}. In particular,
\[
  \mathbb P\!\left(
    \widehat y_t\text{ converges and }
    \operatorname{dist}(\widehat y_\infty,S(G))<\varepsilon
  \right)\geq1-\delta.
\]
Moreover, the good event may be chosen so that every labeled branch has
positive joint capture probability:
\begin{equation}
  \mathbb P\!\left(
    Y_0\in V_{f_a(G)}(G),\quad
    \widehat y_t\text{ converges to }\widehat y_\infty,\quad
    \|\widehat y_\infty-f_a(G)\|<\varepsilon
  \right)>0,
  \qquad a=1,\ldots,M.
  \label{eq:family-init-joint-branch-capture}
\end{equation}
Coincident labels refer to the same distinct target. The learned limit is
required to be $\varepsilon$-close to the solution set; it need not belong
exactly to $S(G)$.
\end{theorem}

The probability law need not be permutation invariant. Equivariance is a
property of the block, and the retained event is made invariant by its
construction. The branch-capture guarantee is joint over graphs and
initializations; it does not assert positive retained mass in every basin
of every individual graph.

\subsection{Exact realization using the current node states}
\label{sec:family-state-realization}

In this subsection, we establish conditions under which a continuous
equivariant operator can be exactly realized by a message-passing GNN
in our setting. We begin with the relevant symmetry-breaking condition.
Without suitable symmetry-breaking features, standard message-passing
GNNs may fail to distinguish graph inputs that the target operator
distinguishes, limiting their ability to represent arbitrary continuous
equivariant operators (see \cite{xu2018powerful}). We therefore introduce a condition under which
the current node states provide distinct node addresses, and prove that
these addresses enable exact realization on a compact family. This
result provides the basis for constructing a message-passing GNN that
follows the ideal dynamics introduced in Appendix~\ref{app:ideal}.

\begin{definition}[Robust state-based node symmetry breaking]
\label{def:family-robust-individualization}
Let $K\subseteq\mathcal G\times\mathcal Y_n$ be compact and
permutation invariant. We say that $K$ admits
\emph{robust state-based node symmetry breaking} if there exist
finitely many compact address regions
$\mathcal A_1,\ldots,\mathcal A_J\subseteq\mathbb R^{d_y}$
and $\eta>0$ such that
\[
  \operatorname{dist}(\mathcal A_j,\mathcal A_k)\geq\eta
  \qquad(j\neq k),
\]
and, for every $(G,y)\in K$, the $n$ node arguments
$y_1,\ldots,y_n$ belong to $n$ different address regions.
\end{definition}

Because there are finitely many compact address regions, pairwise
disjointness already guarantees a positive separation margin;
$\eta$ simply denotes a common margin.
The regions $\mathcal A_1,\ldots,\mathcal A_J$ and their separation margin
$\eta$ are fixed for the whole family $K$, while the assignment of nodes
to regions may vary between inputs. Since the node states
$y_1,\ldots,y_n$ belong to different address regions, each $y_i$ serves
as a unique node identifier within the current input $(G,y)$.

For example, take $d_y=2$ and a fixed two-node graph $G_0$ that is
unchanged by swapping its nodes. Write
$u_\theta=(\cos\theta,\sin\theta)$ and define two opposite arcs:
\[
  \mathcal A_1=\{u_\theta:\theta\in[-\pi/4,\pi/4]\},
  \qquad
  \mathcal A_2=-\mathcal A_1.
\]
These compact regions lie on the right and left sides of the unit
circle, respectively, with
$\operatorname{dist}(\mathcal A_1,\mathcal A_2)=\sqrt{2}$.
The compact permutation-invariant family
\[
  K=\left\{
    (G_0,(u,-u)),\ (G_0,(-u,u)):
    u\in\mathcal A_1
  \right\}
\]
therefore satisfies Definition \ref{def:family-robust-individualization} (i.e., admits robust state-based node symmetry breaking) with $J=2$ and $\eta=\sqrt{2}$:
one node lies in each region, although their assignments can be swapped.

In contrast, the following full-circle family does not satisfy Definition \ref{def:family-robust-individualization}
\[
  K_{\rm circle}=\left\{
    \bigl(G_0,(u_\theta,-u_\theta)\bigr):
    \theta\in[0,2\pi]
  \right\}.
\]
The two node states are still always distance $2$ apart, but each now
ranges over the entire circle. There is no finite collection of compact regions that simultaneously (i) are positively separated,
(ii) cover the entire unit circle, and (iii) place $u_\theta$ and
$-u_\theta$ in different regions for every $\theta$. Indeed, because
the circle is connected, conditions (i) and (ii) force it to lie
entirely in one region, contradicting (iii).

In the remainder of this subsection, we show how continuous equivariant operators can be exactly realized by message-passing GNNs under robust state-based node symmetry breaking. Subsections~\ref{sec:family-separation-boxes} and~\ref{sec:family-asymptotic-proof} then establish this condition for the compact families used in our construction.

\begin{lemma}[Exact MP-GNN realization on a compact family]
\label{lem:family-wide-realization-initialization}
\label{lem:family-wide-realization-persistent}
Let $K\subseteq\mathcal G\times\mathcal Y_n$ be compact, permutation
invariant, and admit robust state-based node symmetry breaking. Fix
integers $u,v\geq0$ with $u+v\geq1$ and, when $v>0$, positive integers
$p_1,\ldots,p_v$. For $\ell=1,\ldots,u$, let
$\varphi_\ell:K\to\mathbb R$ be continuous and permutation invariant.
For $b=1,\ldots,v$, let
$F_b:K\to\mathbb R^{n\times p_b}$ be continuous and permutation
equivariant. Then there exists a MP-GNN, with output width $u+\sum_{b=1}^v p_b$ and
the structural form
\eqref{eq:family-init-mp-initialization}
realizes all these maps simultaneously and exactly on $K$: its output row
at node $i$ is
\[
  \tau(G,y,i):=
  \Bigl(\varphi_1(G,y),\ldots,\varphi_u(G,y),
    [F_1(G,y)]_i,\ldots,[F_v(G,y)]_i\Bigr).
\]
Empty output blocks are omitted. The component maps and all widths are
chosen once for the whole compact family.
\end{lemma}

\begin{proof}
We construct one message-passing layer that records the graph and state
in node-code order, and then define a continuous readout on that record.

\textbf{Step 1: Construct the message-passing representation.}
Fix one collection of compact address regions
$\mathcal A_1,\ldots,\mathcal A_J$ satisfying
Definition~\ref{def:family-robust-individualization} simultaneously
for all $(G,y)\in K$. These regions remain fixed throughout the
construction; only the region containing each node state may vary
with the input. The map that equals the $j$th standard basis vector
$e_j\in\mathbb R^J$ on $\mathcal A_j$ is continuous on their closed
union. The Tietze extension theorem \cite[Theorem~35.1(b)]{munkres2014topology}, applied coordinatewise, gives a continuous map
\[
  c:\mathbb R^{d_y}\longrightarrow\mathbb R^J,
  \qquad c(w)=e_j\quad(w\in\mathcal A_j).
\]
For $(G,y)\in K$, the codes $c_i:=c(y_i)$ are distinct standard basis
vectors. Set $d_0=d_v+d_y+J$ and define
\[
  \operatorname{enc}(X_i,y_i):=(X_i,y_i,c(y_i)).
\]
For the single local layer, take
\[
  \begin{aligned}
  M_0\!\left(h_i^{(0)},h_j^{(0)},A_{ji}\right)
  &:=c_i\otimes c_j\otimes A_{ji},
\\
  U_0\!\left(h_i^{(0)},m_i^{(0)}\right)
  &:=\bigl(c_i,\ c_i\otimes(1,X_i,y_i),\ m_i^{(0)}\bigr).
  \end{aligned}
\]
Thus $h_i^{(1)}=(c_i,N_i,E_i)$, where
\[
  N_i:=c_i\otimes(1,X_i,y_i),\qquad
  E_i:=\sum_{j=1}^n c_i\otimes c_j\otimes A_{ji}.
\]
Set $R_1(c_i,N_i,E_i):=(N_i,E_i)$. The final global sum is
\[
  g^{(1)}=(\mathcal R_{\rm node},\mathcal R_{\rm edge}),\qquad
  \mathcal R_{\rm node}:=\sum_i c_i\otimes(1,X_i,y_i),\qquad
  \mathcal R_{\rm edge}:=\sum_{i,j}c_i\otimes c_j\otimes A_{ji}.
\]
All tensors are vectorized before concatenation. The displayed component
maps are continuous and shared across nodes.

\textbf{Step 2: Recover the graph up to relabeling.}
For a standard basis vector $e_k$, the tensor $e_k\otimes z$ places $z$
in its $k$th block and zero in the remaining blocks. Hence
$\mathcal R_{\rm node}$ records $(1,X_i,y_i)$ in the slot indexed by
$c_i$, and $\mathcal R_{\rm edge}$ records $A_{ji}$ in the slot indexed
by the ordered pair $(c_i,c_j)$.

Note that, the roles of the constant-one coordinates are therefore as follows: The 1 in $(1,X_i,y_i)$ is a node-occupancy indicator. Without it, a node satisfying $X_i=0$ and $y_i=0$ would produce a zero block, indistinguishable from an unused address slot. This ensures that the global tensors reconstruct nodes and edges even when their features are zero; they do not carry solution information or participate in symmetry breaking. 

Suppose $(G,y)$ and $(G',y')$ produce the same two global tensors.
Both inputs have the same number $n$ of nodes, each occupying
a distinct address slot marked by the constant-one coordinate.
Matching these occupied slots therefore defines a unique
permutation $\pi\in\mathfrak S_n$ such that
\[
  c'_{\pi(i)}=c_i,\qquad i\in[n].
\]
Equality of the corresponding blocks then implies
\[
  X'_{\pi(i)}=X_i,\qquad y'_{\pi(i)}=y_i,\qquad
  A'_{\pi(j),\pi(i)}=A_{ji}.
\]
Thus $(G',y')=(\pi G,\pi y)$. If the receiver codes also agree,
$c'_{i'}=c_i$, their distinctness gives $i'=\pi(i)$. Therefore
\[
  \Gamma(G,y,i):=(c_i,\mathcal R_{\rm node},\mathcal R_{\rm edge})
\]
determines the input and its receiving node up to simultaneous relabeling.

\textbf{Step 3: The target is determined by this representation.}
If $\Gamma(G,y,i)=\Gamma(G',y',i')$, Step~2 supplies a permutation
matching both inputs and receivers. Invariance of each $\varphi_\ell$
and equivariance of each $F_b$ imply
\[
  \tau(G,y,i)=\tau(G',y',i').
\]
Consequently, the prescription
\[
  D_0\bigl(\Gamma(G,y,i)\bigr):=\tau(G,y,i)
\]
defines a single-valued map on the reconstruction image
$\mathcal I:=\Gamma(K\times[n])$.

\textbf{Step 4: Construct a continuous output map.}
Both $\Gamma$ and $\tau$ are continuous on the compact space $K\times[n]$,
where $[n]$ has the discrete topology. The induced map $D_0$ is continuous.
For completeness, if $w_k\to w$ in $\mathcal I$, choose preimages
$z_k\in K\times[n]$. Every subsequence of these preimages has a further
convergent subsequence, say $z_{k_j}\to z$. Continuity gives
$\Gamma(z)=w$ and
\[
  D_0(w_{k_j})=\tau(z_{k_j})\longrightarrow\tau(z)=D_0(w).
\]
If $D_0(w_k)$ did not converge to $D_0(w)$, a subsequence staying a
positive distance away would contradict this argument. Thus $D_0$ is
continuous. This is the compact-domain quotient argument in explicit form.

The image $\mathcal I$ is compact and hence closed in its Euclidean
ambient space. Coordinatewise Tietze extension gives a continuous ambient
map $\widetilde D$ agreeing with $D_0$ on $\mathcal I$. Define
\[
  O\bigl((c_i,N_i,E_i),g^{(1)}\bigr)
  :=\widetilde D(c_i,\mathcal R_{\rm node},\mathcal R_{\rm edge}).
\]
On $K$, this output equals $\tau(G,y,i)$, as required.
\end{proof}

Lemma~\ref{lem:family-wide-realization-persistent} shows that, once robust node-state symmetry breaking is available, any prescribed finite collection of continuous permutation-invariant scalar maps and continuous permutation-equivariant node maps can be realized simultaneously and exactly on $K$.  This proof differs from the Stone--Weierstrass density arguments commonly used in GNN universality theory, e.g., \cite{keriven2019universal,azizian2021expressive,geerts2022expressiveness}. It constructs a finite-dimensional representation $(c_i,\mathcal R_{\rm node},\mathcal R_{\rm edge})$ that reconstructs $(G,y,i)$ up to simultaneous node relabeling.  The target therefore factors continuously through this representation, and the Tietze extension theorem provides an ambient continuous readout.  Because the present MP-GNN class allows arbitrary continuous component maps, this construction yields exact realization.  By contrast, Stone--Weierstrass uses a function algebra that contains the constants and separates the relevant points or equivalence classes to establish density in in the relevant continuous-function space, yielding uniform approximation rather than exact equality.

\subsection{Node separation and a common collection of boxes}
\label{sec:family-separation-boxes}

To apply Lemma~\ref{lem:family-wide-realization-persistent}, we need the node states to provide robust addresses (symmetry breaking). The next two lemmas establish this condition: the first derives node separation from random initialization, and the second places the separated states in common address boxes, allowing arbitrarily small probability losses.

\begin{lemma}[Finite-time node separation for the ideal dynamics]
\label{lem:init-undamped-row-collision-null}
Assume $d_y\geq2$, $G\in\mathcal G_\star$ almost surely, and
$\mu_G\ll\mathcal L^m$ almost surely. Put $Y_t:=T_G^t(Y_0)$. Then
\begin{equation}
  \mathbb P\!\left(
    (Y_t)_i\neq(Y_t)_j\text{ for all }t\in\mathbb N_0
    \text{ and }i\neq j\,\middle|\,G
  \right)=1
  \quad\text{for almost every }G.
  \label{eq:family-all-finite-times-separated}
\end{equation}
In particular, this event has joint probability one.

Let $E\subseteq\mathcal G_\star\times\mathcal Y_n$ be compact and
permutation invariant, let $H\in\mathbb N_0$, and let $\gamma>0$.
There exist $\eta>0$ and a compact invariant $E'\subseteq E$ such that
\[
  \begin{aligned}
  \mathbb P(E\setminus E')&<\gamma,
\\
  \|[T_G^t(y_0)]_i-[T_G^t(y_0)]_j\|&\geq\eta
  \quad((G,y_0)\in E',\ 0\leq t\leq H,\ i\neq j).
  \end{aligned}
\]
If finitely many measurable subsets $C_a\subseteq E$ have positive
probability, $E'$ may also be chosen with
$\mathbb P(C_a\cap E')>0$ for every $a$.
\end{lemma}

\begin{proof}
For $n=1$ there are no node pairs; take $E'=E$ and any $\eta>0$.
Assume $n\geq2$, and fix a graph $G\in\mathcal G_\star$ for which the
conditional initialization law is absolutely continuous.
The key observation is that a collision forces an initial node state difference onto a fixed line determined by a solution. In two or more dimensions, such an event has zero measure.

For $i<j$, the node-state-difference map
\[
  D_{ij}:\mathcal Y_n\to\mathbb R^{d_y},\qquad D_{ij}(y)=y_i-y_j,
\]
is linear and surjective. For any $s\in S(G)$, define
\[
  N_{s,ij}:=D_{ij}^{-1}\bigl(\operatorname{span}\{s_i-s_j\}\bigr),
\]
which consists of states whose node difference $y_i-y_j$ is a scalar multiple of the target difference $s_i-s_j$.
The span on the right has dimension at most one. 

Write $W=\operatorname{span}\{s_i-s_j\}$. Since $D_{ij}$ is
surjective, the rank--nullity theorem gives
$\dim\ker D_{ij}=m-d_y$. Its restriction to
$N_{s,ij}=D_{ij}^{-1}(W)$ is a surjective linear map onto $W$
with the same kernel. Applying rank--nullity to this restriction yields
\[
  \dim N_{s,ij}
  =\dim\ker D_{ij}+\dim W
  =m-d_y+\dim\operatorname{span}\{s_i-s_j\}.
\]
The span has dimension zero if $s_i=s_j$ and one otherwise.
Thus, since $d_y\geq2$,
\[
  \dim N_{s,ij}\leq m-d_y+1\leq m-1<m.
\]
Thus $N_{s,ij}$ is a proper linear subspace and has zero Lebesgue measure.
Define set
\[
  N_G:=\Sigma_G^\circ\bigcup
       \left(\bigcup_{s\in S(G)}\ \bigcup_{i<j}N_{s,ij}\right)
\]
where $\Sigma_G^\circ\subseteq\mathcal Y_n$ is the switching set for the fixed graph $G$, as defined in Definition~\ref{def:switching}. 
The set $N_G$ collects exceptional initializations excluded to ensure
a unique selected branch from $\{f_1(G),\cdots,f_M(G)\}$ and rule out node-state collisions at
every finite time.
The set $N_G$ is Lebesgue null: the switching set is Lebesgue null by
Theorem~\ref{thm:joint-multi-attractor}, and the remaining union is finite.
Absolute continuity gives $\mu_G(N_G)=0$.

For $y_0\notin N_G$, let $s=P(G,y_0)$. Basin preservation and
\eqref{eq:family-radial-ideal} give
\begin{equation}
  y_t=s+a_t(y_0-s),\qquad
  a_t:=\prod_{k=0}^{t-1}(1-\beta(G,y_k))>0,
  \qquad a_0:=1.
  \label{eq:family-finite-time-radial-factor}
\end{equation}
If $(y_t)_i=(y_t)_j$ at a finite time, then
\[
  0=a_t\bigl((y_0)_i-(y_0)_j\bigr)+(1-a_t)(s_i-s_j).
\]
Since $a_t>0$, the initial row difference belongs to
$\operatorname{span}\{s_i-s_j\}$, contradicting $y_0\notin N_G$.
The same null set works for all finite times. This proves the conditional
claim \eqref{eq:family-all-finite-times-separated}, and integration over $G$ proves the joint claim.

For the compact restriction, define
\[
  \Delta_H(G,y_0)
  :=\min_{\substack{0\leq t\leq H\\ i<j}}
     \|[T_G^t(y_0)]_i-[T_G^t(y_0)]_j\|.
\]
Finite iterates of $T$ are jointly continuous and equivariant, so
$\Delta_H$ is continuous and invariant. The probability-one statement
implies $\mathbb P(E\cap\{\Delta_H=0\})=0$, and hence
\[
  \mathbb P(E\cap\{\Delta_H<\eta\})\longrightarrow0
  \qquad(\eta\downarrow0).
\]
Choose $\eta>0$ so that this loss is below $\gamma$ and, when subsets
$C_a$ are prescribed, below $\tfrac12\min_a\mathbb P(C_a)$.
Then $E':=E\cap\{\Delta_H\geq\eta\}$ is compact, invariant, and has
all the required properties.
\end{proof}

\begin{lemma}[Separated boxes for finitely many state maps]
\label{lem:family-grid-boxes}
\label{lem:family-separated-grid-boxes}
Let $E$ be a compact metric space with a finite Borel measure $\mu$.
Let $\mathcal T$ be a nonempty finite index set, and let
\[
  V_t=(v_{t,1},\ldots,v_{t,n}):E\longrightarrow\mathbb R^{n\times d_y},
  \qquad t\in\mathcal T,
\]
be continuous. (Here $i$ denotes the node index.) Suppose, for some $\eta>0$,
\[
  \|v_{t,i}(x)-v_{t,j}(x)\|\geq\eta
  \qquad(x\in E,\ t\in\mathcal T,\ i\neq j).
\]
For any $\omega,\gamma>0$, there exist a compact $E'\subseteq E$ and
finitely many pairwise positively separated compact axis-aligned boxes
$Q_1,\ldots,Q_J\subset\mathbb R^{d_y}$, with centers
$b_1,\ldots,b_J$, such that:
\begin{enumerate}[leftmargin=5mm]
\item $\mu(E\setminus E')<\gamma$;
\item every $v_{t,i}(x)$ lies in exactly one box, and the
      $n$ items of each $V_t(x)$ lie in different boxes;
\item replacing every item in $V_t(x)$ by its box center changes the full state by
      less than $\omega$:
\[
  \|V_t(x)-b(V_t(x))\|<\omega
  \qquad(x\in E',\ t\in\mathcal T).
\]
\end{enumerate}
If finitely many measurable subsets $C_a\subseteq E$ have positive measure, $E'$ can preserve positive measure in every $C_a$. If $E$ carries a continuous node-permutation action and $V_t(\pi x)=\pi V_t(x)$, then $E'$ can also be chosen permutation invariant.
\end{lemma}

\begin{proof}
We construct one common grid in the node-state space $\mathbb R^{d_y}$.
First, we choose its mesh small enough to separate different nodes and
control the error from replacing states by cell centers. Second, we
shift the grid so that the set of inputs placing any node state on a
grid boundary has measure zero. Third, we remove inputs whose node
states lie in a sufficiently thin neighborhood of these boundaries,
losing arbitrarily little measure. The remaining node states then lie
in finitely many separated compact boxes.

\textbf{Step 1: Choose a sufficiently small mesh.}
Since $E$ is compact and there are only finitely many continuous maps
$v_{t,i}$, the union of all their ranges is compact and hence bounded.
Choose a grid mesh $\ell$ satisfying
\[
  0<\ell<\min\left\{
    \frac{2\omega}{\sqrt{nd_y}},\frac{\eta}{\sqrt{d_y}}
  \right\}.
\]
For a shift vector $a\in[0,\ell)^{d_y}$, the grid cells are
\[
  a+\ell z+[0,\ell]^{d_y},
  \qquad z\in\mathbb Z^{d_y}.
\]
Each cell has diameter $\sqrt{d_y}\ell<\eta$, so it cannot contain
two node states from the same $V_t(x)$, whose distance is at least
$\eta$. The other bound on $\ell$ will control the center-replacement
error in Step~4.

\textbf{Step 2: Shift the grid to give its boundaries zero measure.}
Fix a coordinate $s\in[d_y]$. We first identify the coordinate values
that occur on a set of inputs of positive measure:
\[
  B_s:=\left\{b\in\mathbb R:
    \mu\{x\in E:(v_{t,i}(x))_s=b\}>0
    \text{ for some }t\in\mathcal T,\ i\in[n]\right\}.
\]
This set is at most countable. Indeed, for any fixed $(t,i)$ and
positive integer $p$, only finitely many values $b$ can satisfy
\[
  \mu\{x\in E:(v_{t,i}(x))_s=b\}\geq\frac1p.
\]
The corresponding level sets are disjoint, and infinitely many such
sets would contradict $\mu(E)<\infty$. Every positive-measure level
has measure at least $1/p$ for some $p$, so there are at most
countably many such levels. Taking the finite union over $(t,i)$
shows that $B_s$ is countable.

The grid boundaries in coordinate $s$ occur at the values
$a_s+k\ell$, $k\in\mathbb Z$. We want none of these values to belong
to $B_s$. A shift that places some $b\in B_s$ on a boundary must
satisfy $a_s=b-k\ell$ for some integer $k$. Thus all such shifts
belong to the countable set
\[
  \{b-k\ell:b\in B_s,\ k\in\mathbb Z\}.
\]
Since the interval $[0,\ell)$ is uncountable, it contains a shift
$a_s$ outside this set. For this choice, every boundary value
$a_s+k\ell$ lies outside $B_s$. By the definition of $B_s$,
\[
  \mu\{x\in E:(v_{t,i}(x))_s=a_s+k\ell\}=0
  \qquad(t\in\mathcal T,\ i\in[n],\ k\in\mathbb Z).
\]

Choose each coordinate shift in this way and fix
$a=(a_1,\ldots,a_{d_y})$ for the entire family. The union of all
grid-boundary hyperplanes is
\[
  \mathcal H
  :=\bigcup_{s=1}^{d_y}\ \bigcup_{k\in\mathbb Z}
    \{v\in\mathbb R^{d_y}:v_s=a_s+k\ell\}.
\]
By the countable union bound, the set of inputs for which any
$v_{t,i}(x)$ lies on $\mathcal H$ has measure zero.

\textbf{Step 3: Remove a thin neighborhood of the boundaries.}
For $0<\tau<\ell/2$, remove the inputs
\[
  N_\tau:=\left\{x\in E:
    \min_{t\in\mathcal T,\,i\in[n]}
    \operatorname{dist}(v_{t,i}(x),\mathcal H)<\tau\right\}.
\]
The set $\mathcal H$ is closed: for each coordinate $s$, the lattice
$a_s+\ell\mathbb Z$ is closed, and there are only finitely many
coordinates. The displayed minimum is continuous, so $N_\tau$ is
relatively open in $E$. As $\tau\downarrow0$, these sets decrease
to the zero-measure event that some node state lies on $\mathcal H$.
Since $\mu$ is finite, continuity of measure from above gives
$\mu(N_\tau)\to0$.

Choose $\tau$ sufficiently small that
\[
  \mu(N_\tau)<\gamma
\]
and, if the subsets $C_a$ are prescribed, also
\[
  \mu(N_\tau)<\tfrac12\min_a\mu(C_a).
\]
Then $E':=E\setminus N_\tau$ is compact, loses less than $\gamma$
measure, and retains positive measure in every prescribed $C_a$.
Node permutations reorder the node states without changing their
internal coordinates. Since every node is tested against the same
fixed grid, equivariance gives
\[
  \min_{t,i}\operatorname{dist}(v_{t,i}(\pi x),\mathcal H)
  =
  \min_{t,i}\operatorname{dist}
    (v_{t,\pi^{-1}(i)}(x),\mathcal H)
  =
  \min_{t,j}\operatorname{dist}(v_{t,j}(x),\mathcal H).
\]
Thus $N_\tau$ and $E'=E\setminus N_\tau$ are permutation invariant,
regardless of whether $\mu$ is invariant.

\textbf{Step 4: Obtain separated boxes and bound the center error.}
Every retained node state lies at distance at least $\tau$ from
every grid boundary, and therefore belongs to one of the shrunken boxes
\[
  Q_z:=a+\ell z+[\tau,\ell-\tau]^{d_y}.
\]
Keep only boxes containing a node state $v_{t,i}(x)$ for some
$x\in E'$, and enumerate them as $Q_1,\ldots,Q_J$.
Only finitely many are needed because all node-state ranges are bounded.

These boxes are compact, and distinct boxes have distance at least
$2\tau>0$. Their diameters satisfy
\[
  \operatorname{diam}(Q_j)
  =\sqrt{d_y}(\ell-2\tau)<\eta.
\]
Consequently, the $n$ node states of each retained $V_t(x)$ belong
to different boxes.

Finally, each node state is within
$\sqrt{d_y}(\ell-2\tau)/2$ of its box center. Summing the squared
errors over the $n$ nodes gives
\[
  \|V_t(x)-b(V_t(x))\|
  \leq\frac{\sqrt{nd_y}}2(\ell-2\tau)<\omega.
\]
This proves all the claims.
\end{proof}

In the application below, $E'$ consists of retained graph--initialization
pairs, and $V_t(G,y_0)=T_G^t(y_0)$ for finitely many times.
The boxes therefore provide one fixed collection of address regions
for all retained graphs, initializations, and these times.
Subsection~\ref{sec:family-asymptotic-proof} then constructs terminal
dynamics that keep each node in its box and contract toward its center,
preserving the addressing condition while ensuring convergence
near the selected solution.

\subsection{Proof of the main result}
\label{sec:family-asymptotic-proof}

Given the preparations, we are now ready to prove Theorem~\ref{thm:family-wide-convergent-initialization}. 

\begin{proof}[Proof of Theorem~\ref{thm:family-wide-convergent-initialization}]
Choose positive probability budgets satisfying
\[
  \delta_{\rm loc}+\delta_{\rm stop}
  +\delta_{\rm sep}+\delta_{\rm grid}<\delta.
\]
These budgets allow us to exclude graph--initialization pairs with
total probability less than $\delta$, while preserving positive
probability for each labeled solution branch. The remaining pairs
form a compact family whose ideal trajectories all approach their
selected solutions within a common finite number of steps.
During these steps, different nodes have uniformly separated states,
which can be placed in small, separated boxes from one fixed
collection shared across all retained inputs.

We then construct an update rule that follows $T_G$ exactly until
the state is sufficiently close to its selected solution. From that
point onward, each node moves halfway toward the center of its
current box at every step. This ensures convergence while keeping
different nodes distinguishable. Switching close to the solution
and using small boxes controls both the final solution error and
the deviation from the ideal update. 
Lemma~\ref{lem:family-wide-realization-persistent} then realizes this
rule with a single fixed GNN block, whose repeated application
generates all retained trajectories.

\textbf{Step 1: Localize the initializations and the ideal trajectories.}
For $G \in \mathcal{G}_{\star}$, the proof of
Theorem~\ref{thm:joint-multi-attractor} identifies the $G$-section of the
joint bad set with $\Sigma_G^\circ$, even after taking the joint closure.
Its switching-set nullity and the first two probability assumptions
therefore give $\mathbb P(\Omega)=1$.
For each label $a$, define
\[
  B_a:=\{(G,y_0)\in\Omega:P(G,y_0)=f_a(G)\}.
\]
These sets are Borel, since $P$ and the regular branches are continuous on
their domains. For $G\in\mathcal G_\star$, the $G$-section of $B_a$
is $V_{f_a(G)}(G)$, and the section is empty otherwise.
Integration of the indicator of $B_a$ against the conditional kernel
is measurable in $G$, so disintegration and the positive-basin
assumption give, with $\mathbb P_G$ denoting the graph marginal,
\[
  \mathbb P(B_a)
  =\int_{\mathcal G_\star}
    \mu_G(V_{f_a(G)}(G))\,d\mathbb P_G(G)>0.
\]
A tight finite Borel measure on a metric space is inner regular by
compact sets: restrict first to a compact set of arbitrarily high
probability and use Borel regularity there. Hence choose compact sets $C_0\subseteq\Omega$ and
$C_a\subseteq B_a$ such that
\[
  \mathbb P(C_0)>1-\delta_{\rm loc},\qquad
  \mathbb P(C_a)>0\quad(a=1,\ldots,M).
\]
Finite symmetrization gives the compact invariant event
\[
  E^0:=\bigcup_{\pi\in\mathfrak S_n}
         \pi\left(C_0\cup\bigcup_{a=1}^M C_a\right)
  \subseteq\Omega,
  \qquad \mathbb P(E^0)>1-\delta_{\rm loc}.
\]
The compact sets $C_a$ are the branch subsets whose mass we preserve.

Write $s(G,y_0):=P(G,y_0)$. The union of the full ideal segments,
\[
  \mathcal C:=\left\{
    \bigl(G,(1-a)y_0+a\,s(G,y_0)\bigr):
    (G,y_0)\in E^0,\ a\in[0,1]\right\},
\]
is a compact subset of $\Omega$: it is the continuous image of
$E^0\times[0,1]$, and Appendix~\ref{app:ideal} shows that each segment
stays in its initial strict basin. Set
\[
  b_*:=\min_{(G,y)\in\mathcal C}\beta(G,y)\in(0,1),
  \qquad
  R_0:=\max_{(G,y_0)\in E^0}\|y_0-s(G,y_0)\|<\infty.
\]
Equation~\eqref{eq:family-radial-ideal} then gives the uniform bound
\begin{equation}
  \|T_G^t(y_0)-s(G,y_0)\|
  \leq(1-b_*)^tR_0
  \qquad((G,y_0)\in E^0,\ t\geq0).
  \label{eq:family-uniform-ideal-convergence}
\end{equation}

\textbf{Step 2: Choose an entrance threshold with a positive margin.}
We use distance to the solution set to decide when to switch from
the ideal update $T_G(\cdot)$ to the final convergent update. By excluding a
small-probability set of graph--initialization pairs, we ensure that
this distance stays uniformly away from the chosen threshold at
every time up to a common finite horizon. The resulting gap separates
the states where the two updates are used, allowing us to combine
them into a continuous rule.

For $x=(G,y_0)\in E^0$, define
\[
  y_t(x):=T_G^t(y_0),\qquad
  d_t(x):=\operatorname{dist}(y_t(x),S(G)).
\]
Recall that $E^0$ is a compact subset of $\Omega$, so its graph
projection is a compact subset of $\mathcal G_\star$. Each branch
$f_a$ is locally Lipschitz on $\mathcal G_\star$ and hence continuous
on this projection. Moreover, $y_0(x)$ is a coordinate projection,
and the recurrence $y_{t+1}(x)=T_G(y_t(x))$ uses the jointly continuous
map $(G,y)\mapsto T_G(y)$. Induction therefore shows that $y_t(x)$
is continuous in $x=(G,y_0)$ for every finite $t$. Consequently,
\[
  d_t(x)=\min_{a=1,\ldots,M}\|y_t(x)-f_a(G)\|
\]
is continuous as the minimum of finitely many continuous functions.

The distance $d_t$ is also permutation invariant. Indeed, equivariance
of $T$ gives $y_t(\pi x)=\pi y_t(x)$, while
$S(\pi G)=\pi S(G)$. Since node permutations preserve Euclidean
distances,
\[
  d_t(\pi x)
  =\operatorname{dist}(\pi y_t(x),\pi S(G))
  =\operatorname{dist}(y_t(x),S(G))
  =d_t(x).
\]
Finally, Appendix~\ref{app:ideal} shows that the ideal trajectory
remains in the strict basin of its selected solution $s(x)$.
Thus $s(x)$ remains the unique nearest solution, and
\[
  d_t(x)=\|y_t(x)-s(x)\|.
\]

Choose the radius before choosing a time horizon. For each fixed
$t\in\mathbb N_0$, there are at most countably many distance values
$b\geq0$ for which
\[
  \mathbb P\{x\in E^0:d_t(x)=b\}>0,
\]
by the same countability argument used in Step~2 of the proof of
Lemma~\ref{lem:family-grid-boxes}. Since there are only countably
many times $t$, the collection of all such values across all finite
times is still countable. The interval
$(0,\min\{\varepsilon,\rho/2\})$ is uncountable, so we can choose
a radius $r$ in this interval outside that collection. This ensures
\[
  \mathbb P\{x\in E^0:d_t(x)=r\}=0
  \qquad\text{for every }t\in\mathbb N_0.
\]
Now choose $H\in\mathbb N_0$ with
\[
  (1-b_*)^H R_0<r/4.
\]
In particular, every trajectory from $E^0$ is within $r/4$ of its selected
solution by time $H$.

For $\zeta>0$, the events
\[
  N_{\rm stop}(\zeta)
  :=\{x\in E^0:|d_t(x)-r|<\zeta
                  \text{ for some }0\leq t\leq H\}
\]
decrease to a null event as $\zeta\downarrow0$. Choose
$0<\zeta<r/2$ such that
\[
  \mathbb P(N_{\rm stop}(\zeta))
  <\min\left\{\delta_{\rm stop},
              \tfrac12\min_a\mathbb P(C_a)\right\}.
\]
The event
\[
  E_{\rm stop}:=E^0\setminus N_{\rm stop}(\zeta)
\]
is compact and permutation invariant, with
\[
  \mathbb P(E_{\rm stop})>1-\delta_{\rm loc}-\delta_{\rm stop},
  \qquad
  \mathbb P(C_a\cap E_{\rm stop})
  >\tfrac12\mathbb P(C_a)>0
  \quad(a=1,\ldots,M).
\]
On $E_{\rm stop}$, every ideal trajectory stays at least $\zeta$
away from the distance threshold $r$ through time $H$:
\[
  |d_t(x)-r|\geq\zeta
  \qquad(x\in E_{\rm stop},\ 0\leq t\leq H).
\]
Thus, at each of these times, the distance to the solution set is
either at most $r-\zeta$ or at least $r+\zeta$.

On $E_{\rm stop}$, define the entrance time
\[
  h(x):=\min\{0\leq t\leq H:d_t(x)\leq r-\zeta\}.
\]
It exists because $d_H<r/4<r-\zeta$. Every retained trajectory satisfies
\begin{equation}
  d_t(x)\geq r+\zeta\quad(0\leq t<h(x)),
  \qquad d_{h(x)}(x)\leq r-\zeta.
  \label{eq:family-entrance-gap}
\end{equation}
Thus it first visits the inner neighborhood of radius $r-\zeta$ after
remaining outside the larger neighborhood of radius $r+\zeta$. For each
$k\in\{0,\ldots,H\}$, its entrance-time level set is
\[
  L_k:=E_{\rm stop}\cap\{d_k\leq r-\zeta\}
       \cap\bigcap_{t<k}\{d_t\geq r+\zeta\}.
\]
Thus $L_k$ consists of graph--initialization pairs, $(G,y_0)$, whose ideal
trajectories first enter the $(r-\zeta)$-neighborhood of the solution
set at time $k$.
This is compact and invariant. When $k=0$, there are no pre-entrance
inequalities. The sets $L_k$ form a finite partition of
$E_{\rm stop}$. The horizon $H$ is a uniform upper bound on this retained
family, whereas $h(x)$ may be smaller and depends on the trajectory.

\textbf{Step 3: Obtain one collection of distinct node-address boxes.}
Apply Lemma~\ref{lem:init-undamped-row-collision-null} to
$E_{\rm stop}$, horizon $H$, loss $\delta_{\rm sep}$, and the subsets
$C_a\cap E_{\rm stop}$. It gives a compact invariant
$E_{\rm sep}\subseteq E_{\rm stop}$ and $\eta>0$ such that
\[
  \mathbb P(E_{\rm stop}\setminus E_{\rm sep})<\delta_{\rm sep},
  \qquad \mathbb P(C_a\cap E_{\rm sep})>0,
\]
and, uniformly over all retained inputs and times up to $H$,
distinct nodes have states at least $\eta$ apart:
\[
  \|[T_G^t(y_0)]_i-[T_G^t(y_0)]_j\|\geq\eta
  \qquad
  \bigl((G,y_0)\in E_{\rm sep},\ 0\leq t\leq H,\ i\neq j\bigr).
\]
Then choose any $0<\omega<\zeta$, apply
Lemma~\ref{lem:family-grid-boxes} on $E_{\rm sep}$ with the restricted
measure $\mathbb P$, tolerance $\omega$, loss $\delta_{\rm grid}$,
subsets $C_a\cap E_{\rm sep}$, and state maps
\[
  V_t(G,y_0):=T_G^t(y_0),\qquad t=0,\ldots,H.
\]
The maps are continuous and equivariant, and the row-separation hypothesis
was just established. The lemma gives boxes $Q_1,\ldots,Q_J$ and a
compact invariant final event
\begin{equation}
  E:=E_{\varepsilon,\rho,\delta}\subseteq E_{\rm sep},\qquad
  \mathbb P(E)>1-\delta,\qquad
  \mathbb P(C_a\cap E)>0\quad(a=1,\ldots,M).
  \label{eq:family-final-retained-event}
\end{equation}
The probability bound follows by adding the four losses. For every
$x\in E$ and $t\in\{0,\ldots,H\}$, there exist pairwise distinct
indices $j_1,\ldots,j_n\in[J]$ such that
\begin{equation}
  [y_t(x)]_i\in Q_{j_i}
  \qquad(i=1,\ldots,n).
\label{eq:robust-symmetry}
\end{equation}
Replacing each node state by the center of its box changes the full
state by less than $\omega$:
\[
  \left\|y_t(x)-(b_{j_1},\ldots,b_{j_n})\right\|<\omega,
\]
where $b_j$ is the center of $Q_j$. The indices $j_i$ may depend on
$(x,t)$, while the boxes and their centers are fixed for the whole family.

Using the final retained event $E$, define two families of
graph--state pairs:
\[
  \begin{aligned}
  K_{\rm tr}
  &:=\bigcup_{k=0}^H\ \bigcup_{t=0}^{k-1}
       \{(G,y_t(x)):x=(G,y_0)\in E\cap L_k\},
  \\
  K_{\rm ent}
  &:=\bigcup_{k=0}^H
       \{(G,y_k(x)):x=(G,y_0)\in E\cap L_k\}.
  \end{aligned}
\]
Here $K_{\rm tr}$ collects the states of each retained trajectory
\emph{before} its first visit to the $(r-\zeta)$-neighborhood of
$S(G)$; these are the transient states. The family $K_{\rm ent}$
collects the state \emph{at} that first visit for each trajectory;
these are the entrance states. By the definition of $L_k$,
\[
  \operatorname{dist}(y,S(G))\geq r+\zeta
  \quad\text{on }K_{\rm tr},
  \qquad
  \operatorname{dist}(y,S(G))\leq r-\zeta
  \quad\text{on }K_{\rm ent}.
\]
Both families are compact, being finite unions of continuous images
of the compact sets $E\cap L_k$. They are permutation invariant
because these sets are invariant and the maps $y_t$ are equivariant.
In every state from either family, different nodes belong to
different boxes from the same collection $Q_1,\ldots,Q_J$---i.e., see \eqref{eq:robust-symmetry}. Hence, $K_{\rm tr}$ and $K_{\rm ent}$ admit robust node-state symmetry breaking (Definition \ref{def:family-robust-individualization}).

\textbf{Step 4: Construct the terminal dynamics and realize both regimes.}
Let $Q:=\bigcup_{j=1}^J Q_j$, and define $b(v):=b_j~(v\in Q_j)$, 
where $b_j$ is the center of $Q_j$. This map is continuous on $Q$
because it is constant on each box and the boxes are positively
separated. For a state $y$ with every node state $y_i\in Q$, define
\[
  b(y):=(b(y_1),\ldots,b(y_n)),\qquad
  \Psi(y):=b(y)+\tfrac12(y-b(y)).
\]
The update $\Psi$ moves each node state halfway toward the center
of its current box. Since the boxes are convex, each node remains
in that same box, so
\[
  b(\Psi(y))=b(y).
\]
Repeated updates therefore converge to the fixed center pattern
$b(y)$, halving each node's distance to its box center at every step.

For each entrance state, include the entire line segment joining it
to its box-center pattern. These segments form the terminal family:
\begin{equation}
  K_{\rm term}:=\left\{
    \bigl(G,b(y)+a(y-b(y))\bigr):
    (G,y)\in K_{\rm ent},\ a\in[0,1]\right\}.
  \label{eq:family-terminal-compactum}
\end{equation}
Here $a=1$ gives the entrance state $y$, while $a=0$ gives the
center pattern $b(y)$ toward which the terminal updates converge. Fix an entrance state $y$. Its node states belong to distinct boxes: there are pairwise distinct indices $j_1,\ldots,j_n$ such that $y_i\in Q_{j_i}$.
For any state $z=b(y)+a(y-b(y))$ on the segment,
\[
  z_i=b_{j_i}+a(y_i-b_{j_i})\in Q_{j_i},
  \qquad a\in[0,1],
\]
because $Q_{j_i}$ is convex and contains both $y_i$ and its center
$b_{j_i}$. Thus each node remains in the box it occupied at entrance,
and different nodes continue to occupy different boxes---i.e., $K_{\rm term}$ admits robust node-state symmetry breaking (Definition \ref{def:family-robust-individualization}).

The family $K_{\rm term}$ is compact as a continuous image of
$K_{\rm ent}\times[0,1]$, and it is permutation invariant.
Moreover,
\[
  \Psi(z)=b(y)+\frac{a}{2}(y-b(y))\in K_{\rm term}(G),
\]
where $K_{\rm term}(G):=\{w:(G,w)\in K_{\rm term}\}$.
Thus every terminal update stays in this family, which contains
all terminal iterates and their limits.

For a terminal state $z=b(y)+a(y-b(y))$ arising from an entrance state
$y=y_{h(x)}(x)$, let $s=s(x)$ be that initialization's selected solution.
The grid bound and \eqref{eq:family-entrance-gap} imply
\begin{equation}
  \|z-y\|<\omega,\qquad
  \|z-b(y)\|<\omega,\qquad
  \|z-s\|\leq r-\zeta+\omega<r.
  \label{eq:family-terminal-error-bounds}
\end{equation}
These estimates hold also at both endpoints of the segment. In particular,
the solution-set distance on $K_{\rm term}$ is below $r$, whereas it is
at least $r+\zeta$ on $K_{\rm tr}$. The two compact families, $K_{\rm term}$ and $K_{\rm tr}$, are disjoint and, when both are nonempty, positively separated.

Set $K:=K_{\rm tr}\cup K_{\rm term}$ and prescribe
\begin{equation}
  F^*(G,y):=
  \begin{cases}
    T_G(y), &(G,y)\in K_{\rm tr},\\[1mm]
    \Psi(y), &(G,y)\in K_{\rm term}.
  \end{cases}
  \label{eq:family-pasted-cell}
\end{equation}
The two pieces are continuous and equivariant, and their domains are
disjoint compact subsets of $K$. Thus $F^*$ is a well-defined continuous
equivariant map on $K$. 

Since $K_{\rm tr}$ and $K_{\rm term}$ admit robust node-state symmetry breaking, $K$ admits robust state-based symmetry breaking, with address regions $\mathcal A_j=Q_j$. Apply Lemma~\ref{lem:family-wide-realization-initialization} with $u=0$, $v=1$, and $F_1=F^*$. It gives one MP-GNN satisfying
\begin{equation}
  \operatorname{GNN}(G,y)=F^*(G,y)
  \qquad((G,y)\in K).
  \label{eq:family-exact-pasted-realization}
\end{equation}
The block is defined on the entire input domain by the extension
steps in Lemma~\ref{lem:family-wide-realization-persistent}.
Its proof extends the address encoder to all of $\mathbb R^{d_y}$
and uses the Tietze extension theorem to extend the decoder from
the compact reconstruction image to its entire Euclidean ambient
space. The resulting GNN is therefore defined for every
$(G,y)\in\mathcal G\times\mathcal Y_n$, while its agreement with
$F^\star$ is guaranteed on $K$. In particular, it already realizes
$\Psi$ on $K_{\rm term}$, so no separate extension of $\Psi$ is needed.
Both regimes are implemented by the same fixed block using only
the current input $(G,y)$, without an additional time input or
stopping flag.

\textbf{Step 5: Verify the full trajectory and branch capture.}
Fix $x=(G,y_0)\in E$ and write $h=h(x)$. For $t<h$, each ideal update
from the transient family leads either to its next transient state or to
the entrance state. From the entrance state onward, $\Psi$ stays in
$K_{\rm term}$. Thus exact realization and induction give
\begin{equation}
  \widehat y_t=T_G^t(y_0)\quad(0\leq t\leq h),\qquad
  \widehat y_{h+k}=b(\widehat y_h)
       +2^{-k}\bigl(\widehat y_h-b(\widehat y_h)\bigr)
       \quad(k\geq0).
  \label{eq:family-asymptotic-tail}
\end{equation}
This includes the case $h=0$. The trajectory converges to
$\widehat y_\infty=b(\widehat y_h)$, and
\eqref{eq:family-terminal-error-bounds} gives
\[
  \|\widehat y_\infty-s(G,y_0)\|
  \leq r-\zeta+\omega<r<\varepsilon.
\]

The learned update equals $T_G$ on the transient. On the terminal family,
$b(z)=b(y)$ for its generating entrance state, so
\[
  \|\Psi(z)-z\|=\tfrac12\|z-b(z)\|<\omega/2.
\]
Using \eqref{eq:family-ideal-displacement-bound} and
\eqref{eq:family-terminal-error-bounds}, uniformly on $K_{\rm term}$,
\[
  \begin{aligned}
  \|\operatorname{GNN}(G,z)-T_G(z)\|
  &\leq\|\Psi(z)-z\|+\|z-T_G(z)\|\\
  & \leq \tfrac12\|z-b(z)\| + \|z - s\| \\
  &\leq\omega/2+r-\zeta+\omega
   <r+\omega/2<2r<\rho.
  \end{aligned}
\]
The right-hand bounds are constants independent of the trajectory and
time, proving the supremum bound in
\eqref{eq:family-init-complete-error}.

The preceding argument establishes convergence and the uniform
operator-error bound for every $(G,y_0)\in E$.
In particular, every retained trajectory converges to a limit within
$\varepsilon$ of its selected solution, and hence of $S(G)$.
Together with \eqref{eq:family-final-retained-event}, this gives
\[
  \mathbb P\!\left(
    \widehat y_t\to\widehat y_\infty,\ 
    \operatorname{dist}(\widehat y_\infty,S(G))<\varepsilon
  \right)
  \geq\mathbb P(E)\geq1-\delta.
\]

To verify branch capture, fix a label $a$. By construction,
$C_a\subseteq B_a$, so every $(G,y_0)\in C_a\cap E$ satisfies
\[
  y_0\in V_{f_a(G)}(G),
  \qquad s(G,y_0)=f_a(G).
\]
The limit estimate therefore gives
\[
  \|\widehat y_\infty-f_a(G)\|
  =\|\widehat y_\infty-s(G,y_0)\|<\varepsilon
  \qquad((G,y_0)\in C_a\cap E).
\]
Since each restriction preserved positive probability in $C_a$,
\[
  \mathbb P\!\left(
    Y_0\in V_{f_a(G)}(G),\
    \widehat y_t\to\widehat y_\infty,\
    \|\widehat y_\infty-f_a(G)\|<\varepsilon
  \right)
  \geq\mathbb P(C_a\cap E)>0.
\]
Thus every labeled branch has positive capture probability under
the joint law of graphs and initializations.

Finally, the graph projection $K_{\varepsilon,\rho,\delta}:=\operatorname{proj}_{\mathcal G}E$ is compact because it is the continuous image of the compact set $E$.
Lemma~\ref{lem:family-wide-realization-persistent} gives a single
GNN that agrees with $F^\star$ on all of $K$. For every retained
pair $(G,y_0)\in E$, the entire learned trajectory stays in $K$.
Thus the same GNN, with all component maps fixed, generates every
retained trajectory by repeated application.
Moreover, the shared nodewise maps and sum aggregations commute with node relabeling, giving the permutation equivariance: $\operatorname{GNN}(\pi G,\pi y) =\pi\operatorname{GNN}(G,y)$. This concludes the proof.
\end{proof}

\section{An obstruction to almost-sure validity for contractive equivariant dynamics}
\label{app:contractive-obstruction}

We use the graph space $\mathcal E_n$, output space
$\mathcal Y_n=\mathbb R^{n\times q}$, Euclidean norms, and permutation
actions of Appendix~\ref{app:ideal}. In particular, $\pi y=P_\pi y$,
and the graph domain $\mathcal G\subseteq\mathcal E_n$ is closed under
node relabeling. The Euclidean norm on a matrix is its Frobenius norm.
All continuity and neighborhood statements on $\mathcal G$ refer to
its relative topology.

For a graph $G\in\mathcal G$, define its automorphism group and the
corresponding fixed subspace by
\[
    \Gamma_G
    :=\{\pi\in\mathfrak S_n:\pi G=G\},
    \qquad
    \operatorname{Fix}(\Gamma_G)
    :=\{y\in\mathcal Y_n:\pi y=y
                      \text{ for every }\pi\in\Gamma_G\}.
\]
The latter is a closed linear subspace, since
\[
    \operatorname{Fix}(\Gamma_G)
    =
    \bigcap_{\pi\in\Gamma_G}
    \ker\bigl(y\mapsto P_\pi y-y\bigr).
\]
Let $\nu$ be a Borel probability measure on $\mathcal G$.
We write $G_0\in\operatorname{supp}\nu$ when every relatively open
neighborhood of $G_0$ has positive $\nu$-measure. Thus a full-support
graph law satisfies this condition at every $G_0\in\mathcal G$.

\subsection{The operator-level obstruction}

\begin{theorem}[Positive-probability failure of continuous contractive equivariant dynamics]
\label{thm:contractive-positive-failure}
Let $S(G)=\{f_1(G),\ldots,f_M(G)\}$ be a permutation-equivariant
finite-branch solution map in the sense of
Definition~\ref{def:equivariant-branches}. For this theorem, assume
additionally that every branch $f_a:\mathcal G\to\mathcal Y_n$ is
continuous. Suppose there is a graph
$G_0\in\operatorname{supp}\nu$ such that
\begin{equation}
    S(G_0)\cap\operatorname{Fix}(\Gamma_{G_0})
    =\varnothing.
    \label{eq:contractive-no-invariant-solution}
\end{equation}
Let
\[
    F:\mathcal G\times\mathcal Y_n\to\mathcal Y_n,
    \qquad (G,y)\mapsto F_G(y),
\]
be jointly continuous and permutation equivariant. Assume that, for
some $\kappa\in[0,1)$,
\begin{equation}
    \|F_G(y)-F_G(z)\|
    \le\kappa\|y-z\|,
    \qquad
    F_{\pi G}(\pi y)=\pi F_G(y),
    \label{eq:contractive-assumptions}
\end{equation}
for all $G\in\mathcal G$, $y,z\in\mathcal Y_n$, and node
permutations $\pi$.
Then the following statements hold.

\begin{enumerate}[leftmargin=5mm]
    \item \textbf{A unique, initialization-independent limit.}
    For every $G\in\mathcal G$, there is a unique fixed point
    $p_F(G)\in\mathcal Y_n$. Every trajectory
    $y_{t+1}=F_G(y_t)$ satisfies
    \begin{equation}
        \|y_t-p_F(G)\|
        \le\kappa^t\|y_0-p_F(G)\|,
        \qquad t\ge1.
        \label{eq:contractive-limit-rate}
    \end{equation}
    The map $p_F:\mathcal G\to\mathcal Y_n$ is continuous and
    permutation equivariant. In particular,
    \[
        p_F(G)\in\operatorname{Fix}(\Gamma_G)
        \qquad(G\in\mathcal G).
    \]

    \item \textbf{Failure on a neighborhood of the obstruction graph.}
    The number
    \begin{equation}
        \gamma_0
        :=
        \min_{1\le a\le M}
        \operatorname{dist}
        \bigl(f_a(G_0),\operatorname{Fix}(\Gamma_{G_0})\bigr)
        \label{eq:contractive-obstruction-margin}
    \end{equation}
    is strictly positive and does not depend on $F$.
    For every $\varepsilon\in(0,\gamma_0)$, there is a relatively
    open neighborhood $V_{F,\varepsilon}$ of $G_0$ such that
    \begin{equation}
        \operatorname{dist}\bigl(p_F(G),S(G)\bigr)>\varepsilon
        \qquad(G\in V_{F,\varepsilon}).
        \label{eq:contractive-open-failure}
    \end{equation}

    \item \textbf{Positive-probability failure under random initialization.}
    Let $(G,Y_0)$ have any Borel probability law $\mathbb P$ on
    $\mathcal G\times\mathcal Y_n$ with graph marginal $\nu$, and
    define
    \[
        Y_{t+1}=F_G(Y_t),\qquad t\ge0.
    \]
    For every $\varepsilon\in(0,\gamma_0)$, set
    $\delta_{F,\varepsilon}:=\nu(V_{F,\varepsilon})>0$.
    Then $Y_t$ converges for every $(G,Y_0)$, and its limit
    $Y_\infty=p_F(G)$ satisfies
    \begin{equation}
        \mathbb P\!\left(
            \operatorname{dist}(Y_\infty,S(G))>\varepsilon
        \right)
        \ge\delta_{F,\varepsilon}>0.
        \label{eq:contractive-positive-error}
    \end{equation}
    Consequently,
    \[
        \mathbb P\bigl(Y_\infty\in S(G)\bigr)<1.
    \]
    No independence, density, or full-support assumption is needed
    for the initialization law.
\end{enumerate}
\end{theorem}

\begin{proof}
We verify the three conclusions.

\medskip
\noindent
\textbf{Step 1: Existence, uniqueness, and convergence for each graph.}
Fix $G\in\mathcal G$ and $y_0\in\mathcal Y_n$. Repeated use of
\eqref{eq:contractive-assumptions} gives
\[
    \|y_{t+1}-y_t\|
    \le\kappa^t\|y_1-y_0\|.
\]
For integers $r>s\ge0$, it follows that
\[
    \|y_r-y_s\|
    \le
    \sum_{t=s}^{r-1}\kappa^t\|y_1-y_0\|.
\]
The right-hand side tends to zero as $s\to\infty$, uniformly in
$r>s$, because the geometric series is summable. Hence $(y_t)$
is Cauchy. The space $\mathcal Y_n$ is complete, so $y_t$ has a
limit $p$. Continuity of $F_G$ yields
\[
    F_G(p)
    =
    \lim_{t\to\infty}F_G(y_t)
    =
    \lim_{t\to\infty}y_{t+1}
    =p.
\]
If $p$ and $\widetilde p$ are fixed points, then
\[
    \|p-\widetilde p\|
    \le\kappa\|p-\widetilde p\|,
\]
which forces $p=\widetilde p$ because $\kappa<1$.
Write this unique fixed point as $p_F(G)$. Its uniqueness makes
the limit independent of $y_0$, and iteration of
\[
    \|y_{t+1}-p_F(G)\|
    \le\kappa\|y_t-p_F(G)\|
\]
proves \eqref{eq:contractive-limit-rate}. If $\kappa=0$, the
trajectory reaches the fixed point after one step.

\medskip
\noindent
\textbf{Step 2: Continuity of the fixed-point map.}
For any $G,H\in\mathcal G$, the fixed-point equations and
\eqref{eq:contractive-assumptions} imply
\[
    \begin{aligned}
        \|p_F(G)-p_F(H)\|
        &=
        \|F_G(p_F(G))-F_H(p_F(H))\|\\
        &\le
        \kappa\|p_F(G)-p_F(H)\|
        +\|F_G(p_F(H))-F_H(p_F(H))\|.
    \end{aligned}
\]
Therefore
\begin{equation}
    \|p_F(G)-p_F(H)\|
    \le
    \frac{\|F_G(p_F(H))-F_H(p_F(H))\|}{1-\kappa}.
    \label{eq:contractive-graph-continuity}
\end{equation}
Fixing $H$ and letting $G\to H$, joint continuity of $F$ makes
the numerator tend to zero. This proves continuity of $p_F$
on $\mathcal G$.

\medskip
\noindent
\textbf{Step 3: Equivariance and preservation of graph automorphisms.}
By definition, $p_F(\pi G)$ is a fixed point of $\pi G$.
For any node permutation $\pi$,
\[
    F_{\pi G}(\pi p_F(G))
    =
    \pi F_G(p_F(G))
    =
    \pi p_F(G)
\]
which implies $\pi p_F(G)$ is also a fixed point of $\pi G$.
By uniqueness of the fixed point at $\pi G$,
\[
    p_F(\pi G)=\pi p_F(G).
\]
If $\pi\in\Gamma_G$, then $\pi G=G$, and consequently
$\pi p_F(G)=p_F(G)$. Thus
$p_F(G)\in\operatorname{Fix}(\Gamma_G)$.

\medskip
\noindent
\textbf{Step 4: A positive error persists on a graph neighborhood.}
The fixed subspace $\operatorname{Fix}(\Gamma_{G_0})$ is closed.
By \eqref{eq:contractive-no-invariant-solution}, each of the
finitely many points $f_a(G_0)$ lies outside this subspace.
Its distance to the subspace is therefore strictly positive,
and the finite minimum in
\eqref{eq:contractive-obstruction-margin} satisfies
$\gamma_0>0$. Step 3 gives
\[
    \|p_F(G_0)-f_a(G_0)\|\ge\gamma_0
    \qquad(1\le a\le M).
\]

Fix $\varepsilon\in(0,\gamma_0)$ and put
$\eta:=(\gamma_0-\varepsilon)/2>0$.
By continuity of $p_F$ and of the finitely many branches, there
is a relatively open neighborhood $V_{F,\varepsilon}$ of $G_0$
such that
\[
    \|p_F(G)-p_F(G_0)\|<\eta,
    \qquad
    \max_{1\le a\le M}\|f_a(G)-f_a(G_0)\|<\eta
    \qquad(G\in V_{F,\varepsilon}).
\]
For every such $G$ and every $a$, the reverse triangle
inequality gives
\[
    \begin{aligned}
        \|p_F(G)-f_a(G)\|
        &\ge
        \|p_F(G_0)-f_a(G_0)\|
        -\|p_F(G)-p_F(G_0)\|
        -\|f_a(G)-f_a(G_0)\|\\
        &>\gamma_0-2\eta
        =\varepsilon.
    \end{aligned}
\]
Taking the minimum over the finite branch set proves
\eqref{eq:contractive-open-failure}.

\medskip
\noindent
\textbf{Step 5: The probability statement.}
Because $G_0\in\operatorname{supp}\nu$,
$\delta_{F,\varepsilon}=\nu(V_{F,\varepsilon})>0$.
The function
\[
    G\longmapsto
    \operatorname{dist}(p_F(G),S(G))
    =
    \min_{1\le a\le M}\|p_F(G)-f_a(G)\|
\]
is continuous, so the events in the statement are measurable.
By Step 1, $Y_\infty=p_F(G)$ for every initial state.
Consequently,
\[
    \{G\in V_{F,\varepsilon}\}
    \subseteq
    \{\operatorname{dist}(Y_\infty,S(G))>\varepsilon\}.
\]
Taking $\mathbb P$-probabilities and using its graph marginal
$\nu$ proves \eqref{eq:contractive-positive-error}.
The last assertion follows because the event on the right
contains no valid limit.
\end{proof}

\noindent\textbf{Remark (Scope of the assumptions).}
The graph domain need not be compact or connected, and the graph law
need not be permutation invariant. Full support of $\nu$ is sufficient
but not necessary: the proof only uses
$G_0\in\operatorname{supp}\nu$.
The neighborhood argument only needs continuity of the branches at
$G_0$; continuity on $\mathcal G$ is imposed above to ensure that all
global error events are measurable without additional assumptions.
In particular, the branch assumption holds on a domain where the
branches are locally Lipschitz, as in the regular stable setting of
Appendix~\ref{app:ideal}.
The failure is uniform over initial states on
$V_{F,\varepsilon}$ only at the level of their limits; no common
finite convergence time over the unbounded space $\mathcal Y_n$
is asserted.

\subsection{Application to independently weighted triangles}

Fix $0<w_-<w_+<\infty$ and let
$W=[w_-,w_+]^3$. For
$w=(w_{12},w_{13},w_{23})\in W$, define
\[
    A(w)=
    \begin{pmatrix}
        0&w_{12}&w_{13}\\
        w_{12}&0&w_{23}\\
        w_{13}&w_{23}&0
    \end{pmatrix},
    \qquad
    G(w)=(A(w),0).
\]
Here $n=3$, $d_e=d_v=1$, and $X=0\in\mathbb R^{3\times1}$.
Set
\[
    \mathcal G_\triangle:=\{G(w):w\in W\},
    \qquad
    L_G:=\operatorname{diag}(A\mathbf1)-A,
    \qquad
    H_G:=I_3+L_G.
\]
We take $q=2$ and write $y=[u\ b]\in\mathcal Y_3
=\mathbb R^{3\times2}$, where $u,b\in\mathbb R^3$ are its
two columns. The single-source diffusion solution set is
\begin{equation}
    S(G)=\{f_1(G),f_2(G),f_3(G)\},
    \qquad
    f_k(G):=[H_G^{-1}e_k\ \ e_k].
    \label{eq:contractive-triangle-solutions}
\end{equation}
Thus a valid output satisfies $H_Gu=b$, with $b$ a one-hot
source indicator.

\begin{corollary}[Impossibility under independent edge weights]
\label{cor:contractive-independent-triangles}
Let the three edge weights be independent and uniform on
$[w_-,w_+]$, and let $\nu$ be the induced law on
$\mathcal G_\triangle$.
Suppose $F:\mathcal G_\triangle\times\mathcal Y_3\to\mathcal Y_3$
satisfies the continuity, equivariance, and global contraction
assumptions of Theorem~\ref{thm:contractive-positive-failure}.
For every
\[
    0<\varepsilon<\sqrt{\frac23},
\]
there is a number $\delta_{F,\varepsilon}>0$ such that, under
any joint law of $(G,Y_0)$ with graph marginal $\nu$, the
recurrence converges and
\[
    \mathbb P\!\left(
        \operatorname{dist}(Y_\infty,S(G))>\varepsilon
    \right)
    \ge\delta_{F,\varepsilon}.
\]
In particular, no such operator converges to a valid solution
almost surely over independently weighted graphs.
The same conclusion holds for any full-support Borel law on
$\mathcal G_\triangle$.
\end{corollary}

\begin{proof}
We verify the solution-map, support, and symmetry hypotheses.

\medskip
\noindent
\textbf{Step 1: The graph family and solution branches.}
Relabeling the nodes permutes the three edge coordinates, so
$\mathcal G_\triangle$ is permutation invariant.
For every $v\in\mathbb R^3$,
\[
    v^\top H_Gv
    =
    \|v\|_2^2+
    \sum_{1\le i<j\le3}w_{ij}(v_i-v_j)^2
    \ge\|v\|_2^2.
\]
Hence $H_G$ is invertible and
$\|H_G^{-1}\|_{\mathrm{op}}\le1$.
The map $G\mapsto H_G$ is linear in the edge matrix up to
the constant identity. For any $G,\widetilde G\in\mathcal G_\triangle$,
the resolvent identity gives
\[
    H_G^{-1}-H_{\widetilde G}^{-1}
    =
    H_G^{-1}(H_{\widetilde G}-H_G)H_{\widetilde G}^{-1}.
\]
Consequently,
\[
    \|H_G^{-1}-H_{\widetilde G}^{-1}\|_{\mathrm{op}}
    \le
    \|H_{\widetilde G}-H_G\|_{\mathrm{op}}.
\]
Since linear maps between finite-dimensional Euclidean spaces
are globally Lipschitz, every branch in
\eqref{eq:contractive-triangle-solutions} is globally Lipschitz.

Furthermore,
\[
    H_{\pi G}=P_\pi H_GP_\pi^\top,
    \qquad
    H_{\pi G}^{-1}=P_\pi H_G^{-1}P_\pi^\top.
\]
Using $P_\pi e_k=e_{\pi(k)}$ yields
\[
    f_{\pi(k)}(\pi G)
    =
    [P_\pi H_G^{-1}e_k\ \ P_\pi e_k]
    =
    \pi f_k(G).
\]
Thus Definition~\ref{def:equivariant-branches} holds with
$\tau_\pi(k)=\pi(k)$. The branches are distinct for every graph
because their second columns are distinct; indeed,
\[
    \|f_k(G)-f_\ell(G)\|
    \ge\|e_k-e_\ell\|_2
    =\sqrt2
    \qquad(k\ne\ell).
\]
Their collapse partition consists of three singleton sets
throughout the domain. Hence $D=U=\varnothing$ and
$(\mathcal G_\triangle)_\star=\mathcal G_\triangle$ in the
notation of Definition~\ref{def:regular-stable}.

\medskip
\noindent
\textbf{Step 2: The graph law has full support.}
The map $w\mapsto G(w)$ is a homeomorphism from $W$ onto
$\mathcal G_\triangle$. In fact,
\[
    \|G(w)-G(\widetilde w)\|
    =
    \sqrt2\,\|w-\widetilde w\|_2.
\]
Every nonempty relatively open subset of the cube $W$ has
positive three-dimensional Lebesgue measure: it contains
the intersection of $W$ with a sufficiently small Euclidean
ball, and this intersection has positive volume.
The uniform product law has constant positive density on
$W$. Its pushforward $\nu$ therefore has full support on
$\mathcal G_\triangle$.

\medskip
\noindent
\textbf{Step 3: A symmetric graph provides a fixed positive margin.}
Choose any $c\in(w_-,w_+)$ and set $G_0=G(c,c,c)$.
Every node permutation fixes $G_0$, so
$\Gamma_{G_0}=\mathfrak S_3$ and
\[
    \operatorname{Fix}(\Gamma_{G_0})
    =
    \{[\alpha\mathbf1\ \ \beta\mathbf1]:
                         \alpha,\beta\in\mathbb R\}.
\]
For any $k\in\{1,2,3\}$ and $\beta\in\mathbb R$,
\[
    \|e_k-\beta\mathbf1\|_2^2
    =
    (1-\beta)^2+2\beta^2
    =
    3\left(\beta-\frac13\right)^2+\frac23
    \ge\frac23.
\]
It follows that
\[
    S(G_0)\cap\operatorname{Fix}(\Gamma_{G_0})=\varnothing,
    \qquad
    \gamma_0\ge\sqrt{\frac23}.
\]
Step 2 gives $G_0\in\operatorname{supp}\nu$.
Applying Theorem~\ref{thm:contractive-positive-failure}
proves the claimed probability bound for every
$\varepsilon\in(0,\sqrt{2/3})$ and for any initialization law.
The proof uses only full support of the graph law, so the
last assertion follows as well.
\end{proof}

\subsection{Consequence for recurrent message-passing GNNs}

\begin{corollary}[Contractive recurrent MP-GNNs]
\label{cor:contractive-mpgnn}
Consider a deterministic, autonomous, weight-tied MP-GNN block
of the form \eqref{eq:family-init-mp-initialization}, with
continuous component maps and only $y_t$ carried between
recurrent calls. Suppose its complete update satisfies
\[
    \|\operatorname{GNN}(G,y)-\operatorname{GNN}(G,z)\|
    \le\kappa\|y-z\|
    \qquad
    (G\in\mathcal G,\ y,z\in\mathcal Y_n)
\]
for some $\kappa<1$.
Then the conclusions of
Theorem~\ref{thm:contractive-positive-failure} hold whenever
its solution-map and graph-law hypotheses hold.
In particular, Corollary~\ref{cor:contractive-independent-triangles}
applies on $\mathcal G_\triangle$, for every finite number
of hidden layers and every choice of finite hidden widths.
\end{corollary}

\begin{proof}
Set $F_G(y):=\operatorname{GNN}(G,y)$.
The encoder is continuous in $(G,y)$. Inductively, if all
$h_i^{(\ell)}$ are continuous, then each message is continuous
by continuity of $M_\ell$, its finite sum is continuous, and
$h_i^{(\ell+1)}$ is continuous by continuity of $U_\ell$.
The final finite sum and continuous output map preserve
continuity. Thus $F$ is jointly continuous.

Under the simultaneous relabeling $(G,y)\mapsto(\pi G,\pi y)$,
the encoder satisfies
$h_{\pi(i)}^{(0)}(\pi G,\pi y)=h_i^{(0)}(G,y)$.
If this identity holds at layer $\ell$, reindexing the sender
sum by $j\mapsto\pi(j)$ gives the same identity for messages
and then for layer $\ell+1$. The final global sum is invariant,
and the node outputs are permuted. Hence
$F_{\pi G}(\pi y)=\pi F_G(y)$, as also stated in
\eqref{eq:family-seed-free-equivariance}.
The assumed bound is exactly the remaining contraction
hypothesis. The two cited results now apply.
\end{proof}

\noindent\textbf{Relation to the multi-attractor results.}
For the triangle family, the preceding proof established
$(\mathcal G_\triangle)_\star=\mathcal G_\triangle$.
Consequently, Theorem~\ref{thm:joint-multi-attractor} applies.
If the initialization law has a density and full support on
$\mathcal Y_3$, every strict Voronoi basin has positive
probability. Indeed, for each solution $s$, a sufficiently
small open ball centered at $s$ lies in $V_s(G)$ because
the other solutions are finitely many and distinct.
The ideal operator therefore has almost-sure validity and
positive capture probability for every solution at every
fixed graph in $\mathcal G_\triangle$.

By contrast, every contractive operator above has one limit
$p_F(G)$ for each fixed $G$. Its limiting law under any
initialization distribution is a point mass. It cannot reach
two distinct solutions with positive probability at the
same graph. Moreover,
Corollary~\ref{cor:contractive-independent-triangles} rules
out exact almost-sure validity under the independent-weight
graph law, even if coverage of multiple solutions is not required.

\section{Single-source diffusion: protocol and additional results}
\label{app:toy}

\textbf{Datasets.}
Each instance is a complete undirected graph on three nodes. Its fixed
input consists only of the three positive edge weights; static node
features are identically zero and omitted from the implemented encoder.
Draw $w_{12},w_{23},w_{31}$ independently from
$\operatorname{Unif}[0.5,1.5]$.
Randomly permute node labels consistently with the adjacency matrix.
No node indices, drawing coordinates, prescribed source, or exact
diffusion fields are supplied as input features.
We generate 500/100/100 train/validation/test instances.

\textbf{Model architecture.}
Across the three models, we use the same architecture:
\begin{align}
 h_i^{(0)}&=\operatorname{enc}(y_i),\nonumber\\
 m_i^{(\ell)}&=\sum_{j\ne i}
   M_\ell\bigl(h_i^{(\ell)},h_j^{(\ell)},w_{ij}\bigr),\nonumber\\
 h_i^{(\ell+1)}&=U_\ell\bigl(h_i^{(\ell)},m_i^{(\ell)}\bigr),
       \qquad \ell=0,1,\nonumber\\
 [N_\theta(G,y)]_i&=O(h_i^{(2)}),\qquad
 \mathrm{GNN}_{\theta}(G,y)=P(a y+N_\theta(G,y)).
 \label{eq:toy-cell}
\end{align}
Each component is a two-affine-layer MLP with a width-32 hidden layer
and a SiLU activation between the affine layers. Messages use sum
aggregation over the two neighbors. 
Layers within a block have separate parameters, but the entire block
is shared across recurrent time $t$. The scalar skip $a$ is also shared
over nodes, coordinates, and time. 
The output projection is $P([u\ b])=[u\ \operatorname{clip}_{[0,1]}(b)]$; it does not impose the mass constraint or choose a source. 
Both columns of $y_0$ are i.i.d.\ $\mathcal N(0,1)$ samples, independently across nodes and trajectories. 

\textbf{Training.}
For supervised training, sample $k_G$ uniformly from $\{1,2,3\}$ once per training graph, independently of the neural initialization, and
retain $s_G=[H_G^{-1}e_{k_G}\ \ e_{k_G}]$ for every subsequent visit. The target is \emph{not} resampled between updates or epochs.
We train the model by minimizing the loss function $\|y_{T}-s_G\|_F^2/6$ with $T=10$. For energy based training, the training loss is $E_G(u_T,b_T)$, where energy function $E_G$ is defined in \eqref{eq:toy-energy} and $T=10$. Our multi-attractor model is directly trained with this energy function without constraints on its weights, while contractive models are trained over the same energy function but applying the Lipschitz control technique described below.

During training, a batch contains
32 different training graphs and eight starts per graph (256 trajectories);
the final batch has 20 graphs (160 trajectories). With 500 graphs,
there are 16 optimizer updates per epoch. We train 1,200 epochs
(19,200 updates) with Adam, its default momentum parameters
$(0.9,0.999)$, no weight decay, and no gradient clipping.
The initial learning rate is $10^{-3}$, multiplied by $0.2$ after
epoch 600 and by $0.05$ relative to the initial rate after epoch 960.

\textbf{Lipschitz control.}
We enforce global contraction with respect to the recurrent state in
the infinity norm. For a node-state array $h$, this means
$\|h\|_\infty=\max_i\|h_i\|_\infty$.
For a weight matrix, we use the corresponding induced norm
$\|W\|_\infty=\max_r\sum_s|W_{rs}|$.

Let $\sigma_{\rm S}(r)=r/(1+e^{-r})$ denote SiLU. We use the fixed,
conservative Lipschitz bound $c_{\rm S}:=1+1/e$, since $ |\sigma_{\rm S}'(r)|
 \le 1+|r|e^{-|r|}
 \le c_{\rm S}$ where $r\in\mathbb R$.
Consequently, a two-affine-layer MLP
$g(z)=W_2\sigma_{\rm S}(W_1z+\beta_1)+\beta_2$
satisfies
\[
 \operatorname{Lip}_\infty(g)
 \le c_{\rm S}\|W_2\|_\infty\|W_1\|_\infty.
\]
For a message MLP $M_\ell(h_i,h_j,w_{ij})$, the graph is held fixed
when measuring state sensitivity. Thus its first-layer norm is taken
only over the \emph{combined} columns multiplying $h_i$ and $h_j$;
the edge-weight columns and all biases are excluded from this norm bound.
In particular, $L_{M_\ell}$ bounds the joint Lipschitz constant in
$[h_i;h_j]$, rather than a separate constant for each argument.

Concatenation uses a maximum in the infinity norm:
\[
 \|[h_i;h_j]-[\widetilde h_i;\widetilde h_j]\|_\infty
 =\max\{\|h_i-\widetilde h_i\|_\infty,
         \|h_j-\widetilde h_j\|_\infty\}.
\]
Therefore, summing the two incoming messages gives a bound
$2L_{M_\ell}$, with no additional factor for the two message arguments.
The update MLP receives $[h_i;m_i]$, giving the layer bound
$L_{U_\ell}\max\{1,2L_{M_\ell}\}$.
Writing $L_{\rm enc}$ and $L_O$ for the encoder and output MLP bounds,
the complete residual block satisfies
\[
 \operatorname{Lip}_y(N_\theta)
 \le L_O L_{\rm enc}
       \prod_{\ell=0}^1
       L_{U_\ell}\max\{1,2L_{M_\ell}\}.
\]
We impose the following constraints once after parameter initialization
and immediately after every Adam update, before the next forward pass.
Set $\rho=0.9$, $c=c_{\rm S}^{-1/2}$, and $\eta=1-10^{-5}$.
First clip $a$ to $[-\rho+0.05,\rho-0.05]$, then use this clipped value
to set the output-layer budget. The absolute row-sum caps and resulting
MLP bounds are
\[
\begin{array}{c|c|c|c}
 \text{MLP} & \text{First-layer cap} & \text{Second-layer cap}
             & \text{Lipschitz bound} \\
 \hline
 \operatorname{enc} & \eta c & \eta c & L_{\rm enc}\le\eta^2 \\
 M_\ell & \eta c/2 & \eta c & L_{M_\ell}\le\eta^2/2 \\
 U_\ell & \eta c & \eta c & L_{U_\ell}\le\eta^2 \\
 O & \eta c & \eta(\rho-|a|)c & L_O\le\eta^2(\rho-|a|)
\end{array}
\]
For $M_\ell$, the first-layer cap applies to the sum of absolute entries
across both state-input blocks together, not to each block separately.
All other caps apply to the full weight-matrix rows.
For a constrained row or row segment $w$ with cap $B>0$, we rescale
\[
 w\leftarrow\frac{w}{\max\{1,\|w\|_1/B\}}.
\]
Thus rows already satisfying the cap are unchanged; oversized rows
are scaled radially to the boundary. This projection is an optimizer-side
operation, performed without differentiating through it.

Since $2L_{M_\ell}\le\eta^2<1$, the two message-passing layers yield $\operatorname{Lip}_y(N_\theta) \le\eta^8(\rho-|a|)$.
The coordinatewise clipping map $P$ is nonexpansive in the infinity norm,
so the full update $\mathrm{GNN}_{\theta}(G,y)=P(a y+N_\theta(G,y))$ satisfies
\[
 \operatorname{Lip}_y(\mathrm{GNN}_{\theta})
 \le |a|+\eta^8(\rho-|a|)
 \le\rho.
\]
The constraints therefore certify contraction throughout training,
rather than only after training or at sampled states.
The multi-attractor and supervised
models have no such projection. Single-target supervision itself does
not guarantee a unique equilibrium.

\textbf{Additional visualizations.}
Figure~\ref{fig:toy-result} in the main text shows the evolution of
the source coordinates $b_t\in\mathbb R^3$.
Since the full recurrent state is $y_t=[u_t\ b_t]$,
Figure~\ref{fig:toy-u} complements this visualization by showing
the diffusion coordinates $u_t\in\mathbb R^3$.
The same pattern emerges: the multi-attractor trajectories
form three distinct clusters, whereas the supervised and contractive
controls each concentrate near a single point that remains visibly
separated from all three exact diffusion solutions.

\begin{figure}[t]
 \centering
 \includegraphics[width=0.99\linewidth]{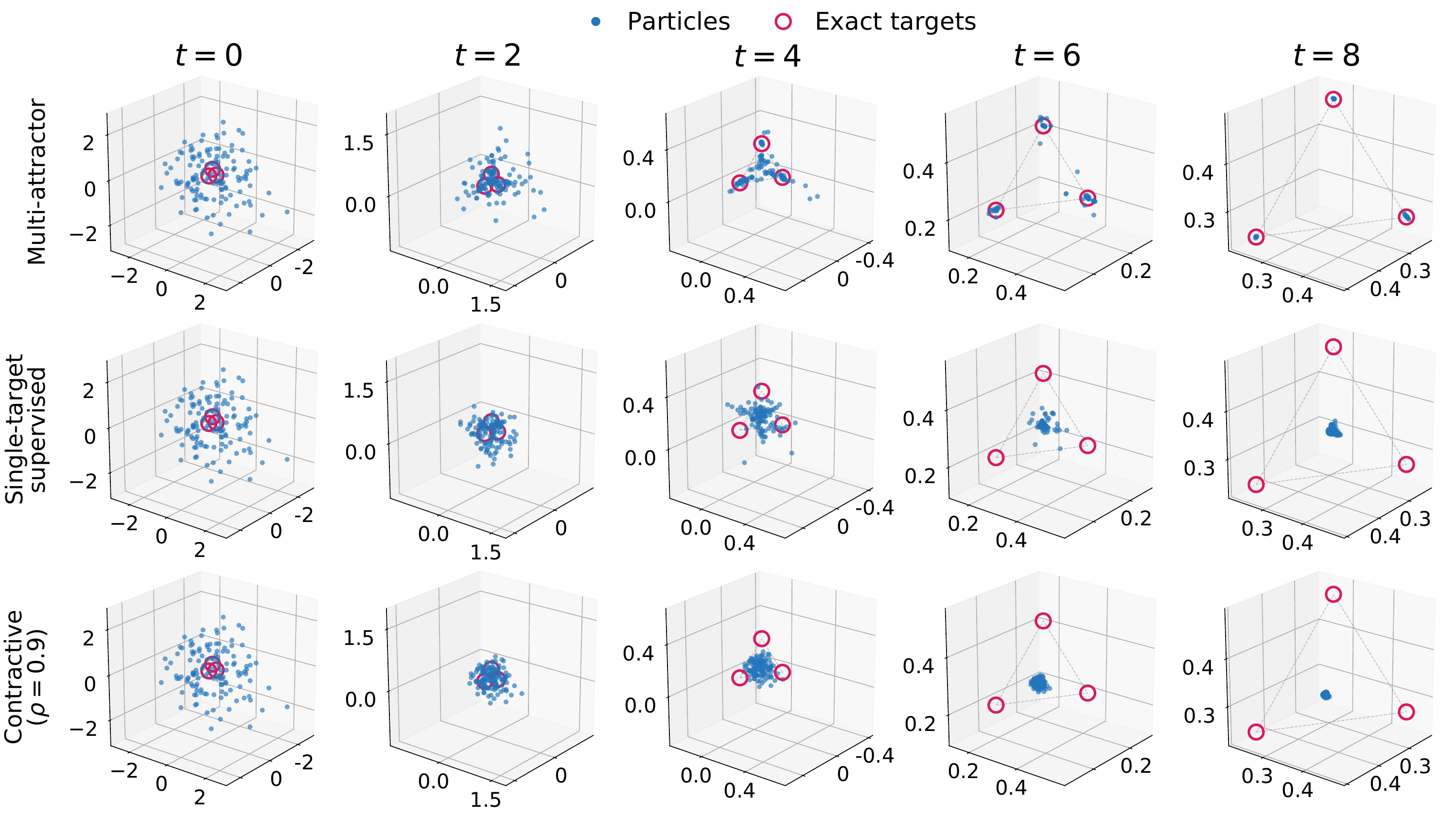}
 \caption{$u$-counterpart of Figure~\ref{fig:toy-result}:
 the same test graph, seed, 128 trajectories, methods, and time steps.
 Blue points show $[(u_t)_1,(u_t)_2,(u_t)_3]$; magenta circles mark the three true solutions.
 }
 \label{fig:toy-u}
\end{figure}

\section{Ising Experiments: Details and Additional Results}
\label{ising:appendix}

This section provides details and additional results regarding the Ising model experiments.

\textbf{Datasets.}
We generate lattice graphs by \texttt{networkx.triangular\_lattice\_graph} with open boundaries, following \cite{jore2026bifurcation}. Each undirected edge receives an independent uniform coupling from $\{-0.5,-1,-1.5\}$; both directed orientations are stored with equal weight. 
We generate graphs with controlled sizes to assess scalability. We consider four levels of sizes (see Table~\ref{ising:datasets}). Every size has 2,500 train-source and 500 test graphs; the final 250 train-source graphs are held out for validation.

For each graph used in supervised training, Gurobi \citep{gurobi} provides one optimal spin configuration with a reported MIP gap of zero. Enumerating the full solution set is computationally infeasible. Labeling extra-large graphs can take more than 5,000 seconds per instance, so we omit supervised methods at that scale. Certified test references are still used to evaluate the extra-large graphs.

\begin{table}[ht]
\centering
\small 
%\vspace{-3mm}
\caption{Graph sizes.}
\label{ising:datasets}
\begin{tabular}{lrrrr}
\toprule
Level & Nodes & Mean nodes & Edges & Mean edges \\
\midrule
Compact & 6--55 & 26.972 & 9--135 & 61.068 \\
Medium & 136--253 & 189.174 & 360--693 & 510.918 \\
Large & 561--990 & 770.580 & 1584--2838 & 2195.904 \\
Extra-large & 2016--3916 & 2965.198 & 5859--11484 & 8667.174 \\
\bottomrule
\end{tabular}
\end{table}

\textbf{Model Architecture.}
For the recurrent dynamics $y_{t+1}=\mathrm{GNN}_{\theta}(G,y_t)$,
we use a graph convolutional backbone augmented with physical node features.

At each step, we compute for each node $i$ (here $N(i)$ represents the neighbors of $i$):
\begin{align}
  f_i&=\sum_{j\in N(i)}J_{ij}y_j,\qquad
  a_i=\sum_{j\in N(i)}|J_{ij}|,\qquad
  q_i=f_i/\max(a_i,10^{-8}),\label{ising:fields}\\
  c_i^{(\beta)}&=\tanh\!\left[
  \frac{\sum_{j\in N(i)}\operatorname{atanh}
  (\operatorname{clip}(\tanh(\beta J_{ij})y_j,-1+10^{-6},1-10^{-6}))}
  {\max(a_i,10^{-8})}\right],\label{ising:cavity}
\end{align}
with $\beta\in\{0.5,1,2,4\}$, giving the seven features
$\phi_i=[f_i,q_i,a_i,c_i^{(0.5)},c_i^{(1)},c_i^{(2)},c_i^{(4)}]$.
Here, $f_i$ is the local interaction field, $a_i$ measures the total
incident coupling strength, and $q_i$ provides a normalized field.
The four $c_i^{(\beta)}$ features provide nonlinear summaries of
neighboring states at different coupling scales. Their functional form
is inspired by the cavity equations for Ising belief propagation
\citep{ricci2012bethe}, with node states replacing directed cavity
messages and an additional normalization by the incident coupling strength.
Clipping keeps the argument of $\operatorname{atanh}$ within $(-1,1)$,
including for Gaussian initial states.
These quantities are supplied as input features; they do not define
a separate belief-propagation solver or prescribe the update direction.

Given these features, the recurrent cell
$y^{+}=\mathrm{GNN}_{\theta}(G,y)$, with $\phi=\phi(G,y)$, is defined as follows.
The encoder is
\[
  h_i^{(0)}=\operatorname{LayerNorm}\!\left(
  W_{e,2}\sigma(W_{e,1}[y_i,\phi_i]+b_{e,1})+b_{e,2}\right),
\]
with $\sigma=\operatorname{SiLU}$, dimensions $8\to16\to16$, and
affine LayerNorm with epsilon $10^{-5}$. No zero-padding coordinates
or random identifiers are supplied.
The three weighted GraphConv layers are
\[
  h_i^{(\ell+1)}=\sigma\!\left(
  W_{r,\ell}h_i^{(\ell)}+
  W_{n,\ell}\sum_{j\in N(i)}J_{ij}h_j^{(\ell)}+b_{n,\ell}\right),
  \qquad \ell=0,1,2.
\]
They use sum aggregation, a biased neighbor transform, and a bias-free
root transform, without intermediate normalization or hidden residual
connections.
Independent $16\to16\to16$ SiLU MLPs produce
$u_i=U(h_i^{(3)})$ and $v_i=R(h_i^{(3)})$.
The readout combines local and graph-level information:
\begin{equation}
  \bar v=\frac1n\sum_jv_j,\qquad
  \Delta_i=w_o^\top[u_i,\bar v]+b_o,\qquad
  y_i^+=\tanh(y_i+\Delta_i).
  \label{ising:readout}
\end{equation}

\textbf{Training.}
For minibatch $\mathcal B$ with $N_{\mathcal B}=\sum_{G\in\mathcal B}|V_G|$, the proposed model uses the following loss in training: $(1/N_{\mathcal B})\sum_{\mathcal{B}}E_G(y_{T})$ with $T=10$. Training uses such node-weighted energy within each minibatch, whereas test energies are averaged equally across graphs. 
In particular, we use Adam for 200 epochs, LR $10^{-3}$, default betas $(0.9,0.999)$,
epsilon $10^{-8}$, and zero weight decay. There is no LR schedule,
clipping, dropout, rollback, or early stopping. Graph batches are
32/32/4/1 for compact/medium/large/extra-large; each update uses one fresh
trajectory per graph. The loader shuffles each epoch. The best validation checkpoint is saved after every completed epoch.

\textbf{Baselines.}
We compare single-target supervision, unique-equilibrium models,
untied feedforward networks, and learned solution distributions.
All methods are evaluated on the same test graphs.
Table~\ref{ising:configs} summarizes model capacities and requested
training budgets; resource limits and completed budgets are reported below.

\begin{itemize}[leftmargin=5mm]
    \item \textit{Single-target supervision.}
    We use the same architecture as our multi-attractor GNN,
    replacing the energy objective with MSE against one Gurobi-certified
    optimal spin configuration per training graph.
    Obtaining even one certified solution can exceed 5,000 seconds per
    graph at the extra-large scale, so we omit this supervised baseline
    at that scale.
    We allocate 2,000 training epochs and select the checkpoint with
    the lowest validation MSE. This extended budget, compared with
    200 epochs for our energy-trained model, accommodates the slower
    optimization observed under single-target supervision: validation
    MSE continued to improve beyond 200 epochs.
    \item \textit{Unique-equilibrium controls.}
    We adopt the implicit graph neural network
    (IGNN) of \citet{gu2020implicit}, using its released implementation
    for the implicit layers and enforcing a contraction bound of
    $\rho=0.9$ to ensure a unique equilibrium.
    To ensure fairness, we use the same physical-feature family and scalar
    spin-output interface as our model and train the adapters with the same
    energy objective.
    We evaluate hidden widths $H\in\{16,32,64\}$, including a
    capacity-comparable model and two larger variants, to assess
    whether increased capacity improves performance under the
    unique-equilibrium constraint.
    \item \textit{Untied feedforward depth controls.}
    We stack independently parameterized copies of our recurrent
    block to obtain
    $y_{t+1}=\mathrm{GNN}_{\theta_t}(G,y_t)$,
    with distinct parameters $\theta_t$ at each stage.
    We consider 1, 3, and 10 blocks, corresponding to 3, 9, and 30
    message-passing layers.
    These models retain the block architecture, physical features,
    Gaussian node initialization, and terminal energy objective
    of our model. They test whether increasing feedforward depth
    and parameter count can match recurrent performance under the
    reported budgets.
    \item \textit{Distributional and diffusion-based models.}
    We include a mean-field GNN, an annealed Bernoulli GNN,
    DIFUSCO, and Variational Annealing on Graphs (VAG-CO; denoted VAG).
    The \emph{mean-field GNN} uses our width-16, three-layer
backbone to map a fresh Gaussian initialization
$y_0\sim\mathcal{N}(0,I_n)$ to predicted spin means
$y_1\in(-1,1)^n$ in one forward pass.
It defines conditionally independent spins with
$\Pr(s_i=+1\mid G,y_0)=(1+y_{1,i})/2$
and is trained by differentiating their analytic expected
energy,
$\mathbb{E}[E_G(s)\mid G,y_0]=E_G(y_1)$,
following the training of
\citet{karalias2020erdos}.
    The otherwise identical \emph{annealed Bernoulli GNN}
    additionally rewards conditional Bernoulli entropy, with its
    coefficient decreasing linearly from $0.5$ to zero over
    200 training epochs, following the annealed variational
    training of \citet{sun2022annealed}.
    \emph{DIFUSCO} \citep{sun2023difusco} is adapted to Ising
    using binary categorical diffusion and a denoiser conditioned
    on noisy spins, diffusion time, and graph features.
    It is trained with clean-label cross entropy against one
    Gurobi-certified spin configuration and generates candidates
    by reverse diffusion from uniform random bits.
    \emph{VAG} \citep{sanokowski2023variational} generates
    spins autoregressively in a randomized breadth-first order,
    jointly sampling five spins at each decoding step.
    It is trained with an energy-minus-entropy objective using
    score-function gradients and a leave-one-out baseline, with
    temperature annealed during training.
    For both DIFUSCO and VAG, the \emph{matched} setting replaces
their GNN backbone with our width-16, three-layer GraphConv backbone
and local/global-mean readout, while retaining method-specific inputs
and output heads.
DIFUSCO additionally receives diffusion-time and structural features
and predicts binary denoising probabilities.
VAG additionally receives partial spin assignments, assignment
masks, and decoding priorities, and predicts a distribution over
the $2^5$ assignments of each five-spin block.
These additions yield 3,250 and 5,904 parameters, respectively,
compared with 3,153 for our model.
The \emph{large} settings adopt the model-size configurations from
their original GitHub repositories: width 256 with 12 layers for
DIFUSCO and width 120 with three layers for VAG;
    see Table~\ref{ising:configs}.
\end{itemize}

\begin{table}[ht]
\centering
\caption{Learned configurations. $H$: hidden width; $L$: total learned message passing depth.}
%\vspace{2mm}
\label{ising:configs}
\small 
\setlength{\tabcolsep}{3pt}
\begin{tabular}{lrrrrr}
\toprule
Model & $H$ & $L$ & Parameters & LR & Epochs \\
\midrule
Ours & 16 & 3 & 3,153 & $10^{-3}$ & 200 \\
\midrule
Single-target Supervision & 16 & 3 & 3,153 & $10^{-3}$ & 2000 \\
IGNN adapter (H16) & 16 & 3 & 2,851 & $10^{-3}$ & 200  \\
IGNN adapter (H32) & 32 & 3 & 10,819 & $10^{-3}$ & 200  \\
IGNN adapter (H64) & 64 & 3 & 42,115 & $10^{-3}$ & 200  \\
\midrule
Feedforward ($\times 1$) & 16 & 3 & 3,153 & $10^{-3}$ & 200 \\
Feedforward ($\times 3$) & 16 & 9 & 9,459 & $10^{-3}$ & 200  \\
Feedforward ($\times 10$) & 16 & 30 & 31,530 & $10^{-3}$ & 200 \\
\midrule
Mean-field GNN & 16 & 3 & 3,153 & $10^{-3}$ & 200 \\
Annealed Bernoulli GNN & 16 & 3 & 3,153 & $10^{-3}$ & 200 \\
DIFUSCO (matched) & 16 & 3 & 3,250 & $10^{-3}$ & 200 \\
DIFUSCO (large) & 256 & 12 & 1,909,762 & $2\cdot10^{-4}$ & 100 \\
VAG (matched) & 16 & 3 & 5,904 & $10^{-3}$ & 200 \\
VAG (large) & 120 & 3 & 281,192 & $5\cdot10^{-4}$ & 30 \\
\bottomrule
\end{tabular}
\end{table}

\textbf{Resource limits.}
Structural models, matched DIFUSCO, and energy diffusion use training batches 32/32/4/1 by size; large DIFUSCO uses 8/8/1/1, and VAG uses 32/1/1/1.
Training samples for each graph are processed in parallel, but VAG tokens
are sequential. Large VAG is OOM on large and extra-large graphs even
alone on a 24-GB GPU at graph batch one.
Larger-graph training is capped at 30 accumulated wall-clock hours,
including validation and retries; capped rows use the best validated
checkpoint. Matched VAG completes 68/17/3 of 200 epochs at
medium/large/extra-large sizes. 

\textbf{Metrics.}
We evaluate $K=500$ test graphs, generating $M=20$ candidate
spin configurations per graph.
Here, $g\in\{1,\ldots,K\}$ indexes test graphs and
$m\in\{1,\ldots,M\}$ indexes candidates.
Let $s_{g,m}\in\{-1,+1\}^{n_g}$ denote the $m$th candidate
for graph $G_g$, where $n_g$ is its number of nodes.
For models with continuous outputs, we set
$s_{g,m}=\operatorname{sign}_+(y_{g,m,T})$,
where $T$ is the evaluation horizon and
$\operatorname{sign}_+(0)=+1$; discrete generators use their
native spin outputs.

We measure energy per node,
$e_{g,m}=E_{G_g}(s_{g,m})/n_g$.
Let $e_g^*$ be the Gurobi-certified optimal energy per node
for $G_g$. The relative optimality gap is
\[
\gamma_{g,m}
=\max\!\left(
0,\frac{e_{g,m}-e_g^*}{\max(|e_g^*|,\epsilon)}
\right),
\]
where $\epsilon>0$ safeguards the denominator.
We report three quality metrics:
\[
\bar e
=\frac{1}{KM}\sum_{g=1}^{K}\sum_{m=1}^{M}e_{g,m},
~
\mathrm{Best}_{20}
=\frac{1}{K}\sum_{g=1}^{K}\min_{1\le m\le M}e_{g,m},
~
\mathrm{Hit10}
=\frac{100}{KM}\sum_{g=1}^{K}\sum_{m=1}^{M}
\mathbf{1}\{\gamma_{g,m}\le 0.10\}.
\]
Thus, $\bar e$ averages energy over all candidates,
$\mathrm{Best}_{20}$ averages the best candidate energy
within each graph, and $\mathrm{Hit10}$ is the percentage
of candidates within 10\% of the certified optimum.
All graphs receive equal weight.
For iterative models, candidates are evaluated at the terminal
step; we do not select intermediate states along a trajectory.

For the experiment in Table \ref{ising:main-table}, we measure diversity among
optimal spin configurations, allowing a numerical tolerance of $10^{-6}$:
\[
 D_M(g)=\bigl|\{s_{g,m}:1\le m\le M,\ \gamma_{g,m}\le10^{-6}\}\bigr|,
 \qquad
 G_M=\frac{100}{K}\sum_{g=1}^{K}\mathbf1\{D_M(g)\ge2\}.
\]
Here, $D_M(g)$ counts the distinct optimal configurations found
among $M$ candidates for graph $G_g$.
We report its graph-average $D_M=K^{-1}\sum_{g=1}^{K}D_M(g)$.
The percentage $G_M$ measures how often the model finds at least
two different optimal configurations for the same graph,
indicating its ability to recover multiple solutions.

For larger graphs, we relax the quality threshold and count
distinct configurations within 10\% of the certified optimum
among 20 candidates:
\[
 V_{20}^{10\%}=\frac1K\sum_{g=1}^{K}
 \bigl|\{s_{g,m}:1\le m\le20,\
 \gamma_{g,m}\le0.10+10^{-6}\}\bigr|.
\]
For example, $V_{20}^{10\%}=20$ means that every graph has
20 distinct candidates meeting this threshold.
Repeated configurations count only once, while global spin flips
are counted as distinct configurations.
These metrics measure the diversity of sufficiently accurate
outputs; they do not require the underlying trajectories to have
converged.

Recurrence is numerically converged at $T$ when
$\sqrt{n^{-1}\|y_T-y_{T-1}\|_2^2}<10^{-4}$; $C_T$ is the percentage
meeting this criterion. Among nonconverged trajectories, periods
$p=2,\ldots,10$ are tested using
\[
 \left[\frac1{pn}\sum_{k=0}^{p-1}
 \|y_{T-k}-y_{T-k-p}\|_2^2\right]^{1/2}<10^{-4},
\]
assigning the smallest qualifying period. Remaining trajectories are
unclassified at the tested horizon, not necessarily divergent or chaotic.
Convergence is not applicable to feedforward or
time-inhomogeneous generative models.

\textbf{Longer test-time trajectories.}
Table~\ref{ising:main-table} reports a convergence rate slightly below 100\% at $t=200$. To determine whether the remaining trajectories exhibit
persistent oscillations or simply require more iterations, we
extend the same 30,000 ($500 \times 20 \times 3$) trajectories across three training seeds
to $t=2000$, without changing the learned weights.
Only 11 trajectories fail the numerical fixed-point criterion
at $t=200$; all satisfy it by the measured horizon $t=500$
and remain numerically stationary at subsequent reported horizons
through $t=2000$ (Table~\ref{tab:ising_test_time}).
The extended checks also confirm sustained numerical stationarity
at the final horizon, supporting the interpretation that these
exceptions reflect longer transients.

Although training uses only $T=10$ iterations, the learned update
remains stable and settles to numerical fixed points when applied
for substantially more iterations than seen during training.
These are numerical observations rather than asymptotic guarantees.
The longer runs serve only to examine convergence:
all reported energy, Hit10, $D_{20}$, and $G_{20}$ metrics
use the terminal outputs at $t=200$.
Thus, the reported solution quality and diversity are already
obtained by stopping at $t=200$, without waiting for every
trajectory to satisfy the convergence criterion.

\begin{table}[ht]
\centering
%\vspace{-2mm}
\caption{Ising test-time dynamics. Percentages are means $\pm$ population standard deviations across three training seeds.}
%\vspace{2mm}
\label{tab:ising_test_time}
\begin{tabular}{rrrr}
\toprule
$t$ & $C_t$ (\%) & Period 2--10 (\%) & Unclassified (\%) \\
\midrule
5 & $3.3733 \pm 0.6370$ & --- & --- \\
10 & $32.5833 \pm 2.0956$ & --- & --- \\
20 & $77.2733 \pm 2.0129$ & $0.0000 \pm 0.0000$ & $22.7267 \pm 2.0129$ \\
50 & $97.7400 \pm 0.2617$ & $0.0000 \pm 0.0000$ & $2.2600 \pm 0.2617$ \\
100 & $99.7100 \pm 0.0638$ & $0.0000 \pm 0.0000$ & $0.2900 \pm 0.0638$ \\
200 & $99.9633 \pm 0.0170$ & $0.0000 \pm 0.0000$ & $0.0367 \pm 0.0170$ \\
500 & $100.0000 \pm 0.0000$ & $0.0000 \pm 0.0000$ & $0.0000 \pm 0.0000$ \\
1000 & $100.0000 \pm 0.0000$ & $0.0000 \pm 0.0000$ & $0.0000 \pm 0.0000$ \\
2000 & $100.0000 \pm 0.0000$ & $0.0000 \pm 0.0000$ & $0.0000 \pm 0.0000$ \\
\bottomrule
\end{tabular}
\end{table}

\textbf{Larger Graphs.}
Table~\ref{ising:quality-scaling} evaluates scalability on larger graphs
while retaining the same width-16
architecture and 3,153 trainable parameters.
Our model achieves the lowest mean and best-of-20 energies,
the highest Hit10, and the largest number of distinct valid
configurations among the evaluated methods at every scale.
Its Hit10 reaches $96.88\%$, $99.96\%$, and $100\%$ on medium,
large, and extra-large graphs, respectively, with an average
of $19.376$, $19.992$, and $20.000$ distinct configurations
within 10\% of optimum among 20 candidates.
Thus, a small shared recurrent model continues to produce
both accurate and diverse solutions as graph size increases.
This is consistent with the set-valued representation established
by our theory: a single shared update can accommodate multiple
solutions for the same input graph.
Here, the empirical evidence concerns finite-horizon outputs,
rather than convergence to distinct optimal attractors.

Increasing baseline capacity also helps in some cases, but does
not consistently recover this combination of quality and diversity.
The 3-block feedforward model substantially improves over
one-shot prediction, and large DIFUSCO is competitive on medium
and large graphs, yet both yield worse mean and best-of-20
energies than our model.
Further increasing feedforward depth to 10 untied blocks reduces
Hit10 and produces the same rounded configuration across all
20 initializations for each tested graph.
These results support recurrent weight sharing as an effective
way to scale solution generation without increasing parameter
count, while the time-capped and OOM entries identify practical
resource limits of the evaluated baseline configurations.

\begin{table}
\centering
\footnotesize
\caption{
Scaling to larger Ising graphs with fixed model capacity.
The proposed model retains its width-16 architecture and
3,153 parameters across medium,
large, and extra-large
datasets, each containing 500 test graphs.
Metrics use 20 candidates per graph:
$\bar e$ and Best$_{20}$ are mean and best-of-20 energies per node;
Hit10 is the percentage of candidates within 10\% of the
certified optimum; and $V_{20}^{10\%}$ is the mean number
of distinct configurations meeting this threshold.
$^{\dagger}$ marks training capped at 30 hours;
OOM denotes a memory failure.
$^{\ddagger}$ marks supervised methods omitted because of
training-labeling costs.
}
\label{ising:quality-scaling}
\label{ising:quality-l2}
\label{ising:quality-l4}
\label{ising:quality-l6}
\begin{tabular}{lrrrr}
\hline
Method & $\bar e\downarrow$ & Best$_{20}\downarrow$ & Hit10 (\%) $\uparrow$ & $V_{20}^{10\%}\uparrow$ \\
\hline
\multicolumn{5}{c}{{Medium}} \\
\hline
Ours & $-1.2431$ & $-1.2862$ & $96.88$ & $19.376$ \\
Single-target Supervision & $1.5248$ & $1.2913$ & $0.00$ & $0.000$ \\
IGNN adapter (H16) & $0.5713$ & $0.5713$ & $0.00$ & $0.000$ \\
IGNN adapter (H32) & $-0.7917$ & $-0.7917$ & $0.00$ & $0.000$ \\
IGNN adapter (H64) & $0.0052$ & $0.0052$ & $0.00$ & $0.000$ \\
Feedforward ($\times 1$) & $-1.0970$ & $-1.1792$ & $1.60$ & $0.320$ \\
Feedforward ($\times 3$) & $-1.1990$ & $-1.2590$ & $58.78$ & $11.756$ \\
Feedforward ($\times 10$) & $-1.1966$ & $-1.1966$ & $56.60$ & $0.566$ \\
Mean-field GNN & $-1.0963$ & $-1.1781$ & $1.61$ & $0.322$ \\
Annealed Bernoulli GNN & $-1.1002$ & $-1.1811$ & $1.73$ & $0.346$ \\
DIFUSCO (matched) & $-1.0864$ & $-1.2224$ & $10.49$ & $2.098$ \\
DIFUSCO (large) & $-1.1934$ & $-1.2653$ & $55.25$ & $11.050$ \\
VAG (matched)$^{\dagger}$ & $-0.9949$ & $-1.0899$ & $0.01$ & $0.002$ \\
VAG (large) & $0.0306$ & $-0.2003$ & $0.00$ & $0.000$ \\
\hline
\multicolumn{5}{c}{{Large}} \\
\hline
Ours & $-1.2789$ & $-1.3023$ & $99.96$ & $19.992$ \\
Single-target Supervision$^{\dagger}$ & $0.4071$ & $-0.2762$ & $0.00$ & $0.000$ \\
IGNN adapter (H16) & $0.0864$ & $0.0864$ & $0.00$ & $0.000$ \\
IGNN adapter (H32) & $-0.9817$ & $-0.9817$ & $0.00$ & $0.000$ \\
IGNN adapter (H64) & $0.1216$ & $0.1216$ & $0.00$ & $0.000$ \\
Feedforward ($\times 1$) & $-1.1296$ & $-1.1722$ & $0.00$ & $0.000$ \\
Feedforward ($\times 3$) & $-1.2444$ & $-1.2734$ & $75.85$ & $15.170$ \\
Feedforward ($\times 10$) & $-1.2289$ & $-1.2289$ & $41.00$ & $0.410$ \\
Mean-field GNN & $-1.1287$ & $-1.1720$ & $0.00$ & $0.000$ \\
Annealed Bernoulli GNN & $-1.1352$ & $-1.1778$ & $0.00$ & $0.000$ \\
DIFUSCO (matched) & $-1.1285$ & $-1.2057$ & $0.59$ & $0.118$ \\
DIFUSCO (large) & $-1.2514$ & $-1.2876$ & $80.26$ & $16.052$ \\
VAG (matched)$^{\dagger}$ & $-0.9778$ & $-1.0318$ & $0.00$ & $0.000$ \\
VAG (large) & OOM & OOM & OOM & OOM \\
\hline
\multicolumn{5}{c}{{Extra-large}} \\
\hline
Ours & $-1.3111$ & $-1.3227$ & $100.00$ & $20.000$ \\
Single-target Supervision$^{\ddagger}$ & -- & -- & -- & -- \\
IGNN adapter (H16) & $0.2937$ & $0.2937$ & $0.00$ & $0.000$ \\
IGNN adapter (H32) & $-0.8143$ & $-0.8143$ & $0.00$ & $0.000$ \\
IGNN adapter (H64) & $2.9178$ & $2.9178$ & $0.00$ & $0.000$ \\
Feedforward ($\times 1$) & $-1.1443$ & $-1.1665$ & $0.00$ & $0.000$ \\
Feedforward ($\times 3$) & $-1.2753$ & $-1.2912$ & $98.85$ & $19.770$ \\
Feedforward ($\times 10$) & $-1.2390$ & $-1.2390$ & $5.00$ & $0.050$ \\
Mean-field GNN & $-1.1435$ & $-1.1658$ & $0.00$ & $0.000$ \\
Annealed Bernoulli GNN & $-1.1549$ & $-1.1767$ & $0.00$ & $0.000$ \\
DIFUSCO (matched)$^{\ddagger}$ & -- & -- & -- & -- \\
DIFUSCO (large)$^{\ddagger}$ & -- & -- & -- & -- \\
VAG (matched)$^{\dagger}$ & $0.0980$ & $0.0355$ & $0.00$ & $0.000$ \\
VAG (large) & OOM & OOM & OOM & OOM \\
\hline
\end{tabular}\end{table}

\section{Structural module detection in protein graphs: details and additional results}
\label{app:proteins}

This appendix gives the experiment details for Section~\ref{sec:proteins}.

\textbf{Dataset.}
We download the PROTEINS dataset through PyTorch Geometric's
\texttt{TUDataset}.  For each graph
we remove self-loops, collapse duplicate directed entries to one undirected
edge, and use both directions only for message passing.  The original graph
classification label is retained for audit purposes but is never supplied to
the models.  The task graph contains 1,113 graphs, 43,471 nodes, and 81,044
undirected objective edges.

We compute exact graph-isomorphism groups, randomly order the groups with split
seed 2027, and allocate complete groups to 80/10/10 splits.  This
gives 891 training, 111 validation, and 111 test graphs.  No exact topology is
shared across validation and test or with another split; duplicate topologies
inside training remain grouped.  Feature-normalization statistics are fit on the 891 training graphs only; standardized added topology features are clipped to $[-8,8]$.

\begin{table}[h]
\centering
%\vspace{-3mm}
\small
\caption{Graph statistics.  Entries for nodes and edges are
min/mean/median/max.}
%\vspace{2mm}
\label{tab:proteins-data-stats}
\begin{tabular}{lccc}
\toprule
Split & Graphs & Nodes & Undirected edges \\
\midrule
Train & 891 & 4 / 36.53 / 24 / 620 & 5 / 67.98 / 46 / 1049 \\
Validation & 111 & 5 / 52.56 / 31 / 504 & 6 / 98.33 / 61 / 894 \\
Test & 111 & 8 / 45.84 / 32 / 285 & 16 / 86.15 / 57 / 424 \\
\bottomrule
\end{tabular}
\end{table}

For exact references, binary $x_{ik}$ assign node $i$ to group $k$,
$\sum_kx_{ik}=1$.  Auxiliary $w_{ijk}$ linearize
$x_{ik}x_{jk}$ on each observed edge, and
$v_k=\sum_i d_ix_{ik}$.  Gurobi maximizes
\begin{equation}
 4m\sum_{\{i,j\}\in E}\sum_k w_{ijk}-\sum_kv_k^2,
 \label{eq:proteins-gurobi-objective}
\end{equation}
whose value divided by $4m^2$ is $Q_G$.  All 1,113 instances terminate with Gurobi's optimal status. One certified partition per graph is retained as the supervised target.

\textbf{Node features.}
All neural models in this experiment receive the same node and edge features.
We augment the original node attributes with standard structural descriptors
drawn from local degree profiles~\citep{cai2018simple}, random-walk
encodings~\citep{dwivedi2022graph}, and graph statistics documented in
NetworkX~\citep{hagberg2007exploring}.
The input $s_i\in\mathbb R^{19}$ comprises five original-attribute/degree
coordinates and fourteen additional structural coordinates:
\[
 s_i=[\,\widetilde a_i,\operatorname{onehot}_3(\ell_i),\widehat d_i,
       \mathcal Z_{\rm node}(\psi_i)\,].
\]
Here $a_i$ is the original continuous node attribute, $\ell_i$ is its original
three-category node label, and $\widehat d_i=d_i/\max(1,\max_jd_j)$.
The standardized attribute $\widetilde a_i$ uses the training-node mean and
sample standard deviation (denominator floored at $10^{-8}$).
The original node labels are attributes, not target module assignments.

Let $N(i)$ be the neighbor set, $n=|V|$, $W_{ij}=A_{ij}/\max(d_i,1)$,
$\tau_i$ the triangle count at $i$, $\kappa_i$ its core number, and
$\pi_i$ its PageRank.
The fourteen raw structural coordinates are
\begin{align}
\psi_i=\big[&
 \log(1+d_i),(W^2)_{ii},(W^3)_{ii},(W^4)_{ii},(W^8)_{ii},
 c_i,\log(1+\tau_i),\frac{\kappa_i}{\max(1,\max_j\kappa_j)},n\pi_i,
 \nonumber\\[-1mm]
 &\log(1+\mu_i),\log(1+\sigma_i),
 \frac{|N_2(i)|}{\max(n-1,1)},\frac{|\mathcal C(i)|}{n},
 \frac{b_i}{\max(d_i,1)}\big].
 \label{eq:proteins-rich19}
\end{align}
The four diagonal powers are random-walk return probabilities at lengths
$2,3,4,8$.
$c_i=2\tau_i/[d_i(d_i-1)]$ is the local clustering coefficient, set to zero
when $d_i<2$. The neighbor-degree mean and population standard deviation
are $\mu_i$ and $\sigma_i$ (both zero for an empty neighborhood).
$N_2(i)$ contains nodes reachable in two hops after excluding $i$ and its
one-hop neighbors; $\mathcal C(i)$ is the connected component containing $i$;
$b_i$ counts incident bridges. Core numbers use the undirected simple graph.
PageRank is approximated by 80 iterations from a uniform initialization,
with damping $.85$, uniform teleportation, and uniform redistribution of
dangling-node mass.
The operator $\mathcal Z_{\rm node}$ standardizes each added coordinate
using all training nodes, floors population standard deviations at $10^{-6}$,
and clips the result to $[-8,8]$.

\textbf{Edge features.}
We combine endpoint attributes with standard neighborhood-similarity and
connectivity descriptors. For a directed message edge $j\to i$, set
$C_{ji}=N(j)\cap N(i)$ and $U_{ji}=N(j)\cup N(i)$.
The thirteen raw coordinates are
\[
\begin{aligned}
\chi_{ji}=\big[&
 \log(1+d_j),\log(1+d_i),\widehat d_j,\widehat d_i,
 \mathbf1[\ell_j=\ell_i],|\widetilde a_j-\widetilde a_i|,
 \log(1+|C_{ji}|),\frac{|C_{ji}|}{\max(|U_{ji}|,1)},\\
 &\sum_{w\in C_{ji}}\frac1{\log(\max(d_w,2)+10^{-12})},
 \sum_{w\in C_{ji}}\frac1{\max(d_w,1)},
 \mathbf1[\{j,i\}\text{ is a bridge}], \\
 & \frac1{\sqrt{\max(d_jd_i,1)}},
 \frac{|C_{ji}|}{\max(\min(d_j,d_i),1)}\big].
\end{aligned}
\]
The neighborhood-similarity terms include common-neighbor counts and
Jaccard similarity~\citep{liben2003link},
the Adamic--Adar index~\citep{adamic2003friends}, and the resource-allocation
and hub-promoted indices~\citep{zhou2009predicting}.
The last coordinate is the hub-promoted index.
Standard implementations of Jaccard, Adamic--Adar, and resource-allocation
scores are also available in NetworkX~\citep{hagberg2007exploring}.
We use $e_{ji}=\mathcal Z_{\rm edge}(\chi_{ji})$, with training-edge
population moments, standard-deviation floor $10^{-6}$, and clipping to
$[-8,8]$. Both edge orientations are included in message passing.
All added features are deterministic functions of graph topology and
original attributes. They contain no node identifiers, target partitions,
or reference optimizer outputs, and do not depend on the current
assignment state.

\textbf{State and initialization.}
The recurrent variable is the soft partition $P_t\in\mathbb R^{n\times 4}$.
For each trajectory we independently draw node-wise and graph-wise Gaussian
logits and initialize
\begin{equation}
 P_{0,i}=\operatorname{softmax}\!\left(
 \frac{\epsilon_i+b_G}{\sqrt2}\right),\qquad
 \epsilon_i,b_G\stackrel{\rm i.i.d.}{\sim}\mathcal N(0,I_4).
 \label{eq:proteins-initialization}
\end{equation}
The same $b_G$ is shared by the nodes of a graph within a trajectory, but is
redrawn for each trajectory. 

\textbf{Message-passing cell.}
We use a width-32, three-layer PNA-style cell \citep{corso2020principal}.
In particular, let $\sigma=\operatorname{ReLU}$, at each time step the hidden embeddings are recomputed as
\[
 h_i^{(0)}=\operatorname{LayerNorm}
 \big(W_{e,2}\sigma(W_{e,1}[s_i,P_{t,i}]+b_{e,1})+b_{e,2}\big),
\]
with encoder dimensions $23\to32\to32$.
The state-pair edge vector is
\begin{equation}
 r_{ji}(P_t)=[P_{t,j},P_{t,i},P_{t,j}\odot P_{t,i},
 |P_{t,j}-P_{t,i}|]\in\mathbb R^{16}.
\end{equation}
For $\ell=0,1,2$, compute
\begin{align}
 m_{ji}^{(\ell)}
 &=M_\ell([h_j^{(\ell)},h_i^{(\ell)},e_{ji},r_{ji}(P_t)]),\nonumber\\
 a_i^{(\ell)}
 &=[\,\operatorname{mean}_{j\in N(i)}m_{ji}^{(\ell)},
       \operatorname{std}_{j\in N(i)}m_{ji}^{(\ell)},
       \operatorname{max}_{j\in N(i)}m_{ji}^{(\ell)},
       \operatorname{sum}_{j\in N(i)}m_{ji}^{(\ell)}\,],\nonumber\\
 h_i^{(\ell+1)}
 &=\operatorname{LayerNorm}\big(
 h_i^{(\ell)}+U_\ell([h_i^{(\ell)},a_i^{(\ell)},P_{t,i}])\big).
 \label{eq:proteins-mp}
\end{align}
$M_\ell$ is a $93\to32\to32$ MLP with ReLU after both affine maps;
$U_\ell$ is a $164\to32\to32$ MLP with ReLU only between affine maps.
Aggregations are elementwise.

\textbf{Global readout and update.}
The graph bottleneck and correction head are
\begin{equation}
 g_G=\frac1n\sum_iR(h_i^{(3)})\in\mathbb R^4,\qquad
 \Delta_{\theta,i}=O([h_i^{(3)},g_G]),
\end{equation}
where $R$ is a $32\to32\to4$ MLP and $O$ is a $36\to32\to4$ MLP,
each with ReLU between its two affine maps. We update
\begin{equation}
 P_{t+1}=\operatorname{softmax}\!\left[
 \log(P_t\vee10^{-6})+\Delta_\theta(G,P_t)\right],
 \label{eq:proteins-transition}
\end{equation}
with row-wise softmax and coordinate-wise clipping before the logarithm.
The three MP layers have distinct parameters within a cell; the entire cell
is tied across time. There are 35,784 trainable parameters:
1,888 in the encoder, $3\times10,464$ in message passing, 1,188 in the
global map, and 1,316 in the output head. The full model sizes are provided in Table \ref{tab:proteins-learned-configurations}.

\textbf{Training and objectives.}
For an undirected graph $G=(V,E)$ with $m=|E|$ edges and node degrees
$d_i$, let $z_i\in\{1,\ldots,K\}$ denote the module of node $i$.
At resolution $\gamma=1$, we minimize the negative modularity:
\[
 E_G(z)=-Q_G(z)
 =\sum_{k=1}^{K}
 \left(\frac{\sum_i d_i\mathbf1[z_i=k]}{2m}\right)^2
 -\frac1m\sum_{\{i,j\}\in E}\mathbf1[z_i=z_j].
\]
The edge term rewards placing connected nodes in the same module;
alone, it is minimized by assigning all nodes to one module.
The degree-volume term accounts for the within-module connectivity
expected from node degrees under the modularity null model.
Combining the two terms therefore favors modules with more internal
connectivity than this baseline predicts.

For a soft partition $P$, where each $P_i$ is a probability vector
over the $K$ modules, we define the differentiable energy
\begin{equation}
 e_G(P)=
 -\frac1m\sum_{\{i,j\}\in E}\langle P_i,P_j\rangle
 +\frac{1}{4m^2}
 \left[
 \left\|\sum_i d_iP_i\right\|_2^2
 -\sum_i d_i^2\|P_i\|_2^2
 \right].
 \label{eq:proteins-soft-energy}
\end{equation}
Specifically, for independent assignments
$Z_i\sim\mathrm{Categorical}(P_i)$,
\[
 e_G(P)=\mathbb E[E_G(Z)]-\frac{1}{4m^2}\sum_i d_i^2.
\]
Thus, the soft objective equals the expected discrete energy up to
a graph-dependent constant.

During training, each update samples ten training graphs with replacement
and four independent initial soft partitions $P_0$ per graph.
We unroll $T=5$ tied cells and minimize $\mathbb E[e_G(P_T)]$.
We use Adam with learning rate $3\times10^{-4}$, $\beta_1=.9$,
$\beta_2=.999$, $\epsilon=10^{-8}$, zero weight decay, global
gradient-norm clipping at 1, 3,500 optimizer updates, batch size
ten base graphs, and no learning-rate decay.
The final update-3,500 checkpoint is used; there is no best-test
or best-validation checkpoint substitution.
Our model and all baselines in
Table~\ref{tab:proteins-main} use training seeds 0, 1, and 2.
The unsupervised model never receives the Gurobi labels.

\textbf{Baselines.}
Table~\ref{tab:proteins-learned-configurations} summarizes the learned
configurations used in the reported comparisons. In this PROTEINS experiment, we use the same baseline groups as in the Ising experiment:
\begin{itemize}[leftmargin=5mm]
    \item \textit{Single-target supervision.} We use the same architecture as our multi-attractor GNN, replacing the energy objective $e_G(P)$ with cross entropy against one Gurobi-optimal partition per graph.
    \item \textit{Unique-equilibrium models.} As in the Ising experiments, we adopt the IGNN with hidden widths $H\in\{16,32,64\}$ and enforce a contraction bound of $\rho=0.9$. To train IGNN, we use the same unsupervised loss function $\mathbb E[e_G(P_T)]$. 
    Again, the input features are the same as in the multi-attractor setting.
    \item \textit{Untied feedforward depth controls.} The feedforward controls replace the same three-layer rich-PNA cell by $B\in\{1,2,3,4,5\}$ independent copies. They receive the same random $P_0$, losses, four training trajectories, optimizer, and 3,500-update budget.
    \item \textit{Distributional and diffusion-based models.}
We include the same distributional models as in the Ising experiments: a mean-field GNN, an annealed Bernoulli GNN,
    DIFUSCO, and Variational Annealing on Graphs (VAG-CO; denoted VAG).
    They are still applicable here. ``Matched" setting uses the same model backbone, learning rate, and training updates. ``Large" setting uses the upstream original repository's configurations. Details are in Table \ref{tab:proteins-learned-configurations}.
\end{itemize}

\begin{table}[ht]
\centering
%\vspace{-3mm}
\caption{Learned configurations on PROTEINS. $H$: hidden width; $L$: total untied depth.}
%\vspace{2mm}
\label{tab:proteins-learned-configurations}
\small 
\setlength{\tabcolsep}{3.5pt}
\begin{tabular}{lrrrrr}
\toprule
Model & $H$ & $L$ & Parameters & LR & Updates  \\
\midrule
Ours & 32 & 3 & 35,784 & $3\cdot10^{-4}$ & 3,500  \\
\midrule
Single-target Supervision & 32 & 3 & 35,784 & $3\cdot10^{-4}$ & 3,500  \\
IGNN (H16) & 16 & 3 & 11,340 & $3\cdot10^{-4}$ & 3,500  \\
IGNN (H32) & 32 & 3 & 39,564 & $3\cdot10^{-4}$ & 3,500  \\
IGNN (H64) & 64 & 3 & 146,700 & $3\cdot10^{-4}$ & 3,500  \\
\midrule
Feedforward ($\times 1$) & 32 & 3 & 35,784 & $3\cdot10^{-4}$ & 3,500  \\
Feedforward ($\times 2$) & 32 & 6 & 71,568 & $3\cdot10^{-4}$ & 3,500 \\
Feedforward ($\times 3$) & 32 & 9 & 107,352 & $3\cdot10^{-4}$ & 3,500 \\
Feedforward ($\times 4$) & 32 & 12 & 143,136 & $3\cdot10^{-4}$ & 3,500 \\
Feedforward ($\times 5$) & 32 & 15 & 178,920 & $3\cdot10^{-4}$ & 3,500 \\
\midrule
Mean-field GNN & 32 & 3 & 35,784 & $3\cdot10^{-4}$ & 3,500 \\
Annealed Bernoulli GNN & 32 & 3 & 35,784 & $3\cdot10^{-4}$ & 3,500 \\
DIFUSCO (matched) & 32 & 3 & 35,880 & $3\cdot10^{-4}$ & 3,500 \\
DIFUSCO (large) & 256 & 12 & 5,555,466 & $2\cdot10^{-4}$ & 2,800 \\
VAG (matched) & 32 & 3 & 40,270 & $3\cdot10^{-4}$ & 3,500 \\
VAG (large) & 40 & 3 & 101,790 & $5\cdot10^{-4}$ & 10,000 \\
\bottomrule
\end{tabular}
\end{table}

\textbf{Metrics.}
The caption of Table~\ref{tab:proteins-main} defines every displayed metric.

\textbf{Longer test-time trajectories.}
Table~\ref{tab:proteins-main} reports a one-step convergence rate of $C_{200}=88.05\pm2.23\%$ at $t=200$. To determine whether the remaining trajectories exhibit persistent oscillations or simply require more iterations, we extend the same 6,660 (111 $\times$ 20 $\times$ 3) trajectories across three training seeds to $t=10000$, without changing the learned weights. Here $t$ denotes test-time recurrence steps, and results are reported in Table~\ref{tab:proteins-long-horizon}.

The convergence rate increases to $C_{800}=98.75\pm0.62\%$ and $C_{3000}=99.49\pm0.29\%$. By $t=5000$, every trajectory is either convergent or classified as periodic at the reported horizons, with 6,632 trajectories satisfying the numerical convergence criterion. The extended checks confirm that the vast majority of non-converged trajectories at $t=200$ simply reflect longer transients, while a small remainder settles into stable periodicity rather than divergent behavior.

Although training uses only $T=5$ unrolled iterations, the learned update remains stable and largely settles to numerical fixed points when applied for substantially more iterations than seen during training. The longer runs serve only to examine convergence: all reported modularity, Hit10, and discovery metrics use the terminal outputs at $t=200$. Thus, the reported solution quality and diversity are already obtained by stopping at $t=200$, without waiting for every trajectory to strictly satisfy the convergence criterion.

\begin{table}[ht]
\centering
\small
%\vspace{-3mm}
\caption{PROTEINS test-time dynamics. Metrics are the same as in Appendix \ref{ising:appendix}, Table \ref{tab:ising_test_time}.}
%\vspace{2mm}
\label{tab:proteins-long-horizon}
\begin{tabular}{rrrr}
\toprule
$t$ & $C_t$ (\%) & Period 2-10 (\%) & Unclassified (\%) \\
\midrule
10 & $3.83\pm1.78$ & -- & -- \\
20 & $16.65\pm3.45$ & -- & -- \\
50 & $47.18\pm1.92$ & -- & -- \\
100 & $68.72\pm0.82$ & -- & -- \\
200 & $88.05\pm2.23$ & $0.02\pm0.02$ & $11.94\pm2.25$ \\
500 & $97.51\pm0.45$ & $0.39\pm0.28$ & $2.10\pm0.34$ \\
800 & $98.75\pm0.62$ & $0.42\pm0.28$ & $0.83\pm0.48$ \\
1000 & $99.07\pm0.30$ & $0.42\pm0.28$ & $0.51\pm0.24$ \\
2000 & $99.49\pm0.31$ & $0.42\pm0.28$ & $0.09\pm0.07$ \\
3000 & $99.49\pm0.29$ & $0.42\pm0.28$ & $0.09\pm0.04$ \\
5000 & $99.58\pm0.28$ & $0.42\pm0.28$ & $0.00\pm0.00$ \\
10000 & $99.58\pm0.28$ & $0.42\pm0.28$ & $0.00\pm0.00$ \\
\bottomrule
\end{tabular}
\end{table}

\section{Chemical reaction networks: details and additional results}
\label{chem:appendix}

This appendix provides the experiment details and additional results for
Section~\ref{chem:main}. 

\textbf{Dataset.}
We construct kinetic instances on the published reaction-network structures
from \citep{yao2025understanding}. We select networks labeled \texttt{Multi}
from the Atom and Joshi construction families, preserving their original
train/validation/test assignments. These labels indicate structural capacity
for multistationarity but do not provide kinetic rates or stationary
concentrations. For each selected structure, we add one inflow
$\varnothing\to X_i$ and one outflow $X_i\to\varnothing$ per species.
We then randomly propose two positive concentration vectors
$c^{(1)},c^{(2)}$ in log space and solve a linear feasibility problem
for shared positive rate constants $\kappa$ satisfying
$f_G(c^{(1)})=f_G(c^{(2)})=0$; the equations are linear in $\kappa$
when the concentrations are fixed.
We retain instances whose two witnesses have normalized residuals
at most $10^{-7}$, concentration-Jacobian condition numbers at most
$10^9$, and log-RMS separation at least $0.25$.
After feasibility filtering, we balance the two construction families
within each split, obtaining 1,242 training, 266 validation, and 236
test instances, with 2--5 species and 6--19 reactions.
Each instance contains the network structure, constructed rate constants,
and two numerically verified stationary witnesses, which need not
exhaust its solution set.
Physics-only training uses only $(\alpha,\beta,\kappa)$;
single-target supervision and supervised diffusion use one witness
selected once and held fixed as the label for each training instance.

\begin{table}[ht]
\centering\small
%\vspace{-3mm}
\caption{Dataset construction. Balance is enforced within each split after
feasibility filtering. Every retained instance has two constructed witnesses.}
%\vspace{2mm}
\label{chem:dataset-counts}
\begin{tabular}{lrrrrrr}
\toprule
Split & Atom attempts & Joshi attempts & Atom feasible & Joshi feasible & Retained & Reactions\\
\midrule
Train & 790 & 2560 & 771 & 621 & 1242 & 6--19\\
Validation & 180 & 570 & 172 & 133 & 266 & 6--18\\
Test & 300 & 600 & 293 & 118 & 236 & 6--18\\
\bottomrule
\end{tabular}
\end{table}

\textbf{Nodes and edges.}
There is one node per species and one per reaction, including inflow/outflow
reactions. For every pair $(i,r)$ with $\alpha_{ri}>0$ or $\beta_{ri}>0$,
we include both directed incidence edges.
There is no extra node for $\varnothing$, no stoichiometric-copy expansion,
and no direct species--species edge.
The six static node coordinates are
\begin{align}
 x_i&=(1,0,0,0,0,0), &&\text{species},\nonumber\\
 x_r&=\left(0,1,\log\kappa_r,
       \sum_i\alpha_{ri},\sum_i\beta_{ri},
       \sum_i|\beta_{ri}-\alpha_{ri}|\right), &&\text{reaction}.
 \label{chem:node-features}
\end{align}
The first two coordinates distinguish node types; the remaining coordinates
encode the log rate, total reactant/product orders, and net stoichiometric
magnitude. Apart from the log-rate transform, these graph inputs are used
directly, without dataset-fitted standardization.
The four edge coordinates are
\[
 e_{i\to r}=(\alpha_{ri},\beta_{ri},1,0),\qquad
 e_{r\to i}=(\alpha_{ri},\beta_{ri},0,1).
\]
The last two entries encode message direction, not reactant versus product.
A species can be both a reactant and a product, with its two coefficients
retained on the same incidence pair. Reaction roles are thus edge attributes,
not a single global species label.
All compared neural models receive these same base graph features.

\textbf{State and initialization.}
The recurrent state is only the species vector $z_t=\log c_t\in[-20,20]^s$.
With width $w=32$, compute the static encoding
\[
 b_v=\tanh\!\big(\operatorname{LayerNorm}(W_xx_v+b_x)\big).
\]
Each candidate draws independent node vectors $\xi_v\sim\mathcal N(0,I_8)$.
A bias-free $8\to32$ map and a scalar affine head initialize
\[
 z_{0,i}=\operatorname{clip}_{[-20,20]}
       \big(w_{\rm init}^{\mathsf T}(b_i+W_\xi\xi_i)+b_{\rm init}\big).
\]
Subsequent calls receive no identifier forcing.
At each call, reconstruct a temporary node matrix $H^{(0)}$ from $b$,
replace coordinate zero of each species row with its current $z_i$, and
leave reaction rows equal to $b_r$.
All temporary hidden coordinates are discarded after the call except the
new species log concentrations. There is no auxiliary carried hidden state.

\textbf{Message-passing backbone.}
One update contains three distinct learned GATv2 \citep{brody2022how} blocks;
the entire stack is tied across recurrent iterations.
For $\ell=1,2,3$, compute
\[
 U^{(\ell)}=\operatorname{LayerNorm}_{\ell}
       \big(\operatorname{GATv2}_{\ell}(H^{(\ell-1)},E,e)\big),\qquad
 H^{(\ell)}=
 \begin{cases}
 \tanh U^{(\ell)},&\ell<3,\\
 U^{(\ell)},&\ell=3.
 \end{cases}
\]
Each block is PyG
\texttt{GATv2Conv(32,32,heads=4,concat=False,edge\_dim=4)}.
Attention heads are averaged, not concatenated.
Self-loops are enabled, with mean incident edge features for their attributes;
these internal loops are not chemical reactions.
Other settings are zero attention dropout, LeakyReLU slope $0.2$,
separate source/target transforms, and biases enabled.
LayerNorm is affine with epsilon $10^{-5}$.
There is no spectral constraint on these attention blocks.

\textbf{Dynamic physical features.}
For $c=\exp z$, compute
$\ell_r=\log\kappa_r+\sum_i\alpha_{ri}z_i$ and protected fluxes
$\widetilde v_r=\exp(\operatorname{clip}_{[-30,30]}(\ell_r))$.
Set $\widetilde a=|N|\widetilde v$ and
$\widetilde q_i=(N\widetilde v)_i/(\widetilde a_i+10^{-12})$.
The node-wise dynamic features are
\[
 \psi_i=(z_i,\widetilde q_i,\log(\widetilde a_i+10^{-12})),\qquad
 \psi_r=(\ell_r,0,0).
\]
A shared affine $3\to32$ map followed by $\tanh$ encodes $\psi_v$.
The reaction log-flux feature itself is not clipped; clipping protects the
exponentials used to construct the internal residual.
The independent evaluation residual uses original coefficients and
unclipped mass-action fluxes.

\textbf{Gated readout and update.}
A bias-free conditioning map $W_b$ and the physical-feature encoder give
\[
 \widehat H=\operatorname{GRUCell}\!\left(
 \tanh\!\left[H^{(3)}+W_bb+\tanh(W_\psi\psi+b_\psi)\right],H^{(0)}\right).
\]
This width-32 GRU \citep{cho2014learning} is a within-call gating computation, not an additional
inter-iteration memory.
Three scalar affine heads applied to each species row of $\widehat H$ produce
\[
 d_i=\tfrac12\sigma(w_d^{\mathsf T}\widehat H_i+b_d),\qquad
 g_i=2\tanh(w_g^{\mathsf T}\widehat H_i+b_g),\qquad
 a_i^{\rm NN}=\tanh(w_a^{\mathsf T}\widehat H_i+b_a).
\]
The update is
\begin{equation}
 z_i^+=\operatorname{clip}_{[-20,20]}
 \left[z_i+d_i\left(g_i\widetilde q_i+
       \tanh(\widetilde R)a_i^{\rm NN}\right)\right],
 \qquad \widetilde R=
 \sqrt{s^{-1}\sum_i\widetilde q_i^2+10^{-24}} .
 \label{chem:update}
\end{equation}
The signed-gain head starts with zero weights and bias
$\operatorname{atanh}(1/2)$, so its initial gain is one.
The model has 35,716 trainable parameters.
It is a physics-assisted learned solver, not a generic GNN with only a physics
loss, and is not guaranteed contractive.
All aggregation/readout operations are permutation-equivariant.
Initialization is equivariant when $\xi$ is permuted with the nodes, and
equivariant in distribution under independent identifier sampling.
After initialization, the shared map depends only on $G$ and $z$.

\textbf{Training and residual objective.}
Using the original mass-action fluxes, define
\begin{equation}
 q_i(c)=\frac{(Nv(c))_i}{(|N|v(c))_i+10^{-12}},~~
 R(c)=\sqrt{\frac1s\sum_iq_i(c)^2},~~
 E_{\rm res}(c)=\frac1s\sum_i\phi_{0.1}(q_i(c)),
 \label{chem:residual}
\end{equation}
where $\phi_{0.1}(u)=u^2/0.2$ when $|u|<0.1$, and $|u|-0.05$ otherwise. We train the model by minimizing $\mathbb E[E_{\rm res}(c_{T})]$. In particular, we use AdamW with LR $0.002$, weight decay $10^{-5}$, default
$(\beta_1,\beta_2)=(0.9,0.999)$, epsilon $10^{-8}$,
gradient-norm clipping at 5, at most 50 epochs, and early-stopping patience 12.
Accumulate gradients over four graphs, with 20 candidate initializations
per graph.
Graph's recurrent unroll length $T$ is sampled
uniformly from $\{4,8,16,24,32\}$.

\textbf{Baselines.}
We retain the four baseline groups used in the Ising and PROTEINS experiments: single-target supervision, unique-equilibrium models, untied feedforward depth controls, and distributional models. Table~\ref{chem:all-counterparts} summarizes their configurations. All models receive the same base graph inputs $x_i, x_r, e_{i\to r}, e_{r\to i}$ and are evaluated on the same 236 test graphs.
\begin{itemize}[leftmargin=5mm]
\item \textit{Single-target supervision.}
We use the same architecture as our multi-attractor model, with the same random initialization, randomized unroll lengths, optimizer, and stopping rule, but replace the residual objective by the concentration-space MSE $s^{-1}\|c_T-c^{\rm target}\|_2^2$. For each training instance, $c^{\rm target}$ is drawn uniformly from the two witnesses once and then held fixed throughout training. The supervised diffusion models below use the same label.

\item \textit{Unique-equilibrium controls.}
As in the other two experiments, we adopt IGNN \citep{gu2020implicit} and enforce a contraction bound $\rho=0.9$, which guarantees a unique equilibrium. It is trained with the same residual objective $\mathbb{E}[E_{\rm res}]$ as our model. We evaluate hidden widths $H\in\{16,32,64\}$, which include a capacity-comparable model ($H=32$) and a smaller and a larger variant. Because its equilibrium is uniquely determined by $G$, IGNN produces a single candidate per graph; replicated copies are not counted toward diversity.

\item \textit{Untied feedforward depth controls.}
We replace the tied recurrence by $B\in\{1,2,3,4,5\}$ independently parameterized copies of our update, $z_{t+1}=\mathrm{GNN}_{\theta_t}(G,z_t)$, with distinct parameters $\theta_t$ at each stage. Each stage has its own input encoder, three GATv2 blocks, GRU, physical-feature encoder, and scalar heads. Only the identifier projection and initialization head, which produce $z_0$, are shared. The model thus has $3B$ message-passing blocks and $289+35{,}427B$ parameters, and it executes exactly $B$ stages at both training and test time. It uses the same random initialization, 20 candidates per graph, residual objective, and optimizer as our model. These controls compare increasing untied depth and parameter count with shared recurrence; inference computation is not matched. None is iterated beyond its trained depth, so no settling rate is reported.

\item \textit{Distributional models.}
Since concentrations are continuous, we replace the discrete generators used for Ising and PROTEINS by DDPM \citep{ho2020denoising}, DDIM \citep{song2021denoising}, and the time-reversed diffusion sampler (DIS) \citep{berner2024optimal}. DDPM is trained with the standard noise-prediction loss against the same fixed witness used for single-target supervision, after mapping log concentrations to the unbounded latent $u=20\,\mathrm{atanh}(\log c/20)$. It uses 1,000 linear-beta noise levels and 100 reverse steps. DDIM shares the trained DDPM checkpoint and samples deterministically ($\eta=0$) with 100 steps. DIS uses no witness. Instead, it learns to sample from the residual-based density $\propto\exp[-E_{\rm res}(c)/\tau]$ with $\tau=10^{-4}$, using 100 steps. The \emph{matched} setting keeps our width-32, three-block GATv2 backbone, GRU, physical-feature encoder, and heads. The only change is that the input encoder also receives the noisy species state and the normalized diffusion time, giving 35,780 parameters compared with 35,716 for our model. It uses the same learning rate, 50-epoch budget, and early-stopping patience as our model. The \emph{large} setting follows the upstream repositories. DDPM/DDIM use DIFUSCO's gated residual graph backbone with width 256 and 12 layers (5,205,249 parameters). DIS uses the original FourierMLP controller, conditioned on a width-32, three-block GAT encoder (59,330 parameters), and is trained with its native log-variance loss. Every generator returns 20 candidates per graph after 100 sampling steps. 
\end{itemize}

\begin{table}[ht]
\centering\small 
\setlength{\tabcolsep}{3pt}
%\vspace{-3mm}
\caption{Learned configurations for chemistry. $H$: hidden width; $L$: total untied depth.}
%\vspace{2mm}
\label{chem:all-counterparts}
\begin{tabular}{lrrrrr}
\toprule 
Model & $H$ & $L$ & Parameters & LR & Epochs \\
\midrule
Ours & 32 & 3 & 35,716 & $2\cdot10^{-3}$ & 50 \\
\midrule
Single-target Supervision & 32 & 3 & 35,716 & $2\cdot10^{-3}$ & 50  \\
IGNN (H16) & 16 & 3 & 10,273 & $2\cdot10^{-3}$ & 50 \\
IGNN (H32) & 32 & 3 & 37,441 & $2\cdot10^{-3}$ & 50 \\
IGNN (H64) & 64 & 3 & 142,465 & $2\cdot10^{-3}$ & 50 \\
\midrule
Feedforward ($\times 1$) & 32 & 3 & 35,716 & $2\cdot10^{-3}$ & 50 \\
Feedforward ($\times 2$) & 32 & 6 & 71,143 & $2\cdot10^{-3}$ & 50 \\
Feedforward ($\times 3$) & 32 & 9 & 106,570 & $2\cdot10^{-3}$ & 50 \\
Feedforward ($\times 4$) & 32 & 12 & 141,997 & $2\cdot10^{-3}$ & 50 \\
Feedforward ($\times 5$) & 32 & 15 & 177,424 & $2\cdot10^{-3}$ & 50 \\
\midrule
DDPM / matched & 32 & 3 & 35,780 & $2\cdot10^{-3}$ & 50  \\
DDIM / matched & 32 & 3 & 35,780 & $2\cdot10^{-3}$ & 50  \\
DIS / matched & 32 & 3 & 35,780 & $2\cdot10^{-3}$ & 50 \\
DDPM / large  & 256 & 12 & 5,205,249 & $2\cdot10^{-4}$ & 200  \\
DDIM / large  & 256 & 12 & 5,205,249 & $2\cdot10^{-4}$ & 200 \\
DIS / large & 32/64 & 3/4 & 59,330 & $5\cdot10^{-3}$ & 100 \\
\bottomrule
\end{tabular}
\end{table}

\textbf{Metrics.}
All metrics use the $n=236$ test graphs with equal weight per graph. Graph $g$ has $M_g$ candidates $c_{gm}$, with $M_g=20$. A candidate is accepted, $I_{gm}=1$, if it is finite, all its coordinates exceed $10^{-12}$, and $R(c_{gm})\le10^{-3}$. Here $R$ and $E_{\rm res}$ are defined as in the training objective, but are recomputed in float64 from the original $\alpha,\beta,\kappa$ rather than from the clipped internal features of the cell. We report
\[
 \overline E_{\rm res}=\frac1n\sum_g\frac1{M_g}\sum_mE_{\rm res}(c_{gm}),
 \qquad
 \mathrm{Hit}_{10^{-3}}=\frac{100}{n}\sum_g\frac1{M_g}\sum_m I_{gm}.
\]
Nonfinite residuals are kept in these means, and acceptance involves no polishing or matching to reference roots. To count distinct predictions, we process the accepted candidates of each graph in saved order. Each one joins the first existing representative within distance $\|\log c-\log c'\|_2/\sqrt{s_g}\le10^{-3}$, or otherwise becomes a new representative. With $K_g$ the resulting number of representatives ($K_g=0$ if no candidate is accepted), we report
\begin{equation}
 D_{20}^{\rm raw}=\frac1n\sum_gK_g,\qquad
 G_{20}^{\rm raw}=\frac{100}{n}\sum_g\boldsymbol1\{K_g\ge2\}.
 \label{chem:raw-counts}
\end{equation}
Numerically separated accepted predictions can approximate the same root; these counts therefore do not certify distinct stationary solutions.
For recurrent models, we measure numerical settling at horizon $T$ by
\begin{equation}
 C_T=\frac{100}{n}\sum_g\frac1{M_g}\sum_m
 \boldsymbol1\!\left\{
 \frac{\|z_{T+1}^{(g,m)}-z_T^{(g,m)}\|_2}{\sqrt{s_g}}<10^{-3}
 \right\},
 \label{chem:pointwise-convergence}
\end{equation}
and the column $C$ in Table~\ref{chem:main-results} reports $C_{200}$. This one-step test on the log-state is independent of the physical residual and does not certify asymptotic convergence. For IGNN, $C$ is instead the percentage of graphs for which all three implicit layers satisfy $\|F_{\ell,g}(H_{\ell,200})-H_{\ell,200}\|_{\infty,p}<10^{-3}$ after 200 updates per layer. This measures hidden-state settling, not chemical validity. Feedforward and generative models have no autonomous state update, so $C$ is not reported for them.

\textbf{Longer test-time dynamics.}
To investigate the trajectories that fail the $T=200$ one-step criterion,
we extend the same three frozen multi-attractor models and initialization seeds to $t=20{,}000$.
No model weights, initialization distributions, or physical acceptance
thresholds change. The main table's quality and diversity remain evaluated
at $T=200$, not at a retrospectively chosen longer horizon.

\begin{table}[ht]
\centering\small
%\vspace{-3mm}
\caption{Chemistry test-time dynamics, mean $\pm$ population SD across three training seeds. $C_t$ is the fraction with next-step species log-state RMS change below $10^{-3}$. Oscillatory denotes detected approximate periods $2$--$64$ among trajectories failing that criterion; unresolved is the remainder. All percentages use the full trajectory population, and the three categories partition that population for $t\ge200$, up to rounding. Dashes denote unperformed early-horizon oscillation tests.}
%\vspace{2mm}
\label{tab:chemistry-extended-dynamics}
\begin{tabular}{rccc}

\toprule

$t$ & $C_t$ (\%) & Oscillatory (\%) & Unresolved (\%) \\

\midrule

10 & 2.26 $\pm$ 0.65 & --- & --- \\

20 & 13.77 $\pm$ 3.72 & --- & --- \\

50 & 63.11 $\pm$ 2.05 & --- & --- \\

100 & 83.99 $\pm$ 0.79 & --- & --- \\

200 & 91.99 $\pm$ 0.35 & 1.75 $\pm$ 1.50 & 6.26 $\pm$ 1.27 \\

500 & 95.20 $\pm$ 1.72 & 2.23 $\pm$ 1.70 & 2.56 $\pm$ 0.54 \\

1000 & 95.64 $\pm$ 1.67 & 2.49 $\pm$ 2.04 & 1.88 $\pm$ 0.46 \\

2000 & 95.95 $\pm$ 1.83 & 2.77 $\pm$ 2.09 & 1.28 $\pm$ 0.26 \\

5000 & 96.15 $\pm$ 2.01 & 2.77 $\pm$ 2.09 & 1.08 $\pm$ 0.27 \\

10000 & 96.13 $\pm$ 2.01 & 2.77 $\pm$ 2.09 & 1.10 $\pm$ 0.27 \\

20000 & 96.12 $\pm$ 2.01 & 2.77 $\pm$ 2.09 & 1.11 $\pm$ 0.31 \\

\bottomrule

\end{tabular}
\end{table}

{Oscillation test:}
For $t\ge200$, use the $257$ saved log states
$w_j=z_{t+j}$, $j=0,\ldots,256$.
Let $A=\max_j\|w_j-\bar w\|_2/\sqrt{s}$.
Among trajectories failing the one-step settling criterion, select the
smallest $p\in\{2,\ldots,64\}$ satisfying
\[
 \max_{0\le j\le256-p}
 \frac{\|w_{j+p}-w_j\|_2}{\sqrt{s}}\le\min(10^{-3},0.02A).
\]
Additionally require $A>0$, half-window amplitude ratio in $[0.8,1.25]$,
and RMS distance at most $10^{-3}$ between the first/last cycle centers.
Half-window amplitudes are maximum RMS deviations from their respective
means over $w_0,\ldots,w_{127}$ and $w_{128},\ldots,w_{256}$;
cycle centers average the first and last $p$ states.
The window contains at least four cycles for every searched period.

{Findings:}
Settling increases from $91.99\pm0.35\%$ at $t=200$ to
$96.15\pm2.01\%$ at $t=5{,}000$ and then plateaus.
At $t=20{,}000$, $96.12\pm2.01\%$ are step-settled,
$2.77\pm2.09\%$ are oscillatory, and $1.11\pm0.31\%$ are unresolved.
The unresolved group contains 157 trajectories across the three seeds.
The oscillatory group includes 40 small-amplitude trajectories excluded by
the earlier amplitude-gated diagnostic.
No nonfinite states occurred.
Settling need not increase monotonically because each horizon is tested
independently.
The log-state box precludes finite-state norm divergence but not bounded
nonconvergence. Unresolved does not establish slow convergence, chaos,
or divergence: the category can include drift, irregular fluctuations,
periods beyond 64, or transients.
These observations concern the learned solver, not physical chemical kinetics,
and do not establish asymptotic attraction.

\section{Related Work}
\label{app:related}

\textbf{GNNs and their theory.}
Message-passing GNNs are at most as expressive as the 1-WL test \citep{xu2018powerful,morris2019weisfeiler}, and a line of work characterizes which invariant and equivariant maps GNNs can approximate \citep{maron2019provably,chen2019equivalence,keriven2019universal,azizian2021expressive,geerts2022expressiveness}. Unique node identifiers or i.i.d.\ random node features restore universality with high probability \citep{loukas2020what,sato2021random,abboud2020surprising}. Recent work extends this to measurable targets with explicit approximation rates \citep{gonon2026universality}, and to recurrent GNNs, which with random initialization can uniformly express any polynomial-time computable function on connected graphs \citep{rosenbluth2026repetition}. All of these results concern \emph{single-valued} targets $G\mapsto f(G)$, with randomness serving only to tell nodes apart. In our setting, the same random initial state must do double duty: it identifies nodes and selects the solution branch. Only the output-dimensional state $y_t$ is carried between steps, so node distinguishability must be preserved by dynamics that also converge to a solution (Lemma~\ref{lem:init-undamped-row-collision-null}, Theorem~\ref{thm:gnn-realization-main}).

\textbf{Recurrent/equilibrium GNNs.}
The original GNN was defined as the fixed point of a contraction \citep{scarselli2009graph}. Later weight-tied models unroll a shared update \citep{li2015gated,dai2018learning}, while several implicit GNNs impose sufficient conditions for equilibrium existence and uniqueness, including contraction and strong monotonicity \citep{gu2020implicit,liu2021eignn,park2022convergent,baker2023implicit}. \citet{yang2025implicit} also analyze unfolded GNN regimes with stationarity guarantees that do not require uniqueness. Deep equilibrium models compute fixed points \citep{bai2019deep}, with monotone variants guaranteeing uniqueness \citep{winston2020monotone}; path independence, meaning convergence to the same state across initializations, has also been advocated \citep{anil2022path}. Weight-tied message passing is also used for algorithmic reasoning and size extrapolation \citep{velivckovic2020neural,tang2020towards,bansal2022end} and for SAT and constraint satisfaction \citep{selsam2019learning,toenshoff2021graph}. RUN-CSP in particular starts from random states and keeps the best of several parallel runs \citep{toenshoff2021graph}, but without a representational account of the solutions reached. On the theory side, \citet{liu2026expressive} show that regular implicit operators (contractive in the state, Lipschitz in the input) represent exactly the locally Lipschitz maps, so that iterations rather than parameters supply expressivity. \citet{jore2026bifurcation} represent set-valued maps through attractor landscapes and demonstrate multi-solution discovery with recurrent GNNs on Ising problems, but do not establish permutation-equivariant dynamics or their message-passing realization. We develop this theory on graphs, where symmetry changes the picture in three ways: (i) the solution set is equivariant although individual branches need not be (Section~\ref{sec:theory-represent}); (ii) our message-passing construction realizes the dynamics through symmetry breaking carried by the state itself, without auxiliary inputs (Theorem~\ref{thm:gnn-realization-main}, Lemma~\ref{lem:init-undamped-row-collision-null}); and (iii) continuous, equivariant, globally contractive updates can miss every valid solution with positive probability (Theorems~\ref{thm:contractive-limits-main} and~\ref{thm:contractive-positive-failure}). As in \citet{liu2026expressive}, complexity comes from iteration: each step of the ideal update is globally Lipschitz, while its limit selector jumps across basin boundaries (Theorem~\ref{thm:ideal-dynamics-main}). Multiple attractors have long served as a resource in Hopfield networks \citep{hopfield1982neural,ramsauer2020hopfield}, where they store fixed patterns rather than solutions of an input-dependent problem.

\textbf{Distributional models.}
Another way to handle multiple solutions is to learn a distribution over them. Unsupervised neural combinatorial optimization trains GNNs from an energy, either as mean-field relaxations \citep{karalias2020erdos,schuetz2022combinatorial} or as annealed, autoregressive, diffusion, or GFlowNet samplers \citep{sun2022annealed,sanokowski2023variational,zhang2023let}; variational autoregressive networks play the same role for statistical-mechanics models \citep{wu2019solving}. Supervised diffusion solvers and diffusion samplers generate candidates through time-inhomogeneous reverse processes \citep{sun2023difusco,ho2020denoising,song2021denoising,berner2024optimal}. Other work keeps supervision but relaxes the single target: multiple output heads with a hindsight loss \citep{li2018combinatorial}, learned selection among one-of-many solutions \citep{nandwani2020neural}, label realignment under formulation symmetry \citep{chen2024symilo}, and equivariant flow matching with symmetry-aware target matching for bifurcation problems \citep{hendriks2025equivariant}. Our approach learns a single autonomous update whose attracting limits are intended to approximate valid solutions. Numerical stationarity and task validity are evaluated separately; iterations can be extended at test time without retraining, and training uses energies or residuals rather than solution labels. Classical multi-solution methods such as deflation and homotopy continuation \citep{farrell2015deflation,breiding2018homotopy} operate instance by instance, whereas our model is trained once and amortized over a graph family.

\textbf{Symmetry breaking.}
Deterministic equivariant networks cannot map a symmetric input to a less symmetric output \citep{smidt2021finding,kaba2023symmetry}. Existing remedies break symmetry through \emph{augmented inputs}: random or positional node features \citep{sato2021random,abboud2020surprising,dwivedi2022graph}, symmetry-breaking sets and objects \citep{smidt2021finding,xie2024equivariant,goel2026any}, and probabilistic symmetry breaking through equivariant conditional distributions and sampled canonicalizations \citep{lawrence2025improving}, the notion closest to our set-valued equivariance. We instead break symmetry through the model's own evolving state, with no auxiliary input: randomness enters once through $y_0$, and Lemma~\ref{lem:init-undamped-row-collision-null} shows that the ideal dynamics preserve node distinguishability at every finite step. The two mechanisms are compatible. Augmented features can be appended to $X$ or used to seed $y_0$ (our PROTEINS model uses random-walk encodings, and our chemistry model seeds $y_0$ from random node vectors), and the construction in Lemma~\ref{lem:family-wide-realization-initialization} only needs distinct node codes, whatever their source. Augmentation alone, however, does not settle \emph{which} solution the network should output: deterministic encodings alone do not provide a mechanism for sampling multiple valid outputs, and continuous canonicalization is impossible in general \citep{dym2024equivariant}. Theorems~\ref{thm:contractive-limits-main} and~\ref{thm:contractive-positive-failure} concern updates without auxiliary inputs; combining multi-attractor dynamics with augmented features is a natural direction for future work.

\end{document}